\documentclass[nohyperref]{article}

\usepackage{microtype}
\usepackage{graphicx}
\usepackage{subfigure}
\usepackage{booktabs} 

\usepackage{hyperref}

\usepackage[accepted]{icml2022}

\usepackage{amsmath}
\usepackage{amssymb}
\usepackage{mathtools}
\usepackage{amsthm}

\usepackage[capitalize,noabbrev]{cleveref}

\theoremstyle{plain}
\newtheorem{theorem}{Theorem}[section]

\theoremstyle{definition}

\theoremstyle{remark}

\usepackage[textsize=tiny]{todonotes}

\usepackage{nicefrac}       
\usepackage{microtype}      

\usepackage{amsmath,amsfonts,bm}
\usepackage{algorithm,algorithmic}

\def\1{\bm{1}}

\def\eps{{\epsilon}}

\def\vxi{{\bm{\xi}}}

\def\vomega{{\bm{\omega}}}
\def\vdelta{{\bm{\delta}}}

\def\vp{{\bm{p}}}

\def\vx{{\bm{x}}}
\def\vy{{\bm{y}}}
\def\vz{{\bm{z}}}

\DeclareMathAlphabet{\mathsfit}{\encodingdefault}{\sfdefault}{m}{sl}
\SetMathAlphabet{\mathsfit}{bold}{\encodingdefault}{\sfdefault}{bx}{n}

\newcommand{\E}{\mathbb{E}}

\newcommand{\R}{\mathbb{R}}

\DeclareMathOperator*{\argmax}{arg\,max}

\DeclareMathOperator{\sign}{sign}

\usepackage{amssymb}

\usepackage{xcolor}
\definecolor{mydarkblue}{rgb}{0,0.08,0.45}
\definecolor{mygreen}{rgb}{0.032, 0.6392, 0.2039}
\definecolor{mypurple}{HTML}{B266FF}

\def\LL{{\mathcal L}}

\def\LL{{\mathcal L}}

\newcommand\ignore[1]{}
\newcommand{\norm}[1]{\left\| #1 \right\|}
\usepackage[inline, shortlabels]{enumitem}  
\usepackage{caption}
\usepackage{dsfont}
\usepackage{graphicx}

\usepackage{bm,xcolor}
\usepackage{multirow}
\usepackage{colortbl}

\usepackage{tabularx}
\usepackage{multirow}
\newcolumntype{Y}{>{\centering\arraybackslash}X}

\usepackage{arydshln}  
\usepackage{tikz}
\usetikzlibrary{arrows,positioning,calc}
\usepackage{placeins}

\icmltitlerunning{Improving Generalization via Uncertainty Driven Perturbations}

\begin{document}
\twocolumn[
\icmltitle{Increasing the Classification Margin with Uncertainty Driven Perturbations}


\begin{icmlauthorlist}
\icmlauthor{Matteo Pagliardini}{epfl}
\icmlauthor{Gilberto Manunza}{epfl}
\icmlauthor{Martin Jaggi}{epfl}
\icmlauthor{Michael I. Jordan}{ucb}
\icmlauthor{Tatjana Chavdarova}{ucb}
\end{icmlauthorlist}

\icmlaffiliation{epfl}{EPFL, Switzerland}
\icmlaffiliation{ucb}{UC Berkeley, USA}

\icmlcorrespondingauthor{Matteo Pagliardini}{matteo.pagliardini@epfl.ch}

\icmlkeywords{generalization,robustness,neural networks}

\vskip 0.3in
]



\printAffiliationsAndNotice{\icmlEqualContribution} 

\begin{abstract}
Recently~\citeauthor{shah2020pitfalls},~\citeyear{shah2020pitfalls} pointed out the pitfalls of the \textit{simplicity bias}---the tendency of gradient-based algorithms to learn simple models---which include the model's high sensitivity to small input perturbations, as well as sub-optimal margins. In particular, while Stochastic Gradient Descent yields max-margin boundary on linear models, such guarantee does not extend to non-linear models. To mitigate the simplicity bias, we consider uncertainty-driven perturbations (UDP) of the training data points, obtained iteratively by following the direction that maximizes the model's estimated uncertainty. The uncertainty estimate does not rely on the input's label and it is highest at the decision boundary, and---unlike loss-driven perturbations---it allows for using a larger range of values for the 
perturbation magnitude. 
Furthermore, as real-world datasets have non-isotropic distances between data points of different classes, the above property is particularly appealing for increasing the margin of the decision boundary, which in turn improves the model's generalization. We show that UDP is guaranteed to achieve the maximum margin decision boundary on linear models and that it notably increases it on challenging simulated datasets. For nonlinear models, we show empirically that UDP reduces the simplicity bias and learns more exhaustive features. Interestingly, it also achieves competitive loss-based robustness and generalization trade-off on several datasets.
Our codebase is available here:~\href{https://github.com/mpagli/Uncertainty-Driven-Perturbations}{github.com/mpagli/Uncertainty-Driven-Perturbations}
\end{abstract}

\section{Introduction}\label{sec:intro}

While neural networks can represent any function~\citep{hornik1989universal} and large models perfectly ``fit'' the \textit{training} data, a question arises if these models \textit{generalize well} for the task at hand, i.e., if the accuracy on data samples that were \textit{not} used during the training is high. For safety critical applications, addressing this question is often complex as it involves assessing the biases of the given dataset~\citep{bolukbasi2016,gebru2021datasheets,shah2020pitfalls}, the inductive bias of the predefined model~\citep{fukushima1982,LeCun1998,botev2021priors}, the implicit bias of the optimization method used to train the model~\citep{Neyshabur2015,zhang2017}, and their ``interaction''. In this paper, we limit our focus to that of \textit{given a training dataset and a predefined model}, improving the chances that the model generalizes well.
In this aim, a useful concept is that of the \textit{margin} of a classifier $\mathcal{C}_\vomega: \vx \mapsto \hat{\vy}$, with $\vx\in\R^d$ denoting a data sample of finite dataset $\{ \vx_i,\vy_i\}_{i=1}^N$ drawn from the training data distribution $p_d$, $\hat{\vy} \in \R^c$ its prediction for $c$ possible classes, and $\vomega\in\R^m$ the parameters of the model.
We define the margin as the maximum of the margins of the single points $m=\max(m(\vx_1),\dots,m(\vx_N))$, where the margin $m(\vx_i)$ of a point $\vx_i$ is defined as the smallest distance from it to a decision boundary $ \mathcal{B}$, i.e., $m(\vx_i) = \min_{\tilde{\vx} \in \mathcal{B}} \norm{\vx_i-\tilde{\vx}} $.

\begin{figure*}[ht]
  \centering
  \subfigure[$\epsilon_{train} \leq \epsilon_1$: UDP \& LDP]{
  \label{subfig:illust_c1}
  \includegraphics[width=.24\linewidth]{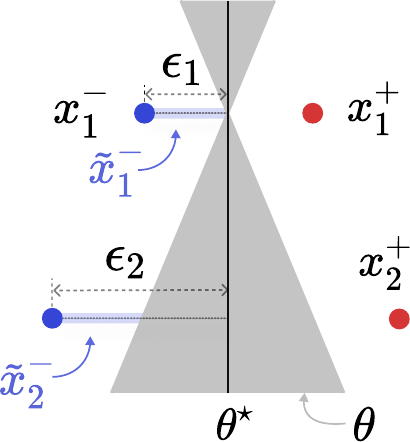}} \hspace{.4cm}
  \subfigure[\ref{eq:ldp} landscape and perturbations]{
  \label{subfig:illust_c2}
  \includegraphics[width=.265\linewidth,trim={0 .0cm 0 .4cm},clip]{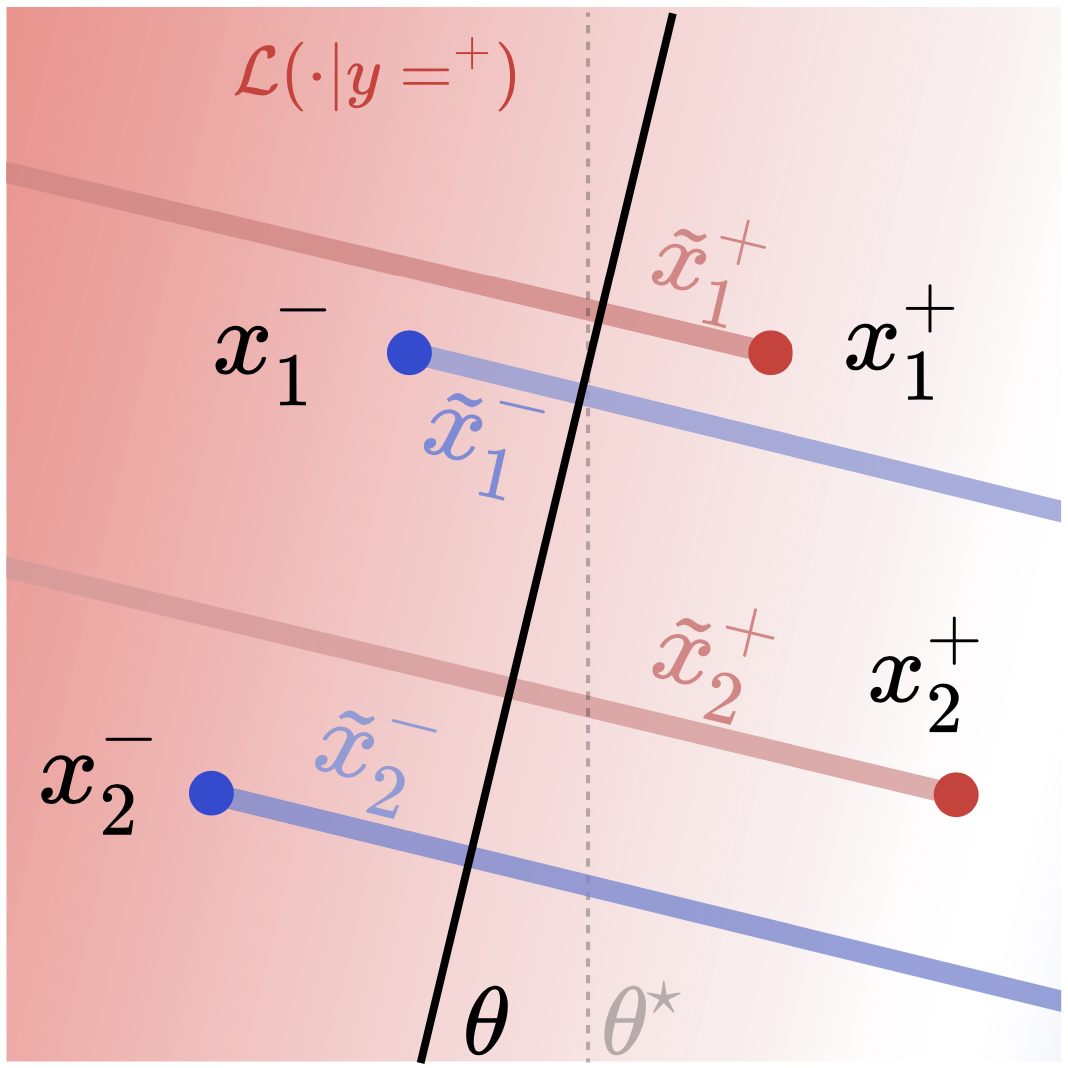}}
  \hspace{.4cm}
  \subfigure[\ref{eq:udp} landscape and perturbations]{
  \label{subfig:illust_c3}
  \includegraphics[width=.255\linewidth,trim={0 .0cm 0 .0cm},clip]{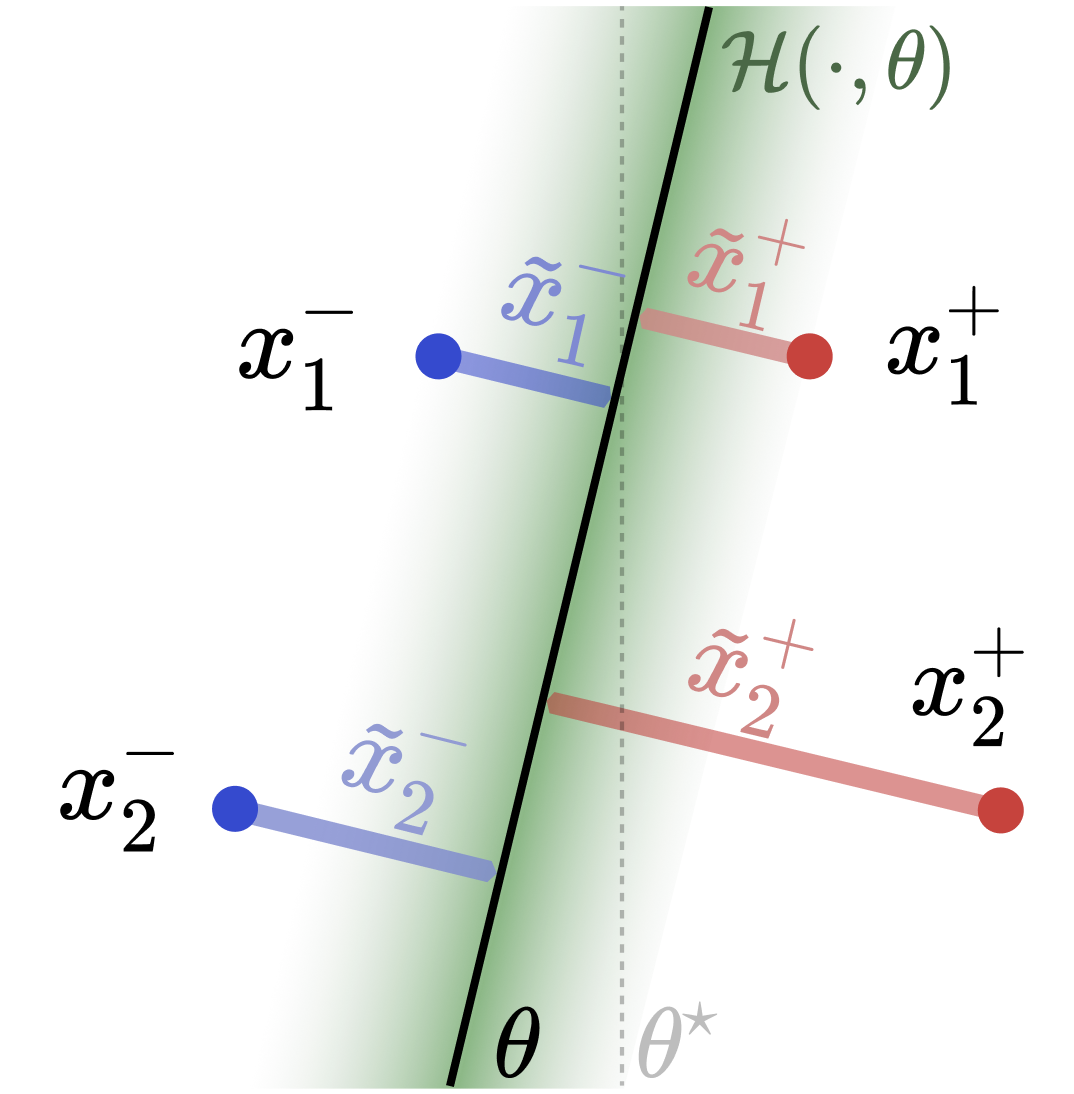}}
\vspace{-.2cm}
\caption{
Pictorial representations of~\ref{eq:ldp} and~\ref{eq:udp}. 
Circles represent training samples, and color their class $y\in\{-,+\}$. 
$\theta$ and $\theta^\star$ denote the current decision boundary and the optimal one, resp.
\textbf{Fig.~\subref{subfig:illust_c1}.}
Given \textit{ideal} data-perturbation training procedure using $\epsilon_{train}<\epsilon_1$ the perturbed data points (either with UDP or LDP) will belong to the shaded blue area (assuming the boundary $\theta$ has $0$ training loss).
Thus, the final decision boundary will belong to the shaded grey area, and we \textit{cannot} guarantee the margin is maximized--see~\S~\ref{sec:intro}.
For clarity, only the perturbations $\tilde\vx^-$ of the negative (blue) class are shown.
\textbf{Fig.~\subref{subfig:illust_c2}} illustrates what is happening for LDP when $\eps$ is potentially large. The loss used for LDP perturbations (of positive samples) keeps increasing the further we go on the left of $\theta$. 
For clarity, only the loss landscape of \textit{positive} samples is shown, where darker is higher.
In \textbf{Fig.~\subref{subfig:illust_c3}}, we see \textit{why UDP's property of not crossing the boundary is advantageous}.
The green shade depicts the uncertainty, 
where darker is higher -- starting from a point with negative label the uncertainty increases, reaches its maximum, and then decreases.
Thus, the LDP and UDP perturbed samples---obtained by iteratively following the direction that \textit{maximizes} these quantities---differ, and their possible regions are shown in red and blue shade, resp. for the positive and negative class.
When the perturbed samples from the positive class pass $\theta^\star$, LDP yields training pair $(\tilde x_i^+, y_i)$ with a semantically wrong label.
This indicates that to avoid oscillations (caused by wrongly labeling input space regions), for LDP we are restricted to using $\eps_{train}\leq \eps_1$, but as Fig.~\subref{subfig:illust_c1} shows, this can be restrictive.
On the other hand, for UDP we can have relatively larger $\eps$ which  ``locally adapts'' and is effectively smaller in some region and larger in others.
See \S~\ref{sec:method} for discussion.
}\label{fig:illustration_udp_ldp}
\end{figure*}

Stochastic Gradient Descent (SGD,~\citet{robbins1951stochastic,kiefer1952sgd,bottou2010SGD})--and variants of it--is the \textit{de facto} optimization method for training neural networks.
Interestingly, on linearly separable datasets, SGD provably converges to the \textit{maximum-margin} linear classifier~\citep{soudry2018implicit}, thereby achieving a superior generalization performance.
Unfortunately,~\citeauthor{shah2020pitfalls}, \citeyear{shah2020pitfalls} rigorously showed that this propriety of SGD does \textit{not} follow on non-linear models such as the commonly used neural network architectures.
In particular, the authors consider a nonlinear neural network with three parameters and a dataset for which we know that achieving max-margin requires using \textit{all} the input coordinates. The authors then show that the model effectively uses \textit{only one} of the input space coordinates, since the two other parameters are close to 0~\citep[see Thm. 1 of][]{shah2020pitfalls}.
This phenomenon is referred to as the ``simplicity bias'', i.e. ``the tendency of standard training procedures such as SGD to find simple models''.
While several works argue that the simplicity bias can act as a form of regularization forcing the boundary hyperplane to be ``simpler'' and thus improve the generalization~\citep{arpit1017,dziugaite2017}, others argue that in general, obtaining the simplest model can cause it to falsely rely on \textit{spurious} features~\citep[see ][ and references therein]{yixin2021jordan} and thus it may generalize poorly.
Other examples where the simplicity bias is adverse include the related fields of \textit{transfer learning}~\citep{yosinski2014,Wang2020Pay,shafahi2020adversarially} and \textit{distributional shift}~\citep{sun2015return,amodei2016concrete,Wang2020Pay} where ``exhaustive'' features are required for good performances on the transfer task.

In addition, the hypothesis that simplicity bias aids generalization does not support the evidence that small perturbations added to the input can easily ``fool'' well-performing DNNs into making wrong predictions~\citep{battista2013evasion,szegedy2014intriguing}--referred to as robustness to perturbations.
This observation motivated an active line of research on \textit{adversarial training}~\citep[see ][ and references therein]{Biggio18adversarial}. AT methods exploit the prior knowledge for some applications, that all data points~$\tilde\vx$ within a small region around a training data point  $\vx_i,  i \in [1, N]$ belong to the same class:
\begin{equation}\tag{AT-Asm}\label{eq:at_assumption}
\begin{aligned}
   &\norm{\tilde\vx - \vx_i } \leq \epsilon \Rightarrow \tilde\vy = \vy_i, \quad \text{and} \quad \nexists j\in[1,N], \\
   &\text{ s.t. } \quad j \neq i,  \vy_j \neq \vy_i, \norm{ \tilde\vx - \vx_j } \leq \epsilon  \,.
\end{aligned}
\end{equation}
Starting from a training datapoint $\vx_i$, AT methods \textit{iteratively} find a perturbation $\tilde \vx_i^{0}\triangleq\vx_i$ using the loss of the classifier $\mathcal{L}(C_\vomega (\vx_i^{j}), \vy_i ), j=0,\dots,k$ as an objective by following the direction that \textit{maximizes} it, see \S~\ref{sec:preliminaries}--herein called \textit{loss-driven perturbations} (LDP).
Popular such methods are \textit{projected gradient descent}~\citep[PGD, ][]{madry2018towards}---where the inner maximization step is implemented with $k$ gradient ascent steps, and its ``fast variants''---which take only one step to compute the perturbation $\vdelta$, e.g., \textit{fast gradient sign method}~\citep[FGSM, ][]{goodfellow2015explaining},  see \S~\ref{sec:preliminaries}.
However, further empirical studies show that~\ref{eq:ldp} methods often reduce the average accuracy on ``clean'' unperturbed test samples, indicating the two objectives---robustness and clean accuracy---might be competing~\citep{tsipras2019robustness,su2018robustness}, and~\citep{bubeck2021law} provide conjecture on the inherent trade-offs between the size of a two-layer neural network and its robustness.
Finally,~\citep{ZhangTRADES} show that learning models are vulnerable to small adversarial perturbations because the probability that data points lie around the decision boundary of the model is high~\citep{ZhangTRADES}.
In this work, we argue that the problems of poor robustness and high simplicity bias are related and that these can be improved by increasing the margin of the decision boundary.

A separate line of work aims at increasing the interpretability of the model by associating to each of its predictions an \textit{uncertainty estimate}~\citep{kim2016,doshivelez2017rigorous}. 
Two main uncertainty types in machine learning are considered:
\begin{enumerate*}[series = tobecont, itemjoin = \quad, label=(\roman*)]
\item \textit{aleatoric}--describing the \textit{noise} inherent in the observations, as well as
\item \textit{epistemic}--uncertainty originating \textit{from the model}.
\end{enumerate*}
While the former \textit{cannot} be reduced, the latter arises due to insufficient data to train the model, and it \textit{can} be explained away given enough data (assuming access to universal approximation function). In the context of classification, apart from capturing high uncertainty due to overlapping regions of different classes, the epistemic uncertainty also captures which regions of the data space are not ``visited'' by the training samples.
Importantly, as \textit{many} functions can perfectly fit the training data--i.e. the function represented by the DNN model is often non-identifiable~\citep{raue2013,martn2010BayesianI,moran2021identifiable}, and given an input $\tilde\vx$ the uncertainty estimate incorporates quantification if possible hypothesis function ``agree'' on the prediction of $\tilde\vx$. 
It is thus natural to consider that training with perturbations that \textit{maximize} this quantity will most efficiently narrow
down the hypothesis set, leaving out the functions that increase the margin.

In summary,
\begin{enumerate*}[series = tobecont, itemjoin = \quad, label=(\roman*)]
\item a promising approach of incorporating inductive bias for improved generalization in training procedure is to use the prior knowledge that the label in the region around the training data points does not change; and on a separate note
\item the epistemic uncertainty quantifies ``how much the modeler sees a possibility to reduce its uncertainty by gathering more data''~\citep{Kiureghian2009}. 
\end{enumerate*}
Moreover, since the former yields infinite training datasets, it is necessary to use efficient ``search'' of perturbations.
Given these considerations, and in aim to improve the generalization of the neural networks, a natural question that arises is: 
\begin{itemize}[leftmargin=*,noitemsep,topsep=0pt,partopsep=0pt,parsep=0pt]
\item \textit{Can we use the direction of maximum uncertainty estimates to guide the search for efficient input-data perturbations?} 
\textit{If yes, are there advantages over loss-driven perturbations?}
\end{itemize}

Since the quantitative methods yielding uncertainty estimates for neural networks are differentiable, we answer the former question affirmatively. In this paper, we consider ``uncertainty driven perturbations'' (UDP)--an inductive bias training framework that uses the model's estimated estimates to find effective perturbations of the training samples.

\noindent\textbf{Contributions.} Our contributions can be summarized as: 
\begin{itemize}[itemsep=-.05em]
    \item
    We first show that LDP suffers from high hyperparameter sensitivity of $\eps$: \textit{(i)} small $\eps$ may lead to reduced generalization, and \textit{(ii)} large $\eps$ to unstable training and reduced test accuracy.   
    \item 
    We propose uncertainty-driven perturbations (UDP) as a general family of algorithms that use the model's uncertainty estimate for finding efficient perturbations of the training data samples.
    We show analytically that UDP converges to the max-margin for a simplistic linear setting, on which LDP does not. We further provide geometrical intuition on the advantages of UDP--namely, UDP reduces the $\eps$-sensitivity relative to LDP. 
    \item Empirically, we show that \textit{(i)} albeit not trained explicitly, UDP is also robust to loss-based perturbations of the test sample, while maintaining a competitive trade-off with generalization on Fashion-MNIST, CIFAR-10, and SVHN \textit{(ii)} it reduces the simplicity bias--as shown by the transferability of the features learned on CIFAR-10 to CIFAR-100, as well as \textit{(iii)} it improves the generalization in low data regime when using latent space perturbations.
\end{itemize}

\noindent\textbf{Structure.}
The background above, along with the preliminaries in \S~\ref{sec:preliminaries} suffice for the remaining of the paper. As this work is related to several relevant lines of works, a more extended overview is deferred to App.~\ref{app:related}, due to space constraints. 
Following the preliminaries, we
\begin{enumerate*}[series = tobecont, itemjoin = \quad, label=(\roman*)]
\item list motivating examples on the failure cases of LDP~\ref{sec:example},
\item formulate the proposed method, discuss its convergence and advantages in \S~\ref{sec:method}, and
\item empirical results are listed in \S~\ref{sec:exp}.
\end{enumerate*}

\begin{figure}
    \centering
    \includegraphics[width=.5\linewidth]{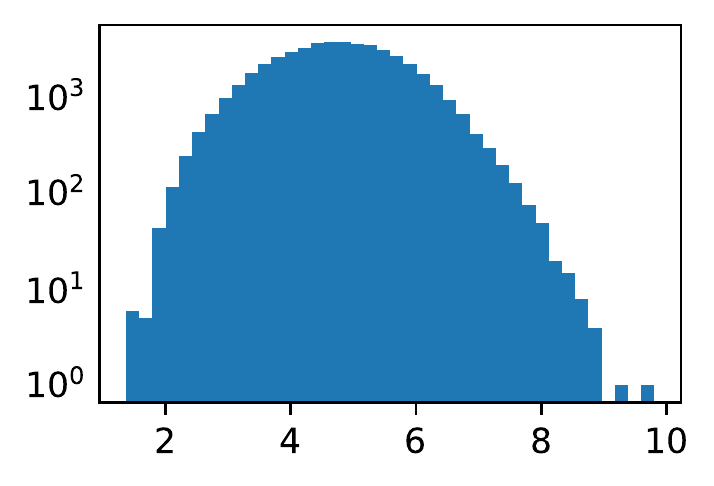}
    \vspace{-.3cm}
    \caption{Histogram of Euclidean distances in input space 
    for CIFAR-10 training samples to their closest neighbour of \textit{opposite class}. That is, for each data sample $\vx_i$, 
    we take the 1-nearest neighbour of different class $\vx_j$, with $\vy_i \neq \vy_j$, and we compute their distance. 
    See \S~\ref{sec:example} for discussion.
    }
    \label{fig:historgram_distances_cifar10}
    \vskip -.2em
\end{figure}

\section{Preliminaries}\label{sec:preliminaries}

\noindent\textbf{Training with perturbations.} 
In this section, we consider the supervised setting for simplicity. The perturbations can be applied in either the input space $\mathcal{X}\triangleq \R ^d$ or the latent space of the model $\mathcal{C}_\vomega(\cdot)$, i.e., $\mathcal{Z}\triangleq \R^l$, with $l$ denoting the dimension of the latent space:
\vspace*{-.1cm}
\begin{equation}\tag{ERM-P}\label{eq:at_erm}
    \min_\vomega \displaystyle \mathop{\mathbb{E}}_{(\vx,\vy)\sim p_d} 
    \big[  \LL( \mathcal{C}_\vomega (\mathcal{E}(\vx)\!+\!\vdelta), \vy)\big] \,,
\end{equation}  
where the encoder $\mathcal{E}$ is the identity when perturbations are applied in the input space, and $\mathcal{L}$ denotes the loss function,. For the latter, with abuse of notation, the model can be seen as having an encoder part $\mathcal{E}:\vx\mapsto\vz$ and a classification part $\mathcal{C}_\vomega:\vz\mapsto\vy$, and in that case $\vdelta \in \R^l$.

\noindent\textbf{Loss Driven Perturbations (LDP).} 
Common adversarial training methods \textit{greedily} target the robustness weakness of DNNs by using \textit{loss}-driven perturbations which  find $\vdelta\triangleq\vdelta_\ell$ as follows:
\vspace*{-.23cm}
\begin{equation}\tag{LDP}\label{eq:ldp}
    \vdelta_\ell = 
    \max_{\vdelta\in \Delta} \LL( \mathcal{C}_\vomega (\mathcal{E}(\vx)\!+\!\vdelta), \vy) \,,
\end{equation}
where the added \textit{worst-case} perturbation $\vdelta\in\R^d$ is constrained to a small region around the sample $\vx$, typically $\Delta \triangleq \{\delta \text{ s.t. } \| \delta\|_k <\epsilon \}$  with $\epsilon > 0$, and $k \in \{1,2,\infty\}$.

\noindent\textbf{Common perturbation-search methods.} 
The~\eqref{eq:ldp} maximization problem can be implemented in several ways. 
As in general the optimization is non-convex,~\citet{lyu2015unified} propose approximating the inner maximization problem with Taylor expansion and then applying a Lagrangian multiplier.
For $\ell_\infty$-bounded attacks, this linearization yields the FGSM~\citep{goodfellow2015explaining} method: 
\vspace*{-.1cm}
\begin{equation}\tag{FGSM}\label{eq:fgsm}
    \vdelta_{\text{FGSM}} \triangleq \varepsilon \cdot \text{Sign}\big(\mathop{\nabla}\limits_{\vx}  \LL(\mathcal{C}_\vomega (\vx), \vy)\big)\,,
\end{equation}
where $\text{Sign}(\cdot)$ denotes the sign function. 
The PGD method~\citep{madry2018towards} applies \ref{eq:fgsm} for $i= 1, \dots, k$ steps (with $\vdelta_{\text{PGD}}^{0} \triangleq \mathbf{0}$): 
\noindent
\begin{equation}\tag{PGD}\label{eq:pgd}
    \vdelta_{\text{PGD}}^{i} \triangleq  \mathop{\Pi}\limits_{\|\cdot\|_{\infty} \leq \epsilon} \Big( \vdelta_{\text{PGD}}^{i-1} + \alpha \cdot \text{Sign}\big(\mathop{\nabla}\limits_{\vx} \LL(\mathcal{C}_\vomega (\vx+ \vdelta_{\text{PGD}}^{i-1}), \vy)\big)
    \Big) \,,
\end{equation}
where $\alpha\in[0,1]$ is selected step size, and $\Pi$ is projection on the  $\ell_\infty$--ball.
PGD with $k$ steps is often referred to as PGD-$k$. See also App.~\ref{app:related_at} for additional methods.

\noindent\textbf{Training with $f(\tilde\vx, \cdot)$ regularizer}. Instead of using solely the perturbed samples during training,~\citep{ZhangTRADES} propose to use these for a regularization term to encourage smoothness of the output around the training samples: 
\begin{equation}\tag{TRADES}\label{eq:trades}
\min_{\omega} 
\displaystyle \mathop{\mathbb{E}}_{(\vx,\vy)\sim p_d} 
\big[ \mathcal{L}(\mathcal{C}_\vomega(\vx), \vy) + \lambda \max_{\delta \in \Delta} \mathcal{L}(\mathcal{C}_\vomega(\vx), \mathcal{C}_\omega(\vx+\delta))  \big] \,,
\end{equation}
where $\lambda \in (0,1)$.
The perturbation $\tilde\vx$ is found such that the model's output for it disagrees with the output for the training point $\vx$ as much as possible, and training aims to smoothen the output of $C_\vomega(\cdot)$ locally around $\vx$.

\noindent\textbf{Uncertainty estimation.} 
For brevity, here we focus on the uncertainty estimation methods relevant for the remaining of the main paper, see App.~\ref{app:related} for a more extensive overview. 
Given an ensemble of $M$ models $\{\mathcal{C}_{\vomega}^{(m)}\}_{m=1}^M$ (for example, trained independently or sampled with MC Dropout~\citep{gal2016mcdropout}, see App.~\ref{app:related}) where each model outputs non-scaled values--``logits'', we define 
$\hat{\vy} \in \R^C$ as the average prediction: 
$
    \hat{\vy} = \frac{1}{M} \sum_{m=1}^M \text{Softmax}\big(\mathcal{C}^{(m)}_\vomega(\vx)\big) \,.
$
Given $\hat{\vy}$, there are several ways to estimate the model's uncertainty, see~\citep[][ \S 3.3]{Gal2016Uncertainty}.
In the remaining, we use the \textit{entropy} of the output distribution (over the classes) to quantify the uncertainty estimate of a given sample $\vx$:
\vspace*{-.1cm}
\begin{equation}\tag{E}\label{eq:entropy}
    \mathcal{H}(\vx, \vomega) = - \sum_{c \in C}  \hat{\vy}_c \log \hat{\vy}_c \,.
\end{equation}

\begin{figure*}[ht] 
  \centering
  \subfigure{ 
  \includegraphics[width=.2\linewidth]{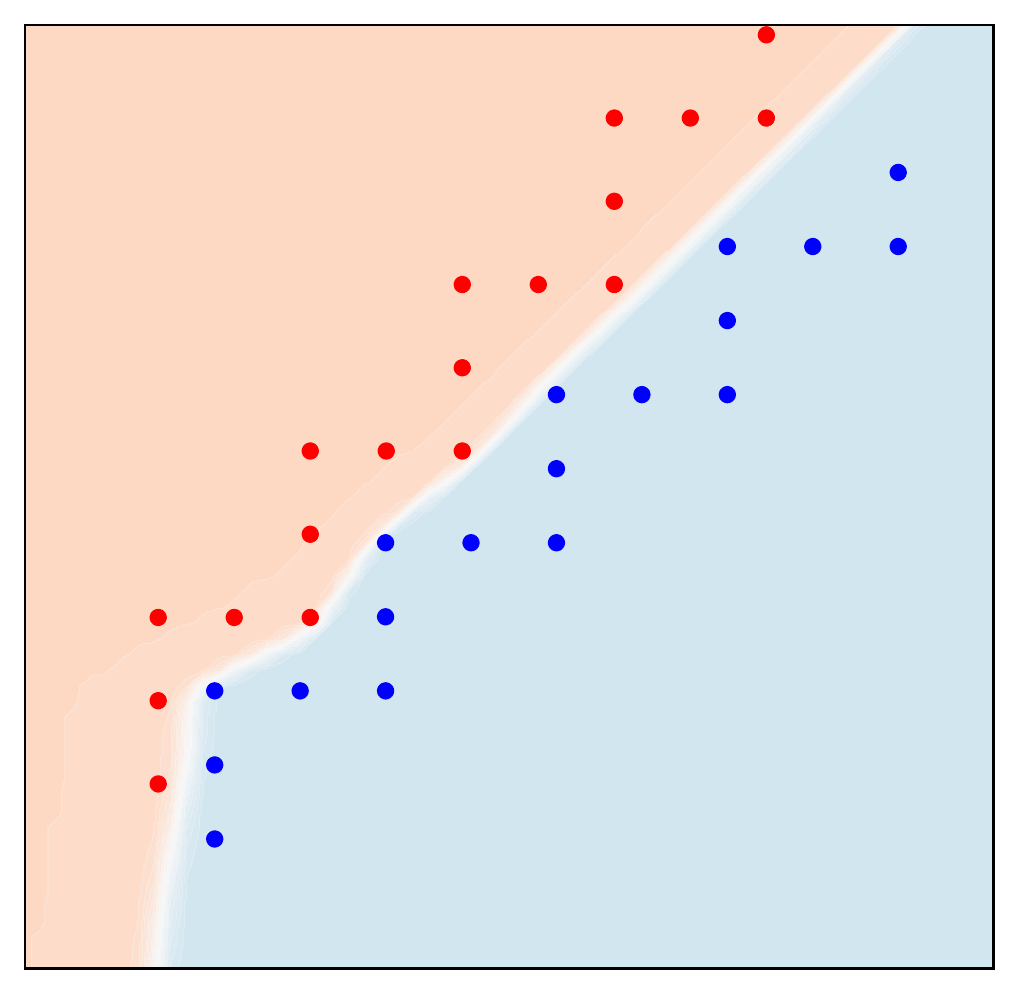}
  }\hspace{.2cm}
  \subfigure{ 
  \includegraphics[width=.2\linewidth,clip]{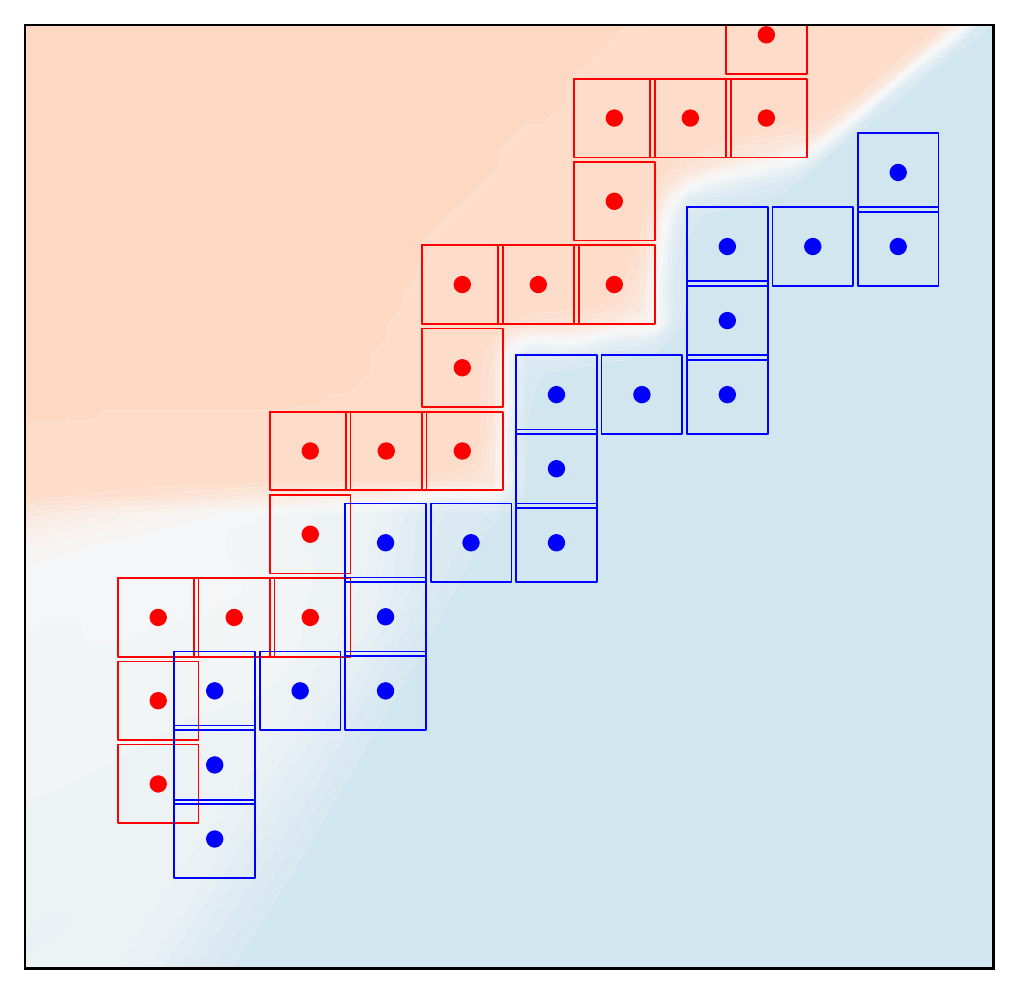}
  }\hspace{.2cm}
  \subfigure{ 
  \includegraphics[width=.2\linewidth,clip]{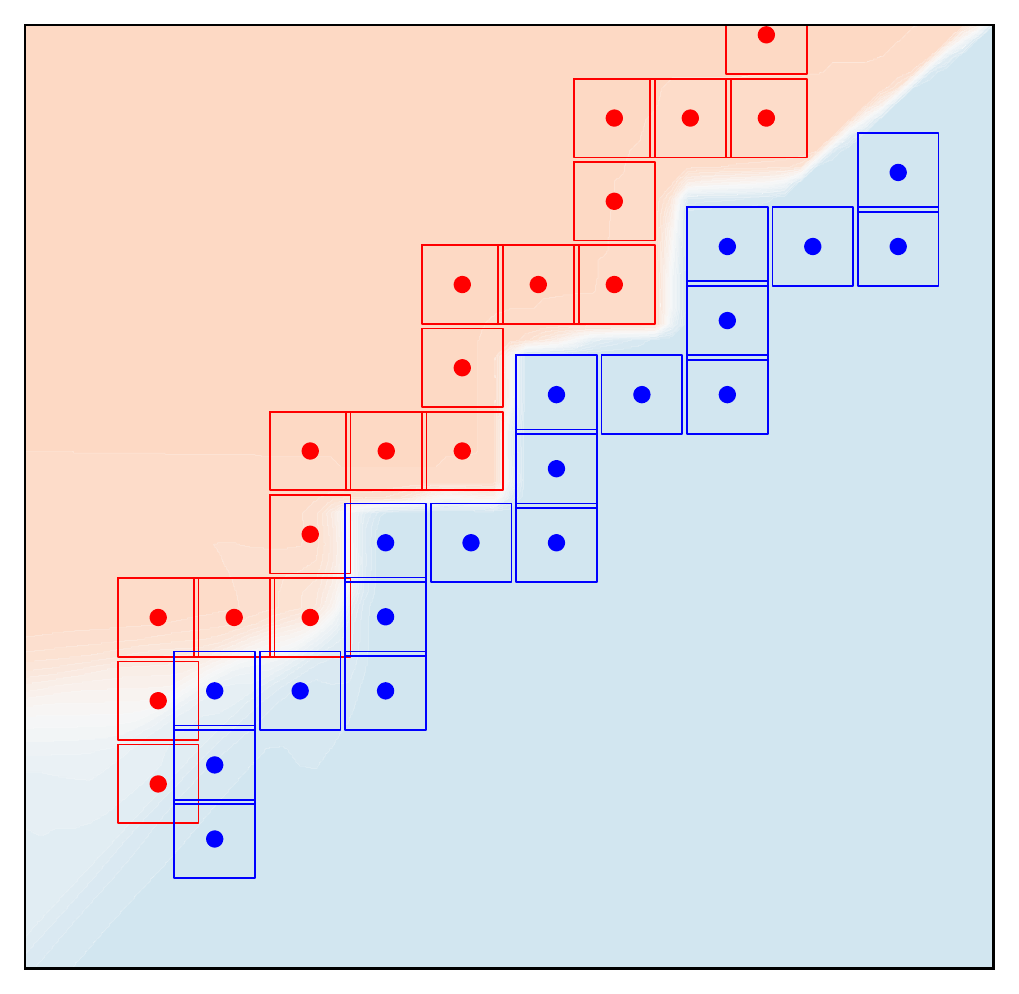}
  }\hspace{.2cm}
  \subfigure{
  \includegraphics[width=.2\linewidth,clip]{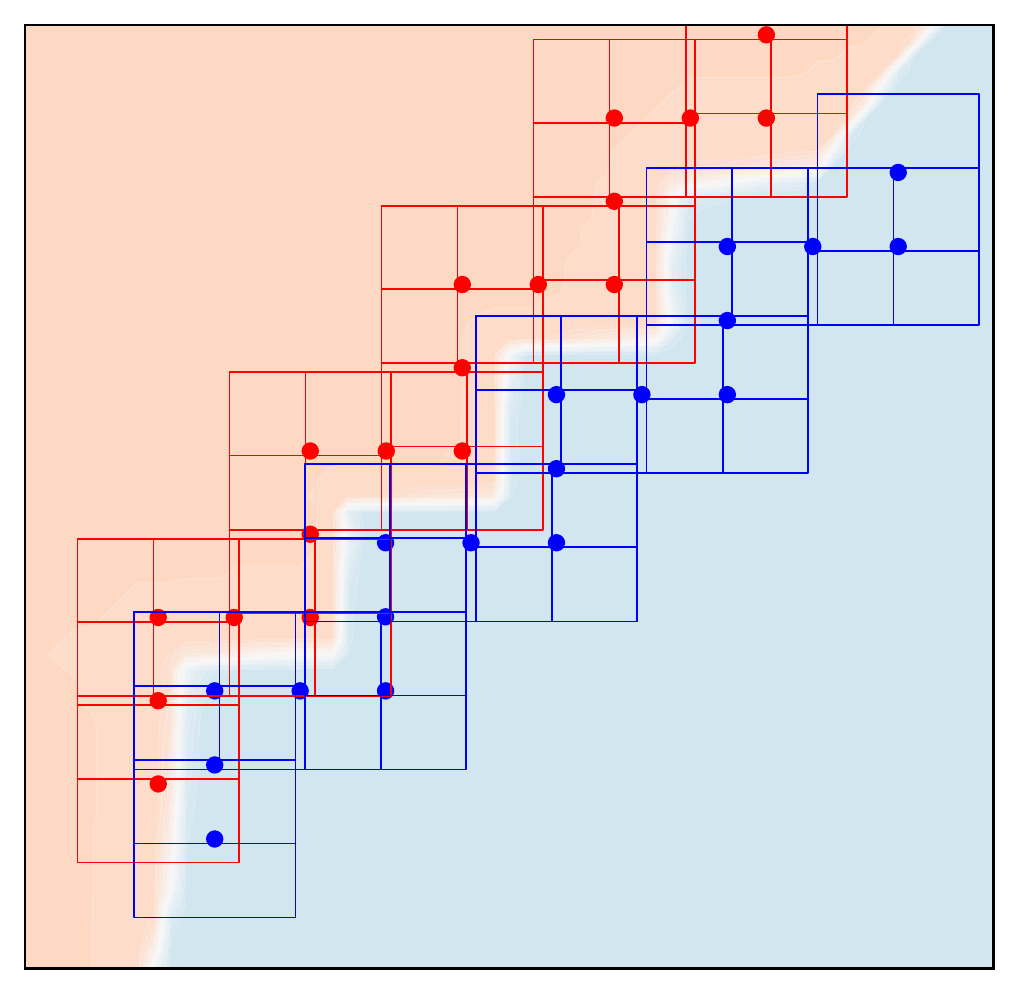}
  }\vspace*{-.4cm}
  \subfigure[Standard training]{ \label{fig:toy_stairs_baseline}
  \includegraphics[width=.2\linewidth]{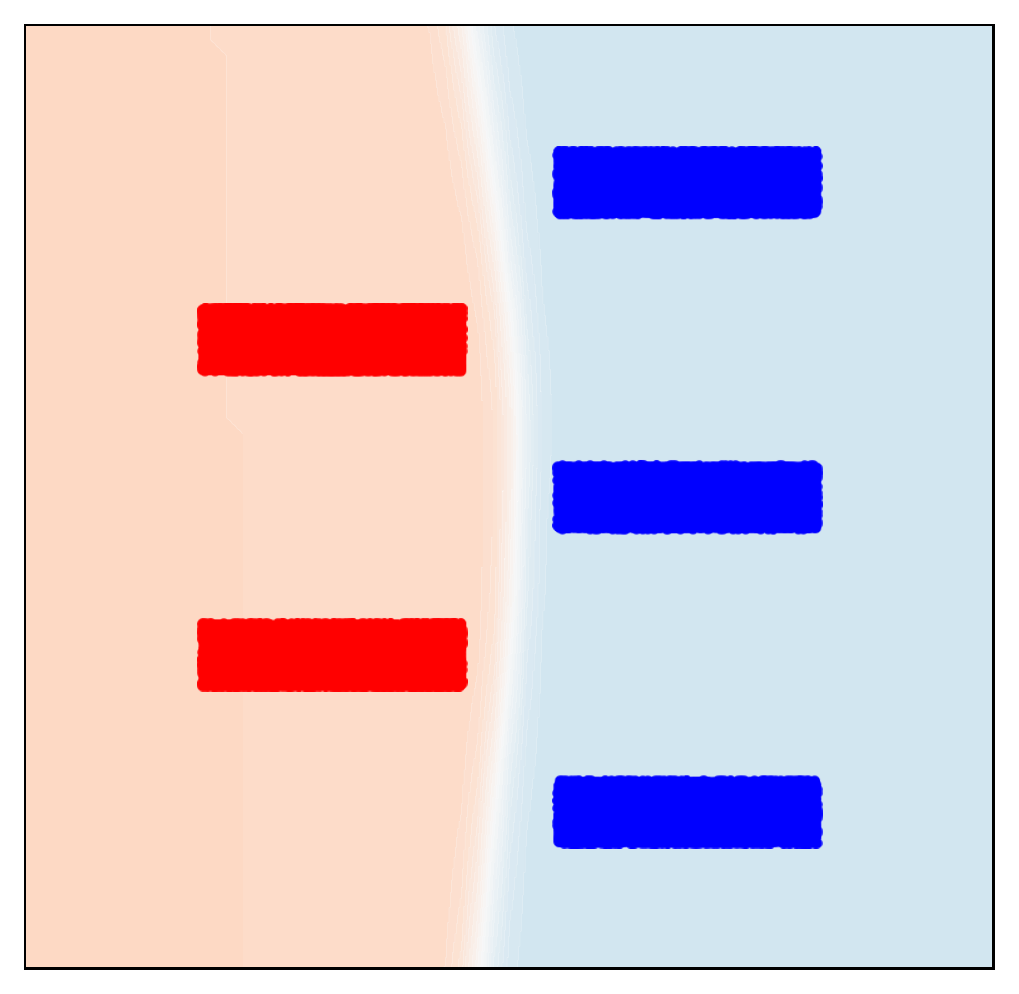}
  \addtocounter{subfigure}{-4} 
  }\hspace{.2cm}
  \subfigure[LDP-PGD]{\label{fig:toy_stairs_pgd}
  \includegraphics[width=.2\linewidth,clip]{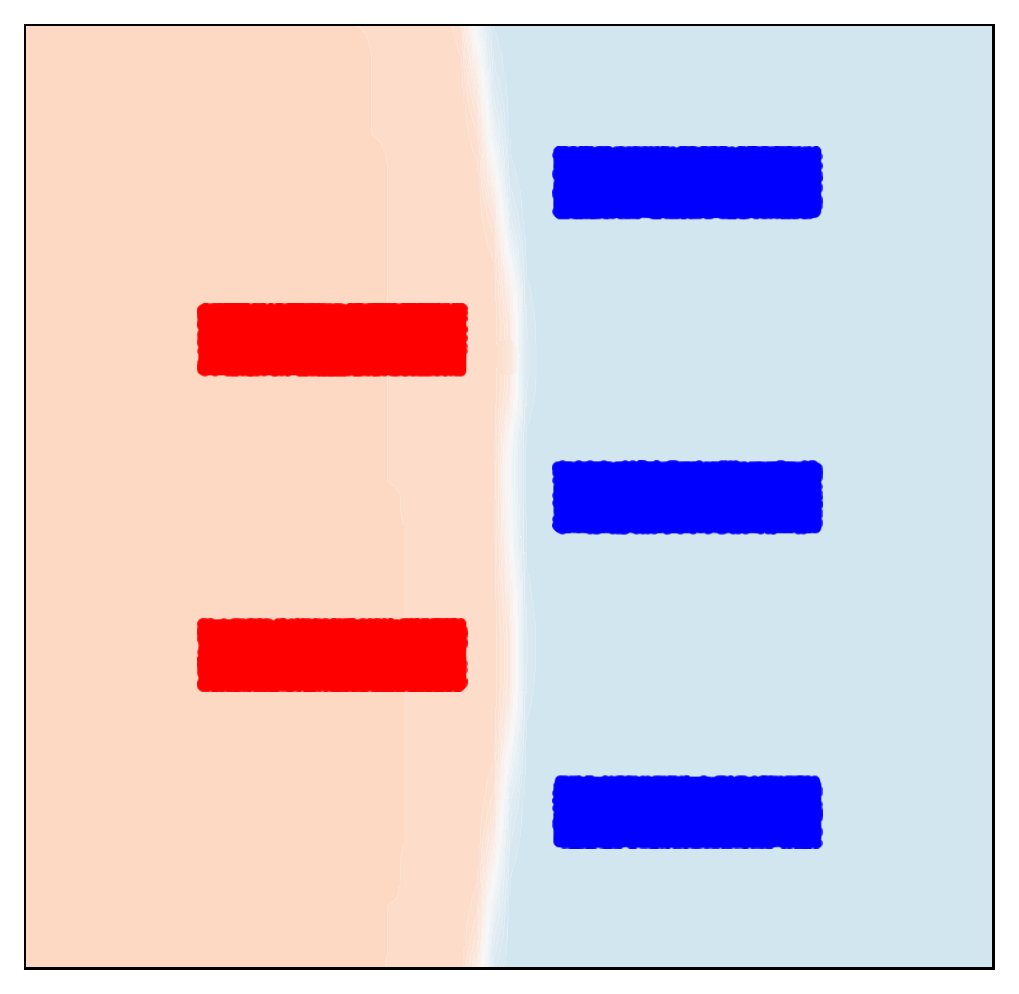}
  }\hspace{.2cm}
  \subfigure[TRADES]{\label{fig:toy_stairs_trades}
  \includegraphics[width=.2\linewidth]{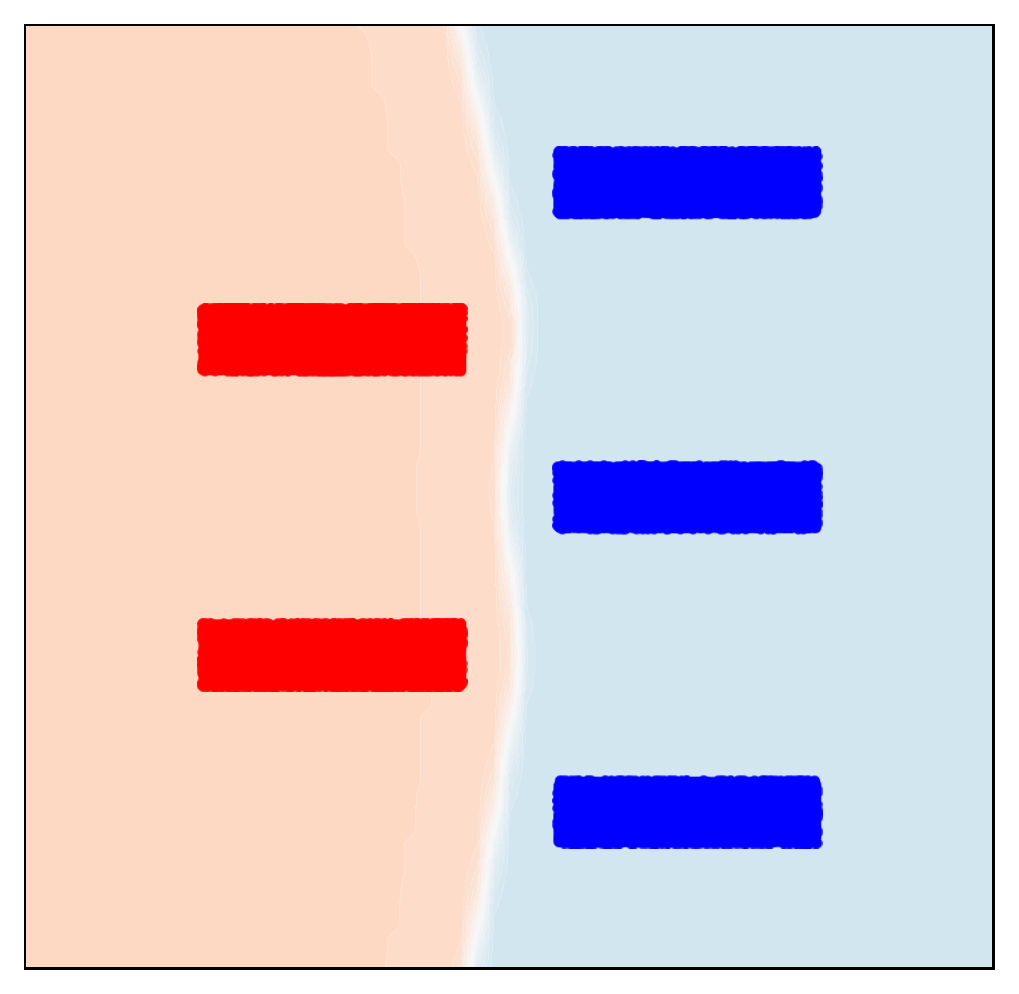}
  }\hspace{.2cm}
  \subfigure[UDP-PGD (ours)]{\label{fig:toy_stairs_udp}
  \includegraphics[width=.2\linewidth,clip]{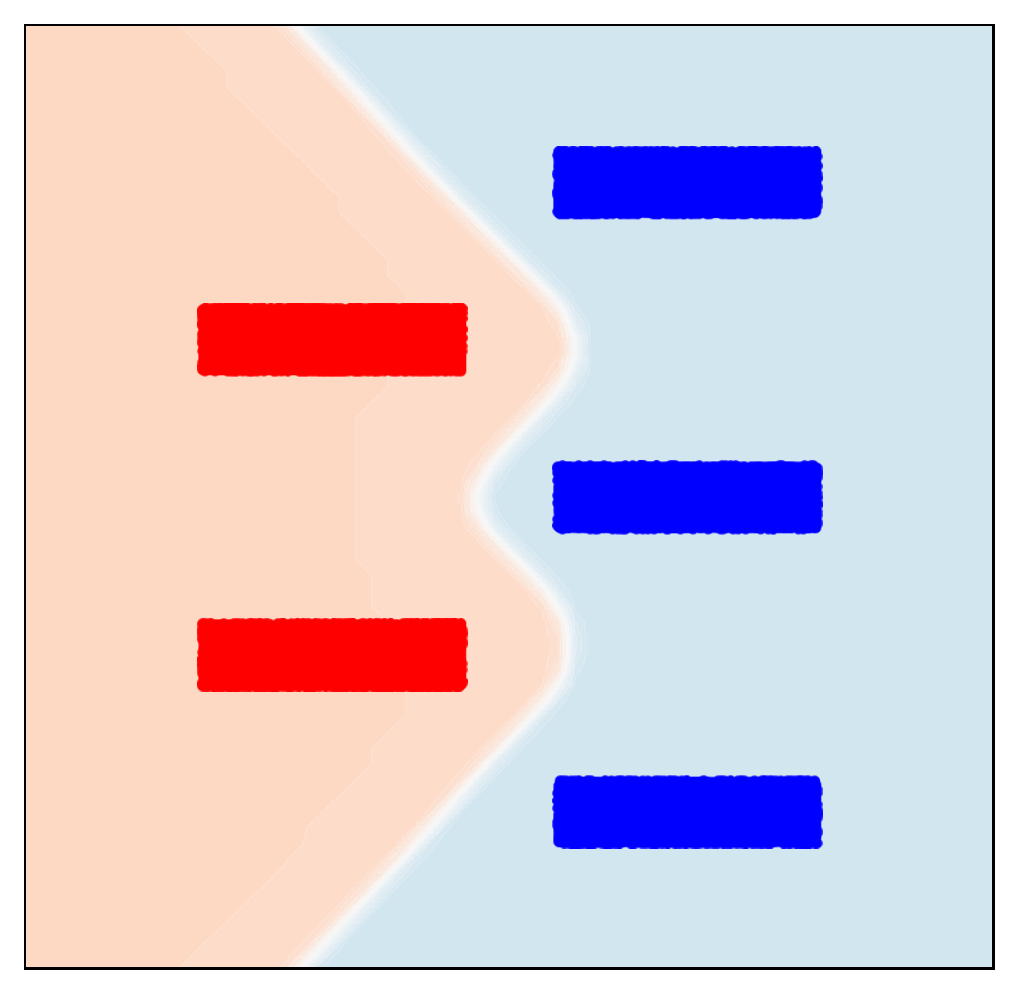}
  }
 \vspace{-.3cm}
 \caption{
 Standard, LDP-PGD, TRADES 
 and UDP training shown in first to fourth column, resp.;  on the \textit{Narrow Corridor} and the \textit{LMS-5} datasets--see \S~\ref{sec:example}, where their color represents the class, and the background the decision boundary after converging. 
 \textbf{Top row.}
 On the NC dataset we observe that 
 \begin{enumerate*}[series = tobecont, itemjoin = \quad, label=(\roman*)]
 \item standard training obtains decision boundary which passes  close to the training data points, 
 \item LDP-PGD and TRADES fail in regions where $\eps_{train}$ violates the \ref{eq:at_assumption} assumption locally 
 \item UDP-PGD despite using an $\eps$ twice as large, can recover a boundary close to the max-margin. 
 \end{enumerate*}
 \textbf{Bottom row.}
 On \textit{LMS-5}, training with LDP-PGD or TRADES using the largest possible $\eps$ under assumption \ref{eq:at_assumption} is not enough to recover the max-margin boundary decision. UDP with a large $\eps$, however, retrieves the optimal boundary. See \S~\ref{sec:example} for discussion.
 }\label{fig:toy_nc_lms5}
\end{figure*}

\section{Motivating examples}\label{sec:example}

We consider two binary classification datasets in 2D, depicted in Fig.~\ref{fig:toy_nc_lms5}, see App.~\ref{app:toy} for additional datasets.
Inspired from that of~\citep{ZhangTRADES}, we design a dataset that has non-isotropic distances between the classes, i.e., the distance of the training samples to the optimal decision boundary increases from bottom left to top-right,
which we name as the \textit{Narrow Corridor} (NC) dataset.
The other dataset is that of~\citep{shah2020pitfalls} called \textit{linear \& multiple 5-slabs} (LMS-5), which is challenging because a single feature (the $x$-coordinate) suffices to obtain $0$ training loss, but the margin of such boundary is sub-optimal.      
The two datasets have in common that the standard training \textit{fails} to learn a decision boundary with a good margin--see first column in Fig.~\ref{fig:toy_nc_lms5}, which motivates training with input perturbations.
From the top row of Fig.~\ref{fig:toy_nc_lms5} we observe that satisfying~\ref{eq:at_assumption} ``globally'' is restrictive for the NC dataset and LDP training both for TRADES and PGD: 
\begin{enumerate*}[series = tobecont, itemjoin = \quad, label=(\roman*)]
\item choosing an $\eps_{train}$ that does not violate this assumption gives sub-optimal decision boundary in the top-right corner, and on the other hand
\item selecting value for $\eps_{train}$ that violates ~\ref{eq:at_assumption} only in some small regions (in bottom-left) causes miss-classifying the \textit{training} samples in that region.
\end{enumerate*}
In the presence of strong tendency toward the simplicity bias phenomenon--bottom row of Fig.~\ref{fig:toy_nc_lms5} we observe LDP has limited capacity to improve the decision boundary--as shown in~\citep{shah2020pitfalls}.

\noindent\textbf{Relevance for real-world datasets.}
Fig.~\ref{fig:historgram_distances_cifar10} shows the distances between the closest data samples of opposite class on CIFAR-10 in input space, where we observe that these distances are non-isotropic.  
In particular, choosing an $\eps_{train}$ that satisfies the~\ref{eq:at_assumption} is restrictive: there exist a large mass of data samples between which we can not guarantee the margin will be satisfactory after training with  perturbations. On the other hand, when training with a larger $\eps$--as demonstrated in this section--\ref{eq:ldp} can have undesirable behavior such as low accuracy in the region where~\ref{eq:at_assumption} is violated, or oscillations--see also Fig.~\ref{fig:toy_oscil}--making the training inefficient. 
Since the margin and data points distances in input space are positively correlated to those in the latent space--see~\citep{sokolic2017margin}, this indicates that the $\eps$-sensitivity of LDP may be relevant for real-world datasets. 

\section{Uncertainty Driven Perturbations}\label{sec:method}

We propose \textit{Uncertainty Driven Perturbations} (UDP) which method, as the name implies, aims at finding perturbations 
that \emph{maximize the model's uncertainty estimate} as follows: \vspace*{-.1cm}
\begin{equation}\tag{UDP}\label{eq:udp}
\begin{aligned}
    \vdelta_u = \argmax_{\vdelta\in \Delta} \mathcal{H}( \mathcal{E}(\vx) + \vdelta, \vomega) \,,
\end{aligned}
\end{equation}
where we use same notation as for~\eqref{eq:ldp}, to include the possibility for latent-space perturbations--see~\S~\ref{sec:preliminaries}.  
When using a \textit{single} model ($M=1$),~\ref{eq:udp} can be seen as \textit{maximum entropy} perturbations.

\subsection{Uncertainty Driven Perturbations (UDP)}\label{sec:method-details}

Similar to loss-based~\ref{eq:at_erm}, the inner maximization of~\ref{eq:udp}---for both image and latent-space perturbations---can be implemented analogously to~\ref{eq:pgd} and~\ref{eq:fgsm} as follows (with $\vdelta_{\text{PGD}}^{0} \triangleq \mathbf{0}$): 
\vspace*{-.2cm}
\begin{equation}\tag{UDP-PGD}\label{eq:udp-pgd}
    \vdelta_{\text{u}}^{i} \triangleq  \mathop{\Pi}\limits_{\|\cdot\|_{\infty} \leq \epsilon} \Big( \vdelta_{\text{UPD}}^{i-1} + \alpha \cdot \text{Sign}\big(\mathop{\nabla}\limits_{\vx} \mathcal{H}(\mathcal{E}(\vx)+ \vdelta_{\text{UPD}}^{i-1}, \vomega)\big)
    \Big) \,.
\end{equation}
Inspired from~\citep{ZhangTRADES}, we also define a variant of UDP which does UDP training with regularization term (UDPR) as follows:
\begin{equation}\tag{UDPR}\label{eq:udpr}
\begin{aligned}
\min_{\omega} 
\displaystyle \mathop{\mathbb{E}}_{(\vx,\vy)\sim p_d} 
\big[ \mathcal{L}(\mathcal{C}_\vomega(\vx), \vy) + \lambda \mathcal{L}( \mathcal{C}_\omega(\vx+\vdelta_u),\vy)  \big]\,,
\end{aligned} 
\end{equation}
with $\lambda \in (0,1)$, and $\vdelta_u$ as in ~\ref{eq:udp}. 
To show convergence on linear models, it is necessary that the perturbed samples $\tilde\vx_i$ are distributed between the original data points $\vx_i$, see \S~\ref{sec:udp_conv}. Rather then running the maximization for $k$ steps--as per~\eqref{eq:pgd} and then project back in a way that makes $\tilde\vx_i$ uniformly distributed over the iterations, it is more computationally efficient to instead sample the number of steps uniformly in the interval $(1,k-1)$ at each iteration, and for each sample.
Alg.~\ref{alg:upd} summarizes the details of the two variants, where for simplicity, we use a single model $M=1$. 

\begin{algorithm}[htb]
    \begin{algorithmic}[1]
       \STATE {\bfseries Input:} 
                  Classifier $C_{\vomega_0}$ with \textit{logits} output and initial weights $\vomega_0$,
                  stopping time $T$, data distribution $p_d$,
                  learning rate $\eta$,  its loss $\LL$, 
                  $L_\infty$ ball radius $\eps$, perturbation step size $\alpha$, number of attack iterations $K$,
                  Boolean flag--train with a regularizer $reg$, regularizer weight $\gamma$.

       \FOR{$t \in 0, \dots, T{-}1$}
            \STATE \textbf{Sample} $\vx, \vy \sim p_d$
            \STATE $\vdelta_{u}^0 \leftarrow \mathbf{0}$
            \STATE \textbf{sample} $k \sim \mathcal{U}(1, K{-}1)$ 
            \FOR{$i \in 0, \dots, k$} 
                \STATE $\vp \leftarrow \text{Softmax}(\mathcal{C}_{\vomega_t}(\vx+\vdelta_u^i))$
                \STATE $\mathcal{H} \leftarrow - \sum_c \vp_c \log(\vp_c)$ \hfill\emph{(Compute entropy)} 
                \STATE $\vdelta_{u}^{i+1} = \vdelta_{u}^i + \alpha \sign(\nabla_{\vdelta_u^i} \mathcal{H})$  
                \STATE \textbf{Projection} $\vdelta_{u}^{i+1} \leftarrow \mathop{\Pi}\limits_{\|\cdot\|_{\infty} \leq \eps} \vdelta_{u}^{i+1}$
            \ENDFOR
            
            \IF{$reg$ is \texttt{True}}
                \STATE $\vomega_{t+1} {=} \vomega_t - \eta  \nabla_{\vomega}\mathcal{L}(\mathcal{C}_{\vomega_t}(\vx), y) 
                - \gamma \nabla_{\vomega}\mathcal{L}(\mathcal{C}_{\vomega_t}(\vx+ \vdelta_{u}^K), y)$
            \ELSE
                \STATE $\vomega_{t+1} = \vomega_t - \eta  \nabla_{\vomega}\mathcal{L}(\mathcal{C}_{\vomega_t}(\vx+ \vdelta_{u}^K), y)$ 
            \ENDIF

       \ENDFOR 
       \STATE {\bfseries Output:} $\vomega_T$   
    \end{algorithmic}
   \caption{UDP pseudocode (in input space). 
   }
   \label{alg:upd}
\end{algorithm}

\subsection{On the convergence of UDP}\label{sec:udp_conv}

\noindent\textbf{Setting.}
Consider a setting of binary classification of $x\in \R$ data samples, sampled from the data distribution $p_d$  which represents a mixture of two Gaussians with means $\mu^1, \mu^2$ (with $\mu^1 < \mu^2$), 
with corresponding label $y^1=-1$ and $y^2=+1$. 
For this problem, the optimal boundary is at $\omega^\star = 0$. Since SGD yields the max-margin classifier for this setting (and so does regular training--see \S~\ref{sec:intro}), to represent more closely non-linear systems we consider a weaker assumption for the optimizer:  at each iteration $n$, given data points $\tilde x^1_n$ and $\tilde x^{(2)}_n$, an oracle returns a \textit{direction} $\tilde\omega_n$ that belongs in the interval between $\tilde x^1_n$ and $\tilde x^{(2)}_n$, that is:
\vspace*{-.1cm}
\begin{equation}\tag{S1}\label{eq:s1}
    \exists \alpha_n \in (0,1), \quad \text{s.t.}\quad  \tilde\vomega_n = \alpha_n \tilde\vx^1_n + (1-\alpha_n)\tilde\vx^{(2)}_n \,,
\end{equation}
where we assume: \vspace*{-.2cm}
\begin{equation}\tag{A1}\label{eq:a1}
   \alpha_n \sim \mathcal{U} (0,1), \quad \text{i.e.} \quad
    \E[\alpha] = \frac{1}{2}\,,
\end{equation}
Given $\vomega_n$ and $\tilde\vomega_n$, the update rule is as follows:
\begin{equation}\tag{S2}\label{eq:s2}
    \vomega_{n+1} = \vomega_n + \eta (\tilde\vomega_n-\vomega_n) \,, 
\end{equation}
where $\eta\in(0,1)$ is the step-size. \looseness=-1

\begin{theorem}[Convergence towards the max-margin without $\eps$ dependence] 
\label{thm:udp_simple_lin_separable}
Given an oracle that takes as input two data points $\tilde\vx^{(1)},\tilde\vx^{(2)} \in \R$ and returns a point that lies between these points sampled uniformly at random, the UDP method as per Alg.~\ref{alg:upd} in expectation tends to move towards the max-margin classifier on a linearly separable dataset. 
\end{theorem}

The proof is provided in App.~\ref{app:proofs}.
Although the problem is simple and standard training converges to the max-margin solution, it is worth noting that LDP fails to converge if $\eps>|\frac{\mu_1 - \mu_2}{2}|$, in which case the trained classifier can have $0$ accuracy, as the labels of the perturbed training samples are flipped. 
This shows that stopping at the current boundary is beneficial.
Our illustrative examples with non-isotropic distances between  different class samples, further illustrate why such sensitivity to $\eps$ is relevant.  

\noindent\textbf{Revisiting the 2D motivating examples.}
We observe from Fig.~\ref{fig:toy_nc_lms5}-\subref{fig:toy_stairs_udp} that by replacing the perturbation objective with the model's estimated uncertainty we are able to use larger $\eps_{train}$ for the NC dataset, and interestingly we obtain a decision boundary (close to) the max-margin.
Similarly, for the challenging high simplicity bias dataset in the second row,~\ref{eq:udp} achieves a decision boundary (close to) that of the max-margin--see Fig.~\ref{fig:lms-max-margin}.
App.~\ref{app:toy} provides the reader with additional insights how~\ref{eq:udp} works using additional datasets as well as a different uncertainty estimation model, in particular one based on Gaussian Process~\citep{rasmussen2005gp}.

\noindent\textbf{Additional advantages of UDP.}
\ref{eq:udp} is \textit{unsupervised}, as it does not use the ground truth labels $\vy$ to compute the perturbations.
As such, it can be used for unsupervised methods that output probability estimates.
On the other hand,~\ref{eq:ldp} relies on the class label of the starting training sample $\vx_i$, thus the loss increases towards the decision boundary and the further we are from $\vx_i$ the larger it increases, see Fig.~\ref{fig:illustration_udp_ldp}.

\begin{table*}[htb]
\caption{Results on CIFAR-10 and SVHN \emph{without} data augmentation including best robust accuracy for $l_\infty$ and $l_2$ as given by the AutoAttack standard pipeline, as well as the best trade-off between generalization and robustness given by averaging the clean and robust accuracy. For each method, the set of hyperparameters giving the best average $l_\infty$ robustness between $\eps_\text{train} = 0.01$ and $\eps_\text{train} = 0.03$ is selected. The best scores along with all scores within standard deviation reach for each column are in bold. We observe how large $\eps_\text{train}$ tend to be selected for UDP methods. }
\label{tab:cifar10-svhn-table}
\vspace*{-.2cm}
\begin{center}
\begin{small}
\begin{tabular}{clcccc}
\toprule
 & & \multicolumn{4}{c}{CIFAR-10} \\
\midrule
  & & \multicolumn{2}{c}{$l_\infty$} & \multicolumn{2}{c}{$l_2$} \\
  & method & $\eps_{\text{test}}=0.01$ & $\eps_{\text{test}}=0.03$ & $\eps_{\text{test}}=0.5$ & $\eps_{\text{test}}=1.0$ \\
\midrule
              & TRADES ($\eps_\text{train}=0.03$)   & $61.8\pm0.8$ & $31.0\pm0.3$ & $46.9\pm1.2$ & $18.3\pm0.7$ \\
  Best Robust & PGD ($\eps_\text{train}=0.02$)   & $62.7\pm0.5$ & $31.2\pm0.3$ & $47.4\pm0.6$ & $17.9\pm0.5$ \\
  Accuracy    & UDPR ($\eps_\text{train}=0.05$)  & $\mathbf{64.9\pm0.4}$ & $31.1\pm0.3$ & $\mathbf{50.2\pm0.5}$ & $\mathbf{19.1\pm0.4}$ \\
              & UDP-PGD ($\eps_\text{train}=0.06$)      & $64.2\pm0.4$ & $\mathbf{33.4\pm0.3}$ & $\mathbf{50.1\pm1.8}$ & $\mathbf{19.5\pm1.4}$ \\ \midrule
              & TRADES ($\eps_\text{train}=0.02$)   & $71.3\pm0.4$ & $54.9\pm0.4$ & $63.4\pm0.8$ & $47.6\pm0.2$ \\
  Best        & PGD ($\eps_\text{train}=0.01$)   & $\mathbf{72.7\pm0.4}$ & $52.6\pm0.2$ & $64.4\pm0.3$ & $47.7\pm0.4$ \\
  Trade-off   & UDPR ($\eps_\text{train}=0.05$)  & $72.3\pm0.4$ & $55.4\pm0.2$ & $65.0\pm0.1$ & $\mathbf{49.4\pm0.3}$ \\
              & UDP-PGD ($\eps_\text{train}=0.04$)      & $\mathbf{73.1\pm0.6}$ & $\mathbf{55.7\pm0.8}$ & $\mathbf{65.8\pm0.3}$ & $48.9\pm0.3$ \\ \midrule
 & & \multicolumn{4}{c}{SVHN} \\ \midrule
              & TRADES ($\eps_\text{train}=0.05$)   & $72.9\pm1.5$ & $\mathbf{44.6\pm1.6}$ & $\mathbf{37.7\pm2.5}$ & $\mathbf{6.7\pm0.8}$ \\
  Best Robust & PGD ($\eps_\text{train}=0.04$)      & $75.7\pm0.2$ & $43.3\pm0.4$ & $30.5\pm1.0$ & $5.5\pm0.1$ \\
  Accuracy    & UDPR ($\eps_\text{train}=0.05$)     & $\mathbf{79.7\pm0.1}$ & $40.3\pm1.4$ & $34.4\pm1.5$ & $\mathbf{6.1\pm0.8}$ \\
              & UDP-PGD ($\eps_\text{train}=0.04$)  & $79.3\pm0.1$ & $39.7\pm0.7$ & $33.4\pm0.6$ & $5.8\pm0.6$ \\ \midrule
              & TRADES ($\eps_\text{train}=0.05$)   & $79.4\pm1.2$ & $65.2\pm0.3$ & $61.8\pm0.8$ & $46.3\pm0.1$ \\
  Best        & PGD ($\eps_\text{train}=0.02$)      & $85.6\pm0.3$ & $64.6\pm0.3$ & $61.9\pm0.2$ & $48.9\pm0.1$ \\
  Trade-off   & UDPR ($\eps_\text{train}=0.05$)     & $\mathbf{86.4\pm0.0}$ & $\mathbf{67.1\pm0.4}$ & $\mathbf{63.8\pm0.7}$ & $\mathbf{49.6\pm0.2}$ \\
              & UDP-PGD ($\eps_\text{train}=0.04$)  & $\mathbf{86.5\pm0.1}$ & $\mathbf{66.7\pm0.4}$ & $\mathbf{63.6\pm0.2}$ & $\mathbf{49.8\pm0.2}$ \\
\bottomrule
\end{tabular}
\end{small}
\end{center}
\vskip -0.1in
\end{table*}

\section{Experiments}\label{sec:exp}

As~\ref{eq:udp} is a general framework where the objective for finding perturbations is estimated uncertainty, 
we conduct several types of experiments to empirically verify if such an approach is promising for real world datasets as well.
However, we do not aim to provide new state of art results in related fields,  and in \S~\ref{sec:discussion} we list some open directions. 
Primarily,
\begin{enumerate*}[series = tobecont, itemjoin = \quad, label=(\roman*)]
\item given a fixed model capacity we compare the generalization-robustness trade-off of UDP in input and in latent space---see \S~\ref{exp:tradeoff},
\item the model's generalization with increasingly larger capacity--\S~\ref{sec:model_capacity}, and finally 
\item we estimate the simplicity bias by analyzing the tranferability of the learned features from a source (CIFAR-10) to a different target domain (CIFAR-100) in \S~\ref{sec:simplicity_bias}. 
\end{enumerate*}

\noindent\textbf{Datasets \& models.} 
We evaluate on Fashion-MNIST~\citep{fashionmnist}, SVHN~\citep{svhn}, CIFAR-10~\citep{cifar10}  and CIFAR-100~\citep{Krizhevsky09cifar100}. 
Throughout our experiments, we use two different architectures: LeNet~\citep{Lecun98gradient} for Fashion-MNIST, as well as ResNet-18~\citep{resnet} for CIFAR-10 and SVHN. For all experiments in the main paper, we use a single model for UDP, for fair comparison with the baselines. See App.~\ref{app:implementation} for details on the implementation.

\begin{table}[htb]
\caption{AutoAttack robustness (top) and trade-off (bottom) on Fashion-MNIST, where for the latter we average the clean and robust accuracy. For each method, one set of hyperparameters is selected according to its mean $l_\infty$-robust accuracy over $\eps_\text{test}\in \{0.05,0.1,0.2\}$. 
Notice that the best scores for UDP are obtained for large $\eps_\text{train}$--often $0.5$--which clearly violates~\ref{eq:at_assumption}.
Results are averaged over $3$ seeds.}
\label{tab:robust-acc-fmnist}
\vspace*{-.2cm}
\begin{center}
\begin{small}
\begin{tabular}{lccccc}
\toprule
\multicolumn{6}{c}{Best Robust Accuracy} \\ \midrule
 & clean & \multicolumn{3}{c}{$l_\infty$}  & \multicolumn{1}{c}{$l_2$} \\ 
$\eps_{\text{test}}$.            & acc    & $0.05$ & $0.1$  & $0.2$  &  $1.5$ \\ \midrule
PGD ($\eps_\text{train}.2$)      &  $78.9$& $73.9$ & $69.3$ & $58.3$ & $2.8$  \\
TRADES ($\eps_\text{train}.2$)  &  $82.1$ & $77.1$ & $70.6$ & $55.1$ & $5.6$  \\
UDPR ($\eps_\text{train}.5$)     &  $84.8$& $72.5$ & $60.9$ & $40.0$ & $34.1$ \\
UDP-PGD ($\eps_\text{train}.5$)  &  $82.3$& $74.9$ & $66.8$ & $49.6$ & $27.5$ \\ \midrule
\multicolumn{6}{c}{Best Trade-off} \\\midrule
$\eps_{\text{test}}$.            & acc    & $0.05$ & $0.1$  & $0.2$  &  $1.5$ \\ \midrule
PGD ($\eps_\text{train}.2$)      &  $78.9$& $75.5$ & $73.5$ & $67.9$ & $40.2$  \\
TRADES ($\eps_\text{train}.2$)  &  $82.1$& $79.7$ & $76.5$ & $68.7$ & $44.0$  \\
UDPR ($\eps_\text{train}.5$)     &  $84.8$& $78.6$ & $72.8$ & $62.3$ & $59.4$ \\
UDP-PGD ($\eps_\text{train}.3$)  &  $84.2$& $80.2$ & $75.6$ & $65.4$ & $53.7$ \\ 
\bottomrule
\end{tabular}
\end{small}
\end{center}
\vskip -0.05in
\end{table}

\subsection{Robustness-generalization trade-off for fixed model}\label{exp:tradeoff}
As UDP aims to increase the classification margin, $l_p$ robustness may be a suitable test bench.Indeed, if the margin for a classifier $\mathcal{C}_\vomega$ is more than $\eps$, then there are no perturbations $x+\delta$ constrained to $\|\delta\|_2<\eps$ that can be found such that $\mathcal{C}_\vomega(x+\delta)$ crosses the boundary. 
Thus, we evaluate the $l_2$ as well as the $l_\infty$ robustness for several datasets and $\eps_{\text{test}}$. We also compute the robustness-generalization trade-off by averaging robust and clean accuracy.
Results on catastrophic overfitting are available in App.~\ref{app:results}. 

\noindent\textbf{Methods \& robustness evaluation.} 
We compare:  
\begin{enumerate*}[series = tobecont, itemjoin = \quad, label=(\roman*)]
\item \textbf{\ref{eq:udp}},
\item \textbf{\ref{eq:pgd}}~\citep{madry2018towards}, and
\item \textbf{\ref{eq:trades}}~\citep{ZhangTRADES},
\end{enumerate*}
the latter two are described \S~\ref{sec:preliminaries}.
For all our $l_\infty$ and $l_2$-robustness values, we use the standard pipeline of \textit{AutoAttack} \citep{autoattack}, which considers four different attacks: $\text{APGD}_{\text{CE}}$, $\text{APGD}_{\text{DLR}}^{\text{T}}$, $\text{FAB}^{\text{T}}$, and Square.

\subsubsection{Input space}\label{sec:input_exp}

\begin{table}[t]
\caption{Results on CIFAR-10 \emph{with} data augmentation. For each method, we trained models with $l_\infty$-norm and $\eps_{\text{train}}$ ranging from $0.01$ to $0.08$. We pick one model to display based on the best average $l_\infty$ robustness on $\eps_\text{test} = 0.01$ and $\eps_{\text{test}} = 0.03$. We can observe how the best $\eps_\text{train}$ for UDP methods tends to be large. All robust accuracies are computed using the AutoAttack standard pipeline.}
\label{table:cifar10-da}
\begin{center}
\begin{small}
\begin{tabular}{lccccc}
\toprule
 & clean & \multicolumn{2}{c}{$l_\infty$}  & \multicolumn{2}{c}{$l_2$} \\ 
$\eps_{\text{test}}$&        acc                   & $0.01$  & $0.03$ & $0.5$  &  $1.0$ \\ \midrule
PGD ($\eps_\text{train}.03$)      & $83.8$ & $72.9$ & $47.6$ & $58.5$ & $25.5$ \\
UDPR ($\eps_\text{train}.05$)     & $83.4$ & $74.1$ & $45.3$ & $58.5$ & $25.5$ \\
UDP-PGD ($\eps_\text{train}.08$)  & $87.5$ & $76.9$ & $37.1$ & $59.5$ & $21.4$ \\
\bottomrule
\end{tabular}
\end{small}
\end{center}
\vskip -0.1in
\end{table}

\begin{table}[thb]
\caption{
Feature transferability: features of CIFAR-10 pre-trained models (from Tab.\ref{tab:cifar10-svhn-table}) are fixed, and we train a classifier on CIFAR-100 using their embeddings. See \S~\ref{sec:simplicity_bias}.
}
\label{table:cifar10-transfer}
\begin{center}
\begin{small}
\begin{tabular}{clc}
\toprule
&method & transfer accuracy\\ \midrule
 &baseline                            & $42.5\pm0.3$ \\ \midrule
 Best  &TRADES ($\eps_\text{train} = .03$)  & $42.7\pm0.1$ \\
Robust     &PGD ($\eps_\text{train} = .02$)     & $42.6\pm0.1$ \\
 Accuracy           &UDP-PGD ($\eps_\text{train} = .06$) & $\mathbf{45.3\pm0.3}$ \\ \midrule
Best        &TRADES ($\eps_\text{train} = .02$)  & $43.1\pm0.3$ \\
Trade-off   &PGD ($\eps_\text{train} = .01$)     & $43.3\pm1.2$ \\
            &UDP-PGD ($\eps_\text{train} = .04$) & $\mathbf{45.7\pm0.6}$ \\ 
\bottomrule
\end{tabular}
\end{small}
\end{center}
\vskip -.1in
\end{table}

\begin{figure}[htb]
  \centering
  \subfigure[SVHN]{\label{subfig:svhn-single_clean}
  \includegraphics[width=.51\linewidth]{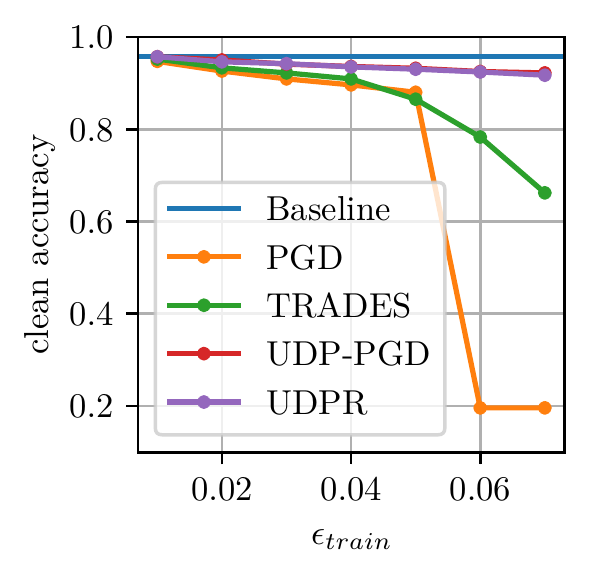}}
  \subfigure[CIFAR-10]{\label{subfig:cifar10-single_clean}
  \includegraphics[width=.45\linewidth,trim={0.9cm .0cm 0cm .0cm},clip]{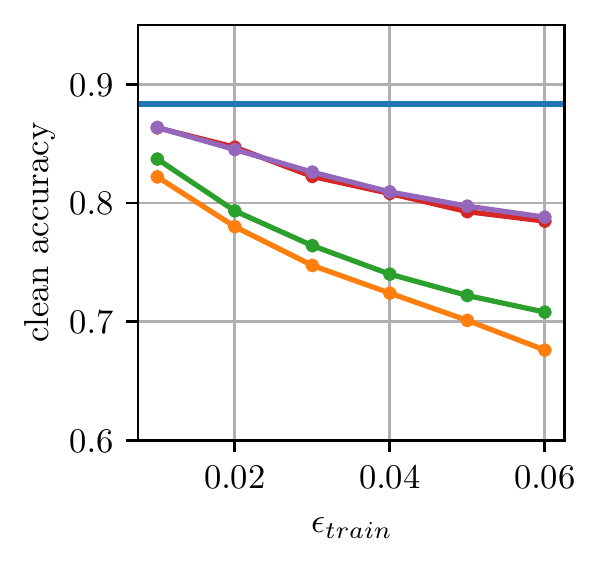}}
  \vspace*{-.2cm}
\caption{ \textit{Clean accuracy} ($y$-axis) comparison between PGD, TRADES, UDP-PGD, and UDPR on \textbf{SVHN} and \textbf{CIFAR-10}, for varying $\eps_\text{train}$ used for training ($x$-axis). Results are averaged over multiple runs. See \S~\ref{exp:tradeoff}.}\label{fig:clean-acc}
\vskip -.5em
\end{figure} 

\noindent\textbf{Setup.} 
For TRADES, UDP, and PGD, we train several models with $\eps_{\text{train}}$ ranging from $0.01$ to $0.08$, all the perturbation at train time are obtained using the $l_\infty$ norm. To find the perturbations, we used step sizes $\alpha$ ranging from $0.001$ to $0.012$ and a number of steps from $10$ to $40$. For the \textit{AutoAttack} evaluation, we use 
$\eps_{\text{test}} \in \{0.01, 0.03\}$ for $l_\infty$-robustness and $\eps_{\text{test}} \in  \{0.5,1.0\}$ for $l_2$-robustness.

\noindent\textbf{Results.} In Fig. \ref{fig:clean-acc}, we compare the clean test accuracy of \ref{eq:pgd}, \ref{eq:trades} and \ref{eq:udp-pgd} on SVHN and CIFAR-10, see App.~\ref{app:results} for Fashion-MNIST. 
We \textit{consistently} observe in all our experiments that the clean accuracy of UDP is is \emph{consistently superior} to the one of LPD-PGD and TRADES, across varying $\eps_\text{train}$. 
Tab.\ref{tab:cifar10-svhn-table}, \ref{tab:robust-acc-fmnist} and~\ref{table:cifar10-da} show that UDP methods offer competitive and often improved trade-off between robustness and accuracy, relative to LPD-PGD and TRADES, in terms of their robust accuracy as given by AutoAttack. We also show UDPR and UDP-PGD provide a better trade-off between robustness and clean accuracy. Those improvements are typically obtained for larger values of $\eps_\text{train}$.

\begin{table}[ht]
\caption{Low-data regime comparison, with $5\%$ of CIFAR-10 training set, and  \textit{latent-space} perturbations  (\textit{LS--} prefix).
}
\label{table:cifar10-ldr-ls}
\begin{center}
\begin{small}
\begin{tabular}{lc}
\toprule
method & clean accuracy\\ \midrule
baseline & $48.2\pm0.9$ \\
LS-PGD ($\eps_\text{train} = 2$)     & $37.3\pm0.3$ \\
LS-TRADES ($\eps_\text{train} = 2$)  & $46.9\pm0.3$ \\
LS-UDPR ($\eps_\text{train} = 2$)    & $\mathbf{53.1\pm0.9}$ \\
LS-UDP-PGD ($\eps_\text{train} = 2$) & $\mathbf{52.5\pm0.3}$ \\
\bottomrule
\end{tabular}
\end{small}
\end{center}
\vskip -0.2in
\end{table}

\begin{figure}[ht]
  \centering
  \includegraphics[width=.6\linewidth]{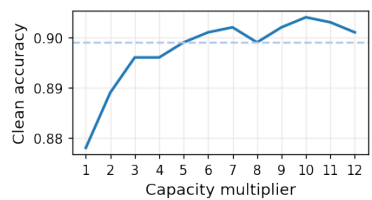}
  \vspace{-.2cm}
\caption{\textit{Clean accuracy} ($y$-axis) comparison between models trained with UDP using $\eps=0.1$, for varying capacity of the model ($x$-axis). Dashed line represents the clean accuracy of standard training and corresponds to model capacity 1--however, has $\sim0\%$ robustness accuracy.
}\label{fig:cap-vs-clean-acc}
\end{figure} 

\subsubsection{Latent space \& low-data regime}\label{sec:latent_exp}

We consider the low-data regime, to study if fewer samples suffice to learn a decision boundary that has a good margin between the considered training samples. As in~\citep{Yuksel2021SemanticPW} we take $5\%$ of the training data of CIFAR-10. We apply the perturbations in the latent space of a ResNet-18 model by splitting it in an encoder part $\mathcal{E}$ (up to the last convolutional layer) and a classifier part $\mathcal{C_{\omega}}$ (the rest of the network)--see \ref{eq:at_erm}. 
Both the encoder and the classifier are trained jointly with perturbations, using $\eps=0.2$, $\alpha=0.01$ and $20$ attack iterations.
Tab.~\ref{table:cifar10-ldr-ls} indicates that the UDP methods lead to better clean accuracy, relative to LDP and TRADES.

\subsection{Increasingly large model capacity}\label{sec:model_capacity}

In the previous section, we observed that the clean accuracy decreases for UDP as $\eps_\text{train}$ increases--although notably less than LDP. As pointed out by~\citet{madry2018towards}, and recently shown by~\citet{bubeck2021law} larger model capacity may be necessary for robustness. 
Thus, in this section, we increase the model capacity to observe if the test accuracy of UDP increases, see App.~\ref{app:model_capacity_implementation} for implementation. 
From Fig.~\ref{fig:cap-vs-clean-acc} we observe that this is indeed likely the case: with increased capacity, we outperform the clean training in terms of clean test accuracy, while having significantly improved robustness--see Tab.~\ref{tab:robust-acc-fmnist} for the corresponding robustness.

\subsection{Estimating simplicity bias}\label{sec:simplicity_bias}
To estimate if the simplicity bias is reduced---if the model maintains more exhaustive set of features, we use the embedding (while keeping it fixed) of the previously trained models on CIFAR-10, to train a model on CIFAR-100. Tab.~\ref{table:cifar10-transfer} indicates that UDP has reduced simplicity bias.

\section{Discussion}\label{sec:discussion}

We proposed using the model's estimated uncertainty as an objective to iteratively find perturbations of the input training samples called UDP.
Moreover, we provided some geometrical explanations on the difference between loss-driven perturbations (LDP) around the decision boundary, and we showed that such an advantage allows UDP to use a potentially larger magnitude of perturbation.
Our preliminary, yet diverse, set of experiments showed that this approach is promising for real-world datasets as well.

We hope to initiate the study of the inherent trade-offs for varying perturbation schemes using online learning tools, to provide more rigorous bounds on the worst-case performance to any adversary. This tackles generalization in a broader sense, as relevant per safety-critical applications. On the empirical side, while we focused on single-model UDP for a fair comparison, it remains an interesting direction to explore UDP with different uncertainty estimation methods. 

\subsubsection*{Acknowledgements}
TC thanks the support by the Swiss National Science Foundation (SNSF), grant P2ELP2\_199740.
The authors thank Yixin Wang for insightful discussions and feedback.
\FloatBarrier
\bibliography{main}
\bibliographystyle{icml2022}
\newpage
\appendix
\onecolumn

\section{Extended Overview of Related Works and Discussions}\label{app:related}

The proposed method is inspired by several different lines of works, such as for example uncertainty estimation, works showing that low robustness of a data sample is correlated with high uncertainty estimate for that sample, training with input perturbations (or using these as a regularizer), among others. This section gives a brief overview of some relevant works, in addition to those listed in \S~\ref{sec:intro} and~\ref{sec:preliminaries} of the main paper. 

\subsection{Brief background on uncertainty estimation}\label{app:uncertainty_related}

In  this section, we give an overview of the most popular uncertainty estimation (UE) methods.

\paragraph{Gaussian Processes (and extensions).}
As the gold-standard UE method is considered the Gaussian processes~\citep[GPs, ][]{rasmussen2005gp} initially developed for regression tasks. 
GPs are a flexible Bayesian non-parametric models where the similarity between data points is encoded by a distance aware kernel function.
More precisely, GPs model the similarity between data points
with the kernel function, and use the Bayes rule to model a \textit{distribution over functions} by maximizing the marginal likelihood~\citep{rasmussen2005gp} of the Bayes formula.
Thus, GPs require access to the full datasest at inference time. 
Accordingly, this family of methods and their approximations are computationally expensive and do not scale with the dimension  of the data.

To use the advantages of both worlds:
\textit{(i)} the fast inference time of neural networks, and
\textit{(ii)} the good UE performances of GPs,
a separate line of works focuses on combining neural networks with GPs. 
\citeauthor{amersfoort2020}, \citeyear{amersfoort2020} combine Deep Kernel Learning~\citep{wilson2015deep} framework and GPs by using NNs to learn low dimensional representation where the GP model is jointly trained, resulting in a method called \textit{Deterministic Uncertainty Quantification} (DUQ).
More recently, \citeauthor{van2021feature},~\citeyear{van2021feature} showed that the DUQ method based on DKL \textit{can} map OOD data close to training data samples, referred as ``feature collapse''.
The authors thus propose the \textit{Deterministic Uncertainty Estimation} (DUE) which in addition to DUQ method ensures that the encoder mapping is bi-Lipschitz.
It would be interesting to explore if these methods perform well on similar OOD experiments as those considered in this work.
In App.~\ref{app:toy_example_gp} we use the DUE model for UDP.

\paragraph{Bayesian Neural Networks (BNNs).}
BNNs are stochastic neural networks trained using a
Bayesian approach.
Following the same notation as in the main paper, given $N$ training datapoints $\mathcal{D} = \{(\bm{x}_i, y_i)\}_{i=1}^N$, 
training BNNs extends to posterior inference by estimating the posterior distribution:
\begin{equation}
    p(\bm{\omega} | \mathcal{D}) = \frac{p(\mathcal{D}|\bm{\omega}) p(\bm{\omega})}{p(\mathcal{D})} \,,
\end{equation}
where $p(\bm{\omega})$ denotes the prior distribution on a parameter vector $\bm{\omega} \in \Omega$.
Given a new  sample $\bm{x}',y'$, the predictive distribution is then:
\begin{equation}
    p(y'|\bm{x}', \mathcal{D}) =  \int_{\Omega} p(y'|\bm{x}', \bm{\omega}) p(\bm{\omega}|\mathcal{D}) d\bm{\omega}\,.
\end{equation}
As GPs, BNNs are also often computationally expensive, what arises due to the above integration with respect to the whole parameter space $\Omega$.
The two most common approximate methods to train BNNs are
\textit{(i)} Mean-Field Variational Inference~\citep[MFVI, ][]{blei2017}, and 
\textit{(ii)} Hamiltonian Monte Carlo~\citep[HMC, ][]{neal2012}, see~\citep{jospin2021handson}.

\paragraph{Ensembles \& Monte Carlo Dropout (MCDropout).}
Popular epistemic UE method is \textit{deep ensembles} \citep{lakshminarayanan2017}, which trains a large number of models on the dataset and combines their predictions to estimate a predictive distribution over the weights.  
\citeauthor{gal2016mcdropout},~\citeyear{gal2016mcdropout} further argue that \textit{Dropout}~\citep{srivastava2014dropout} when applied to a neural network approximates Bayesian inference of a Gaussian processes~\citep{rasmussen2005gp}. The proposed \textit{Monte Carlo Dropout} (MC Dropout)---which applies Dropout at inference time---allows for a more computationally efficient uncertainty estimation relative to deep ensembles.
~\citeauthor{ovadia2019lldrop} further reduce the computational cost by applying Dropout only to the last layer.

\subsection{Connecting robustness and simplicity bias with the margin of the model}
\citet{vapnik95} showed that the
VC dimension---that measures the capacity of the  model---of linear classifiers restricted to a particular data set can be bounded in
terms of their margin, which measures how much they separate the data.
\citet{ElsayedKMRB18} argue on the difference between the theory of large margin and deep networks, and a novel training loss to increase the margin of the classifier.
Inspired by the fact that the training error of the neural networks is $0$, recently~\citep{nakkiran2021the} empirically showed that ``improved ideal-world accuracy is achieved by models  which do not reduce the training loss too quickly'', where ``ideal-world accuracy'' refers to that of a having the infinite data-samples of the true generative model that generated the limited training samples.  This is inline with our discussion in \S~\ref{sec:intro} on the \textit{simplicity bias}.

\subsection{Adversarial training in input and in latent space}\label{app:related_at}
In the context of computer vision, small perturbations in image space that are not visible to the human eye can fool a well-performing classifier into making wrong predictions. 
However, further empirical studies showed that such training reduces the training accuracy, indicating the two objectives---robustness and generalization---are competing~\citep{tsipras2019robustness,su2018robustness}. \cite{ZhangTRADES} characterize this trade-off by decomposing the robust error as the sum of the natural error and the boundary error. They further introduce~\ref{eq:trades}--see \S~\ref{sec:preliminaries}, which adds a regularizer encouraging smoothness in the neighborhood of samples from the data distribution. 
Interestingly, both TRADES and UDP are unsupervised, in a sense that they do not rely on the label $\vy_i$ of the starting data point $\vx_i$ to obtain $\tilde\vx_i$.
However, TRADES is orthogonal to UDP, as \textit{given} a perturbation $\tilde\vx_i$ it forces that the model's output is smooth between  $\tilde\vx_i$  and $\vx_i$.
Thus, given a perturbation obtained from UDP, the same method can be applied.
We leave it as a future direction to further explore combining it with the herein proposed uncertainty driven perturbations.
~\citet{wong2020fast} further pointed out a phenomenon referred to as \textit{catastrophic overfitting} where the robustness of fast~\ref{eq:ldp} methods rapidly drops to almost zero, within a \textit{single} training epoch, and in \S~\ref{sec:cat_overfit} we provide results with UDP.

To improve the catastrophic overfitting of~\ref{eq:fgsm}, \citet{tramer2018ensemble} propose adding a random vector $\vxi$ to~\ref{eq:fgsm} as follows: 
\noindent
\begin{equation}\tag{R-FGSM}\label{eq:r-fgsm}
    \vdelta_{\text{R-FGSM}} \triangleq \mathop{\Pi}\limits_{\|\cdot\|_{\infty} \leq \epsilon} \Big(\vxi + \alpha\cdot \text{Sign}\big(\mathop{\nabla}\limits_{\vx}  \LL(\mathcal{C}_\vomega (\vx), \vy) \big)\Big) \,,
\end{equation}
where $ \vxi \sim U([-\varepsilon,\varepsilon]^d)$, $\alpha\in[0,1]$ is selected step size, and $\Pi$ is projection on the  $\ell_\infty$--ball.

Most similar to ours, \cite{ZhangXH0CSK20} argue that ``friendly'' attacks--attacks which limit the distance from the boundary when crossing it through early stopping--can yield competitive robustness. Their approach differs from ours as they do not consider the notion of uncertainty which also more elegantly handles the ``manual'' stopping.

\citet{stutz2019disentangling} 
postulate that the observed drop in clean test accuracy appears because the adversarial perturbations leave the data-manifold, and that `on-manifold adversarial attacks' will hurt less the clean test accuracy. 
The authors thus propose to use perturbations in the latent space of a VAE-GAN~\citep{larsen2016,rosca2017variational}, and recently~\cite{Yuksel2021SemanticPW} used Normalizing Flows for this purpose due to their exactly reversible encoder-decoder structure. 
Our approach is orthogonal to these works as it can also be applied in latent space---see \S~\ref{sec:latent_exp}, and interestingly, we show that our method does not decrease notably the clean test accuracy relative to common~\ref{eq:ldp} methods, while it notably improves the model's robustness relative to standard training.

\subsection{Empirical results on the correlation between calibration and robustness}\label{app:confidence_robustness}

Recently,~\citeauthor{qin2021improving},~\citeyear{qin2021improving} studied the connection between adversarial robustness and
calibration, where the latter indicates if the model's predicted confidence correlates well with the true likelihood of the model being correct~\citep{Guo2017calibration}.
The authors find that the inputs for which the model is sensitive to small perturbations are more likely to have poorly calibrated predictions.
To mitigate this, the authors propose  an adaptive variant of label smoothing, where the smoothing parameter is determined based on the miscalibration, and the authors report that this method improves the calibration of the model even under distributional shifts.
The above empirical observation on the correlation between the calibration and robustness motivates the proposed method here: using perturbations with poor calibration during training may improve both the robustness and calibration at inference.

\subsection{Maximum entropy line of works}
As one way to estimate the model's uncertainty is using the entropy--see \S~\ref{sec:preliminaries}, and moreover since we define a UDP variant that implements the UDP objective as a regularizer--Eq.~\ref{eq:udpr}, the approaches based on the \textit{maximum entropy} principle~\citep{jaynes57} can be seen as relevant to ours. 
In particular, in the context of standard classification,~\citet{pereyra2017RegularizingNN} penalize \textit{confident} predictions by adding a regularizer that \textit{maximizes} the entropy of the output distribution, since confident predictions correspond to output distributions that have low entropy. Training is formulated as follows:
\begin{equation}\tag{MaxE}\label{eq:max_entropy}
    \min_\vomega \displaystyle \mathop{\mathbb{E}}_{(\vx,\vy)\sim p_d} 
    \big[  \LL( \mathcal{C}_\vomega (\mathcal{E}(\vx)\!+\!\vdelta), \vy)  
    - \gamma \cdot \mathcal{H} (\vx, \vomega)
    \big] \,,
\end{equation} 
where $\mathcal{H}(\cdot)$ is defined in Eq.~\eqref{eq:entropy}.
~\citeauthor{pereyra2017RegularizingNN} point out that the above training procedure has connections with label smoothing~\citep{inceptionmodel} for uniform prior
label distribution~\citep[see \S 3.2 of ][]{pereyra2017RegularizingNN}, and the latter is also shown to improve the generalization. See also App.~\ref{app:confidence_robustness} above that relates to label smoothing.
Moreover, in the context of adversarial training, the results of~\citep[][ \S 3.2]{cubuk2017intriguing} indicate that adding such a regularizer improves the model robustness to PGD attacks.
While both~\ref{eq:udpr} and~\ref{eq:max_entropy} involve the  maximum-entropy principle and modify the training objective with a regularizer, the difference is that~\ref{eq:udpr} uses the max-entropy principle to obtain the perturbations $\tilde\vx$ of the input data points.

Entropy regularization is also widely used in reinforcement learning (RL) to improve the exploration by encouraging the policy to have an output distribution with high entropy~\citep{WilliamsPeng91}. It remains an open interesting direction to find connections of UDP and the max-entropy principle in RL.

\subsection{Discussion}
The presented idea is connection between several lines of works, some of which are listed above.
A high level summary of the above works is that:
(i) the robustness to perturbations is tied to the value of the  uncertainty estimate,
(ii) epistemic uncertainty defines the objective we would like to target in training with perturbations, as we ultimately care about generalization and improving the margin of the classifier.

Finally, it is worth noting that uncertainty estimation for neural networks is an active line of research, since many popular methods fail to detect OOD data points~\citep{ovadia2019lldrop,liu2021peril}.
However, such methods work well around the decision boundary, and we showed that this property is particularly useful, and it provides promising results for training with perturbations.
Nonetheless, it remains an interesting direction to explore UDP with newly proposed UE methods.


\clearpage
\section{Proof of Theorem~\ref{thm:udp_simple_lin_separable} }\label{app:proofs}

Recall that we consider a setting of binary classification $x\in\R$ data samples for simplicity, sampled from the data distribution $p_d$  which is linearly separable, and each cluster  has mean $\mu_1, \mu_2$ ($\mu_1 < \mu_2$), respectively, 
with label $y^{(1)}=-1$ and $y^{(2)}=+1$. 
For this problem, the optimal boundary is at
$\omega^\star = \frac{\mu_1+\mu_2}{2}$.
Since SGD yields the max-margin classifier for this setting (and so does regular training--see \S~\ref{sec:intro}), we consider a weaker requirement for the optimizer:  at each iteration $n$, given data points $\tilde x^{(1)}_n$ and $\tilde x^{(2)}_n$, an oracle returns a \textit{direction} $\tilde\omega_n$ that belongs in the interval between $\tilde x^{(1)}_n$ and $\tilde x^{(2)}_n$, that is:
\begin{equation}\tag{S1}\label{eq:s1-dup}
    \exists \alpha_n \in (0,1), \quad \text{s.t.}\quad  \tilde\omega_n = \alpha_n \tilde x^{(1)}_n + (1-\alpha_n)\tilde x^{(2)}_n \,,
\end{equation}
where we assume: \vspace*{-.1cm}
\begin{equation}\tag{A1}\label{eq:a1-dup}
   \alpha_n \sim \mathcal{U} (0,1), \quad \text{i.e.} \quad
    \E[\alpha] = \frac{1}{2}\,,
\end{equation}
Given $\omega_n$ and $\tilde\omega_n$, the update rule is as follows:
\begin{equation}\tag{S2}\label{eq:s2-dup}
    \omega_{n+1} = \omega_n + \eta (\tilde\omega_n-\omega_n) \,, 
\end{equation}
where $\eta\in(0,1)$ is the step-size. \looseness=-1

\begin{theorem}[Convergence of UDP] 
\label{thm:udp_simple_lin_separable-dup}
Given an oracle that takes as input two data points $\tilde x^{(1)},\tilde x^{(2)} $ and returns a point that lies between these points sampled uniformly at random, the UDP method as per Alg.~\ref{alg:upd} converges to the max-margin classifier on a linearly separable dataset.
\end{theorem}

\begin{proof}
Following Alg.~\ref{alg:upd}, the UDP perturbations are described with: \vspace*{-.15cm}
\begin{equation}\tag{UDP-x}\label{eq:udp-x-dup}
\begin{aligned}
  \text{given } x_n^{(1)}\sim p_d ( x|y=-1), \qquad
  \exists \beta_n &\in (0,1), \quad \text{s.t.}\qquad 
    \tilde x_n^{(1)} = \beta_n  x_n^{(1)} + (1-\beta_n)\omega_n \\
    \text{given } x_n^{(2)}\sim p_d( x|y=1), \qquad 
    \exists \gamma_n &\in (0,1), \quad \text{s.t.}\qquad 
    \tilde x_n^{(2)} = \gamma_n  x_n^{(2)} + (1-\gamma_n)\omega_n \,,
\end{aligned}
\end{equation}
where $\omega_n$ is the current iterate. Moreover, from Alg.~\ref{alg:upd} we have that $\E[\beta]=\frac{1}{2}$ and $\E[\gamma]=\frac{1}{2}$.

Given $ x_1, x_2 \sim p_d$, and $\omega_n$, we are interested in $\omega_{n+1}$ in expectation, and how it relates to the optimal max margin boundary $\omega^\star = \frac{\mu_1+\mu_2}{2}$.
We have: 

\begin{align*}
    \E[\omega_{n+1}|\omega_n] &= 
    \E[\omega_n + \eta (\tilde{\omega}-\omega_n)|\omega_n] \\
    &= \omega_n + \eta (\mathbb{E}[\tilde{\omega}]-\omega_n) \\
    &= (1-\eta) \omega_n + \eta \E \big[\alpha \tilde{ x}^{(1)}_n + (1-\alpha) \tilde{ x}^{(2)}_n \big] \\
    &= (1-\eta) \omega_n + \eta \mathbb{E}[\alpha \big(\beta  x^{(1)}_n + (1-\beta) \omega_n\big) + (1-\alpha) \big(\gamma  x_n^{(2)} + (1-\gamma) \omega_n\big)] \\
    &= (1-\eta) \omega_n + \frac{\eta}{4}(\mu_1 + \omega_n + \mu_2 + \omega_n) \\
    &= (1-\frac{\eta}{2} ) \omega_n + \frac{\eta}{2} (\frac{\mu_1+\mu_2}{2}) \\
    &= (1-\frac{\eta}{2} ) \omega_n + \frac{\eta}{2} \omega^\star\,,
\end{align*}
where in fifth line we used~\eqref{eq:a1}.
Since $\eta\in(0,1)$, in expectation  $\omega_{n+1}$ lies on a line between $\omega_n$ and $\omega^\star$.
This shows the boundary decision $\omega_n$ will have a tendency to move towards the optimal max margin boundary $\omega^\star$ for $\eta\in(0,1)$. 
\end{proof}

In the case of LDP-PGD, the position of the perturbations ``ignores'' the position of the boundary decision (it does not stop at it), as the loss is larger the further the perturbation moves following the direction from it to the decision boundary.
Thus, for large $\eps>1$ the perturbations will switch the label, and  for $\eps=\infty$, the perturbations go to $\pm\infty$.

\section{Pictorial Representation: \ref{eq:udp} Vs. \ref{eq:ldp}}\label{app:pictorial}

In Fig.~\ref{fig:illustration_udp_ldp} of the main paper we considered two samples per class, whose distances to the optimal decision boundary---denoted with $\eps_1$ and $\eps_2$---largely differs, that is, we have $\eps_1 << \eps_2$. 
For completeness, in Fig.~\ref{fig_app:illustration_udp_ldp} we depict all the three cases: 
\textit{(i)} $\epsilon_{train} \leq \epsilon_1$,
\textit{(ii)} $\epsilon_{train} {\in} (\epsilon_1, \epsilon_2) $, and
\textit{(iii)} $\epsilon_{train} \geq \epsilon_2 $ 
While for the first case when $\eps_{train}< \eps_1$, \ref{eq:udp} and \ref{eq:ldp} perturbations coincide, we observe that for the latter two these differ. 
\ref{eq:udp} can be seen as it ``locally adapts'' the effective $\eps_i$, making it less likely to cross the optimal decision boundary $\theta^\star$, avoiding training with a $\tilde\vx$ whose label is semantically incorrect.

\begin{figure*}[ht]
  \centering
  \subfigure[$1^{\circ}: \epsilon_{train} \leq \epsilon_1$]{
  \label{subfig_app:illust_c1}
  \includegraphics[width=.24\linewidth]{illustration_UDPvsLDP_case1.pdf}} \hspace{.4cm}
  \subfigure[$2^{\circ}:  \epsilon_{train} {\in} (\epsilon_1, \epsilon_2) $]{
  \label{subfig_app:illust_c2}
  \includegraphics[width=.265\linewidth,trim={0 .4cm 0 .4cm},clip]{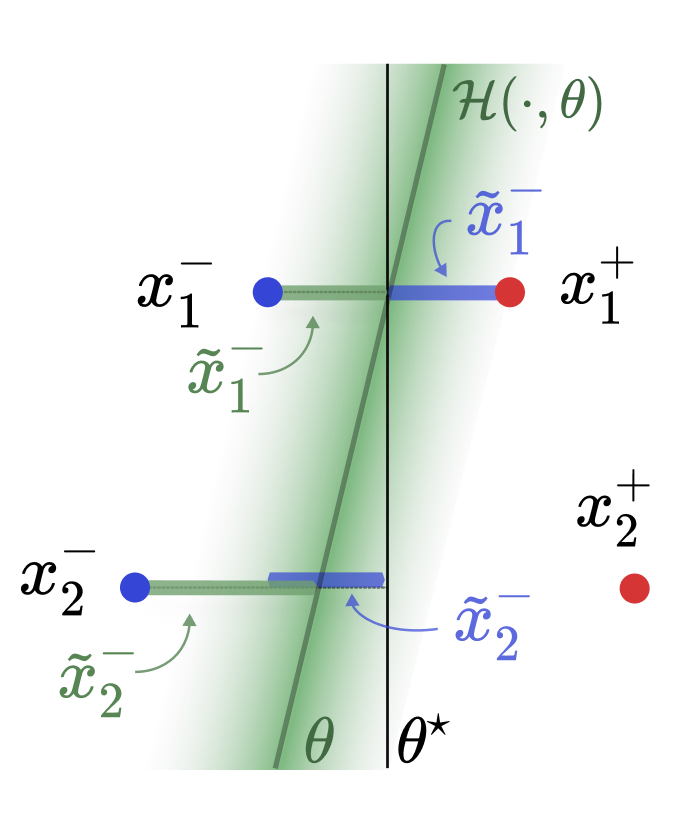}}
  \hspace{.4cm}
  \subfigure[$3^{\circ}:  \epsilon_{train} \geq \epsilon_2 $]{
  \label{subfig_app:illust_c3}
  \includegraphics[width=.29\linewidth,trim={0 .4cm 0 .4cm},clip]{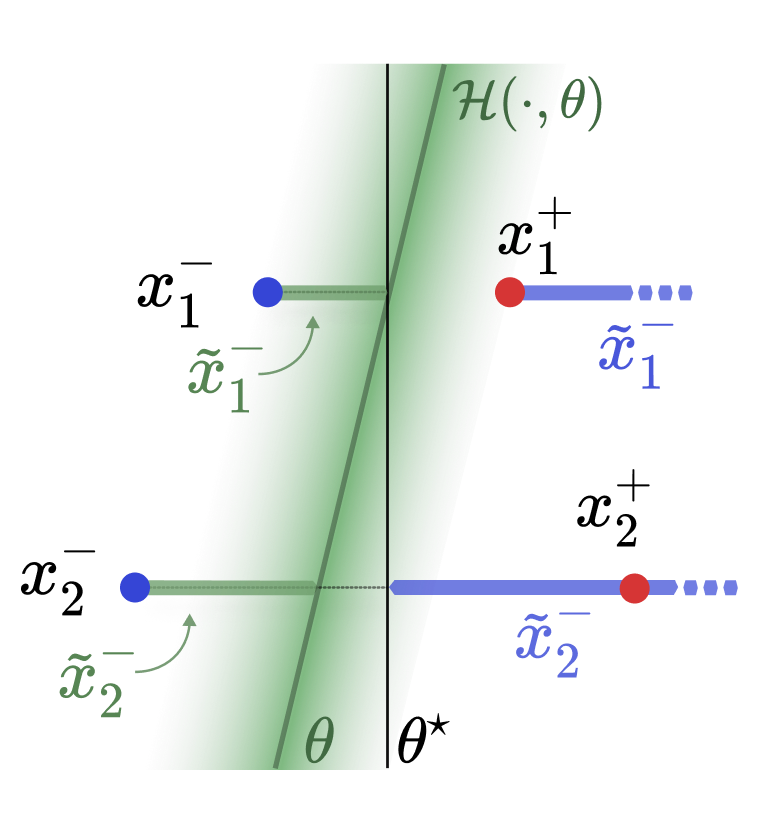}}
\vspace{-.2cm}
\caption{
Complementary to Fig.~\ref{fig:illustration_udp_ldp}---
pictorial representations of~\ref{eq:ldp} and~\ref{eq:udp}---where in contrast to Fig.~\ref{fig:illustration_udp_ldp} we depict the three cases: $\eps_{train} \leq \eps_1$, $\eps_{train} \in (\eps_1, \eps_2)$ and $\eps_{train} \geq \eps_2$. 
Circles represent training samples, and color their class $y\in\{-1,1\}$. 
$\theta$ and $\theta^\star$ denote the current decision boundary and the optimal one, resp.
For clarity, only the perturbations $\tilde\vx^-$ of the negative (blue) class are shown.
\textbf{Fig.~\subref{subfig_app:illust_c1}.}
Given \textit{ideal} data-perturbation training procedure using $\epsilon_{train}<\epsilon$ the perturbed data points (either with UDP or LDP) will belong to the shaded blue area (assuming the boundary $\theta$ has $0$ training loss).
Thus, the final decision boundary will belong to the shaded grey area, and we \textit{cannot} guarantee the margin is maximized--see~\S~\ref{sec:intro}.
\textbf{Fig.~\subref{subfig_app:illust_c2} and~\subref{subfig_app:illust_c3}}--explanation of \textit{why UDP's property of not crossing the boundary is advantageous}.
The green shade depicts the uncertainty, 
where darker is higher -- starting from a point with negative label the uncertainty increases, reaches its maximum, and then decreases.
On the contrary, the loss used for LDP perturbations (of negative samples) keeps increasing the further we go on the right of $\theta$. 
Thus, the LDP and UDP perturbed samples---obtained by iteratively following the direction that maximizes these quantities---differ, and their possible regions are shown in blue and green shade, resp., for different choices of $\epsilon_{train}>\epsilon_1$.
When the perturbed samples from the negative class pass $\theta^\star$, LDP yields training pair $(\tilde x_i^-, y_i)$ with a wrong label.
This indicates that to avoid oscillations (caused by wrongly labeling input space regions), for LDP we are restricted to using $\eps_{train}\leq \eps$, but as Fig.~\subref{subfig_app:illust_c1} shows, this can be restrictive.
On the other hand, for UDP we can have relatively larger $\eps$ which  ``locally adapts'' and is effectively smaller in some region and larger in others.
\textit{Remark:} for easier comprehensiveness on the $\eps$ distances, we draw the perturbed samples region orthogonal to $\theta^\star$, whereas these are orthogonal to the current boundary $\theta$.
}\label{fig_app:illustration_udp_ldp}
\end{figure*}

\clearpage
\section{Additional Results on Simulated Datasets}\label{app:toy}

\subsection{Max-margin decision boundary on the LMS-5 dataset}

The maximum margin boundary for the LMS-5 dataset from \S \ref{sec:example} is shown in Fig. \ref{fig:lms-max-margin}.

\begin{figure}[htb]
    \centering
    \includegraphics[width=.34\linewidth]{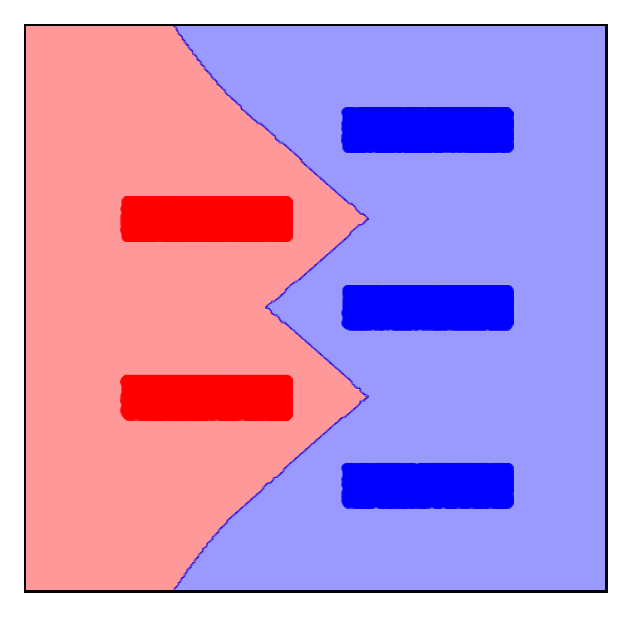}
    \caption{Complementary to Fig.~\ref{fig:toy_nc_lms5} (bottom row): The boundary with the maximum margin on the LMS-5 dataset.}
    \label{fig:lms-max-margin}
\end{figure}

\subsection{Simple dataset with two different distances}\label{app:toy_setup_two_diff_distances}

In the context of computer vision, adversarial training aims at finding ``imperceptible'' perturbation which does not change the label of the input data-point--see~\ref{eq:at_assumption}. Thus the selected $\epsilon$ value for training--herein denoted with $\epsilon_{train}$--is typically small.
In this work, we argue that satisfying~\ref{eq:at_assumption} ``globally'' can be in some cases very restrictive.
As an additional simplified example to those presented in \S~\ref{sec:example}, Fig.~\ref{fig:toy_pgd_uta} depicts a toy experiment with non-isotropic distances between samples of opposite classes in $\R^2$. 
In particular, there are two regions with notably different distances between samples from the opposite class, let us denote the respective optimal $\epsilon$ values for each region with $\epsilon_1$ and $\epsilon_2$, and from Fig.~\ref{fig:toy_pgd_uta}, $\epsilon_1 << \epsilon_2$.
In this section, we focus on~\ref{eq:pgd}---top row, and below we discuss two \textit{potential} issues of loss-based attacks.

~\ref{eq:at_assumption} enforces that we select $\epsilon_{train}=\epsilon_1$. In cases when $\epsilon_1<< \epsilon_2$ adversarial training only marginally improves the model's robustness in the second region where $\epsilon_2$ can be used locally during training as such perturbations will not modify the class label of samples of that region. 
Moreover, if the data is disproportionately concentrated---e.g. small mass of the ground truth data distribution is concentrated in the first region, and the remaining is concentrated in the second region---one still needs to select $\epsilon_{train}=\epsilon_1$, as per~\ref{eq:at_assumption}.

To later contrast the behavior of loss-based PGD with uncertainty-based PGD, in Fig.~\ref{fig:toy_1_step}-\ref{fig:toy_20_steps} we consider loss-based \ref{eq:pgd} with $\epsilon_{train}$ which \textit{violates}~\ref{eq:at_assumption} (thus this value would not be used in practice by practitioners).
From Fig.~\ref{fig:toy_pgd_uta} we observe that \ref{eq:pgd} can perturb the sample by moving it on the opposite side of the boundary---resulting in \textit{mislabeled} samples.
See Fig.~\ref{fig:compare_imgs_1000steps_cifar} for such analysis on CIFAR-10.
While Fig.~\ref{fig:toy_1_step}-\subref{fig:toy_20_steps} illustrates the difference between attacks for a \textit{fixed model}, Fig.~\ref{fig:toy_dec_boundary} illustrates the decision boundaries \textit{after} adversarial training with ~\ref{eq:pgd}. Similarly, Fig.~\ref{fig:toy_oscil} shows that for this ``violating'' case~\ref{eq:pgd} training is characterized with large oscillations of the decision boundary, resulting in inefficient training. 
See App.~\ref{app:toy_example_gp} for results on the same dataset but using Gaussian Process as a model.

In the general case, the largest $\epsilon$ value which does not break the~\ref{eq:at_assumption} assumption \textit{locally} would \textit{differ} among $m$ non-overlapping local regions of the input data space, and let us denote these with $\epsilon_1, \epsilon_2, \dots, \epsilon_m$. Note that, as such none of the $\epsilon_i, i \in 1, \dots, m$ when applied to the corresponding $i$\textit{-th} region changes the label class. 
Let without loss of generality $\epsilon_1 < \epsilon_2 < \dots < \epsilon_m$. 
\ref{eq:at_assumption} forces that for training we select $\epsilon_{train} = min(\epsilon_1, \epsilon_2, \dots , \epsilon_m) = \epsilon_1$. 
The decision boundary of a model $f$ trained adversarially with $\epsilon_{train}$ lies within a margin that is far by \textit{at least} $\frac{\epsilon_{train}}{2}$ from any training data point.

On the other hand, UDP achieves the above desired goal in an elegant way: as ~\ref{eq:udp} does not cross the decision boundary, it effectively uses varying $\epsilon_i \leq \epsilon_{train}$, for different local regions of the input data space, where $\epsilon_{train}$ is selected beforehand.
From Fig.~\ref{fig:toy_pgd_uta}, we observe that~\ref{eq:udp}-based training or more precisely~\ref{eq:udp-pgd} is relatively less sensitive to the choice of $\epsilon_{train}$.
Importantly, inline with the above discussion, using large $\varepsilon$ for~\ref{eq:udp-pgd} during training does not deteriorate its final clean accuracy, as Fig.~\ref{fig:toy_dec_boundary} depicts.  
Finally, from Fig.~\ref{fig:toy_oscil} we observe that even for large selected values of $\epsilon_{train}$--which value breaks ~\ref{eq:at_assumption} solely \textit{globally}, the boundary decision of~\ref{eq:udp} training oscillates relatively less to that of PGD and mostly close to the ground-truth decision boundary.
By being able to use larger values for $\epsilon$,~\ref{eq:udp} allows for larger ``exploration'' of the input space and in turn learn ``better'' latent representations and decision boundary, characterized by better clean-accuracy and PGD robustness trade-off, as we shall see in \S~\ref{sec:exp}.
This is in sharp contrast to PGD, which is forced to use the smallest $\epsilon_1$ even in cases when $\epsilon_1 << \epsilon_m$, as using $\epsilon_{train}$ value which violates~\ref{eq:at_assumption} performs poorly.
Moreover, we argue that such non-isotropic margins as in Fig.~\ref{fig:toy_pgd_uta} are more likely to occur in real-world datasets.

\begin{figure*}[ht] 
  \centering
  \subfigure{
  \includegraphics[width=.235\linewidth]{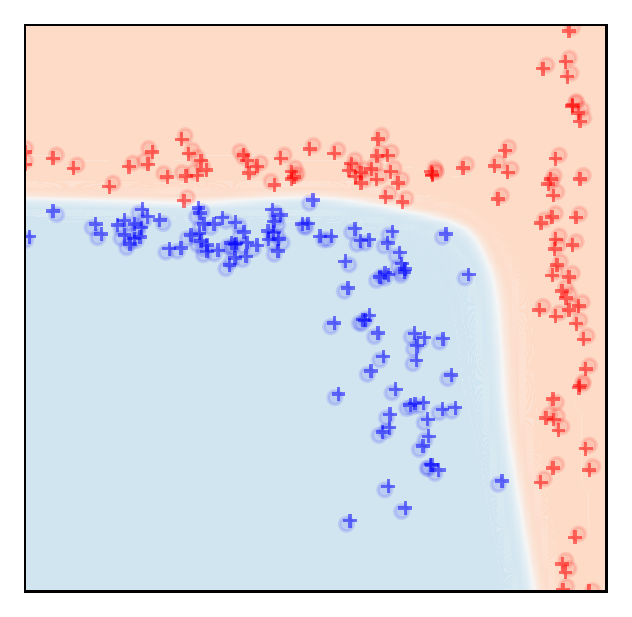}
  }
  \subfigure{
  \includegraphics[width=.235\linewidth]{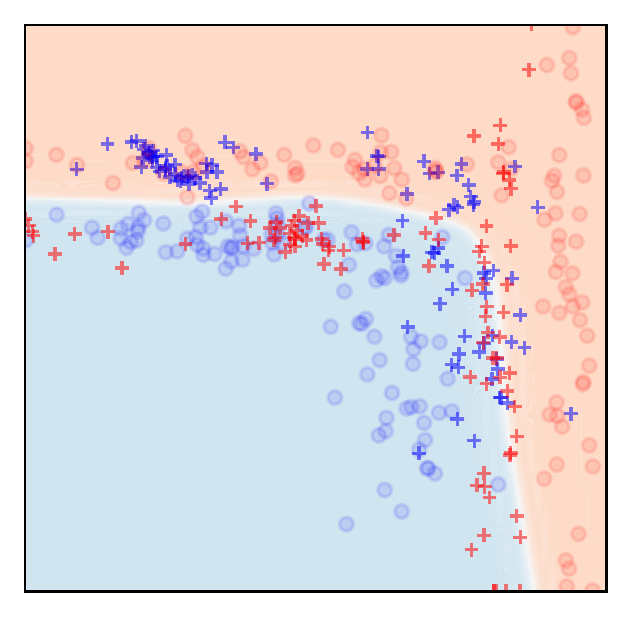}
  }
  \subfigure{
  \includegraphics[width=.235\linewidth]{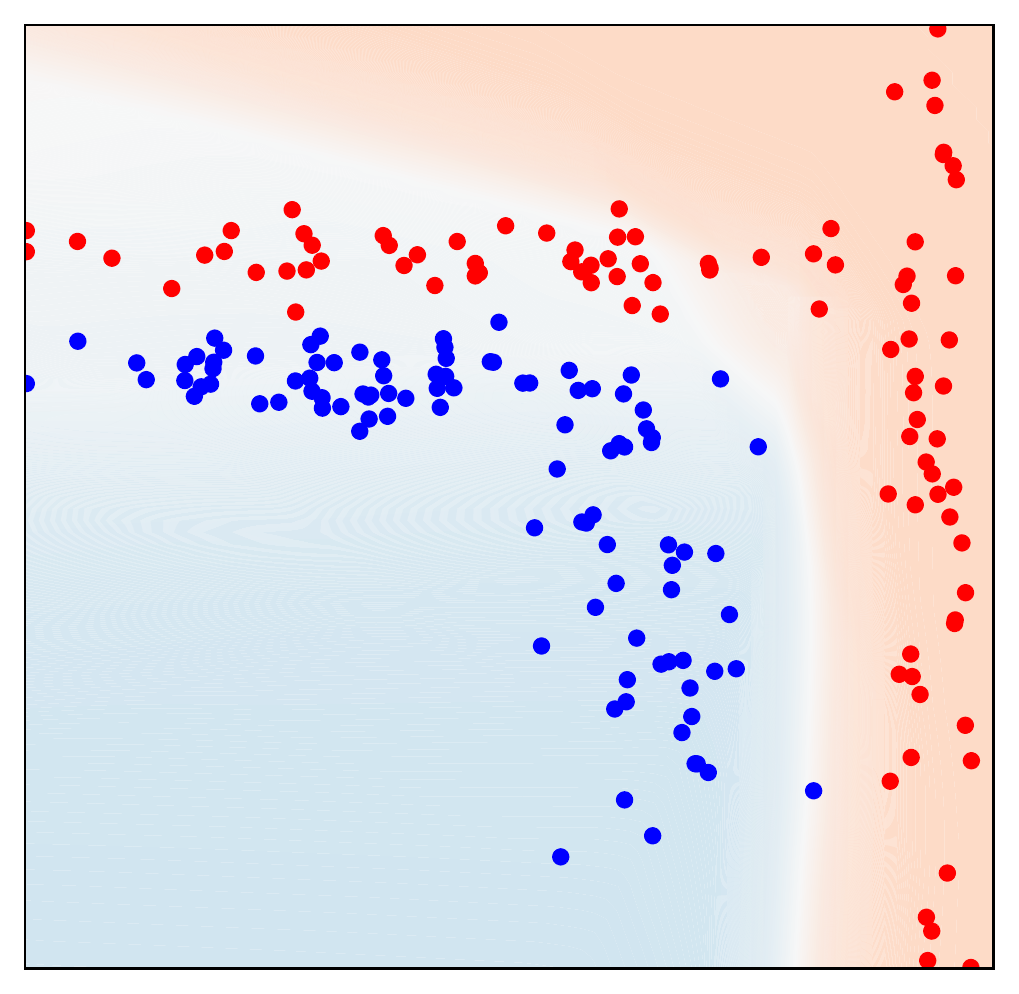}
  }
  \subfigure{
  \includegraphics[width=.235\linewidth]{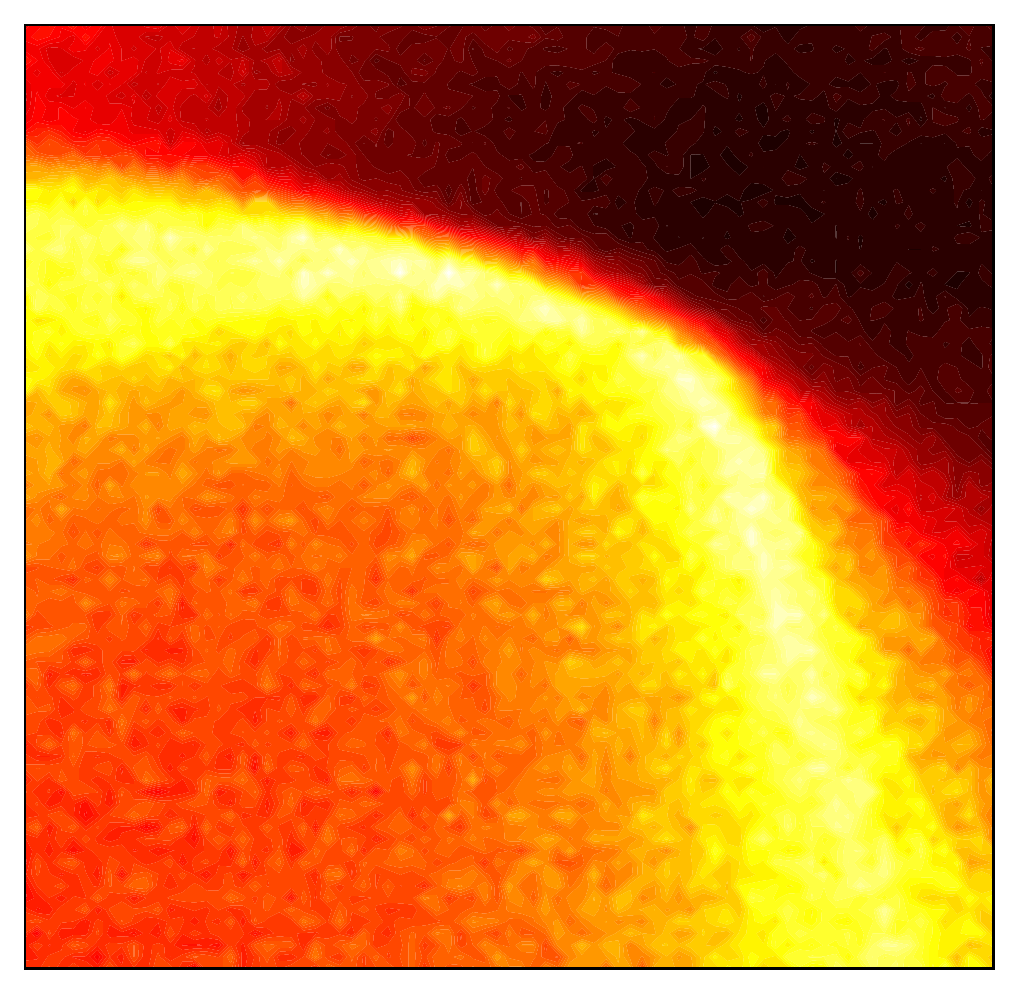}
  } 
    \subfigure[1 step]{\label{fig:toy_1_step}
  \includegraphics[width=.235\linewidth,clip]{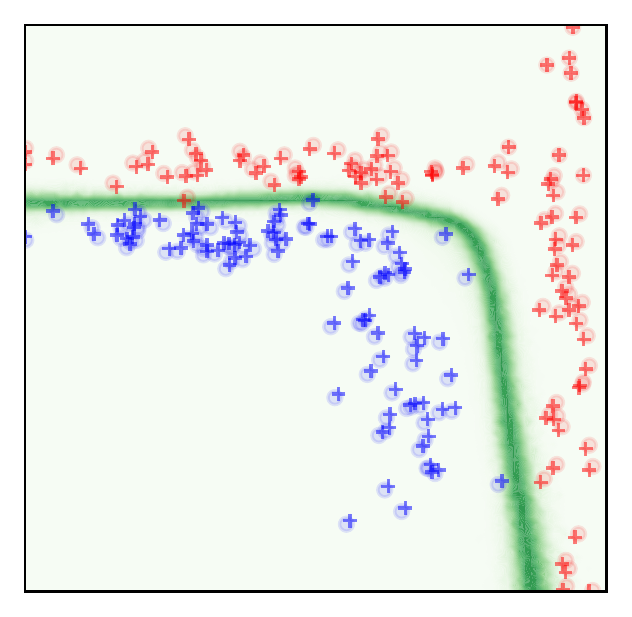}\addtocounter{subfigure}{-4} 
  }
  \subfigure[20 steps]{\label{fig:toy_20_steps}
  \includegraphics[width=.235\linewidth,clip]{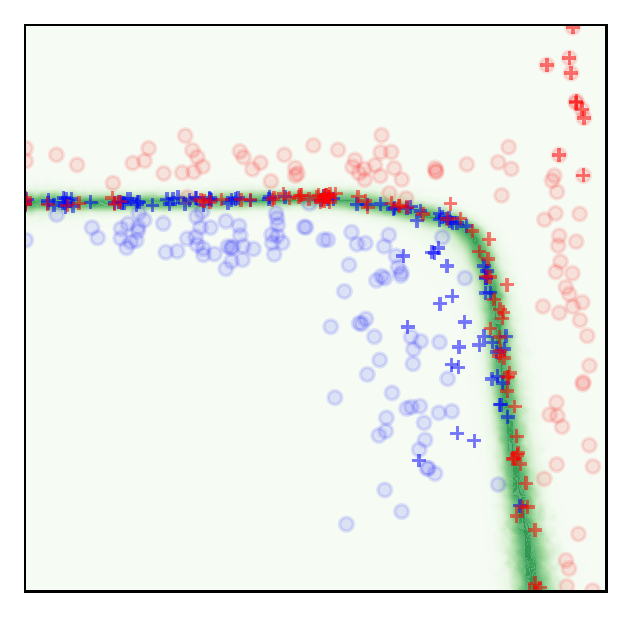}
  }
  \subfigure[Decision boundary]{\label{fig:toy_dec_boundary}
  \includegraphics[width=.235\linewidth,clip]{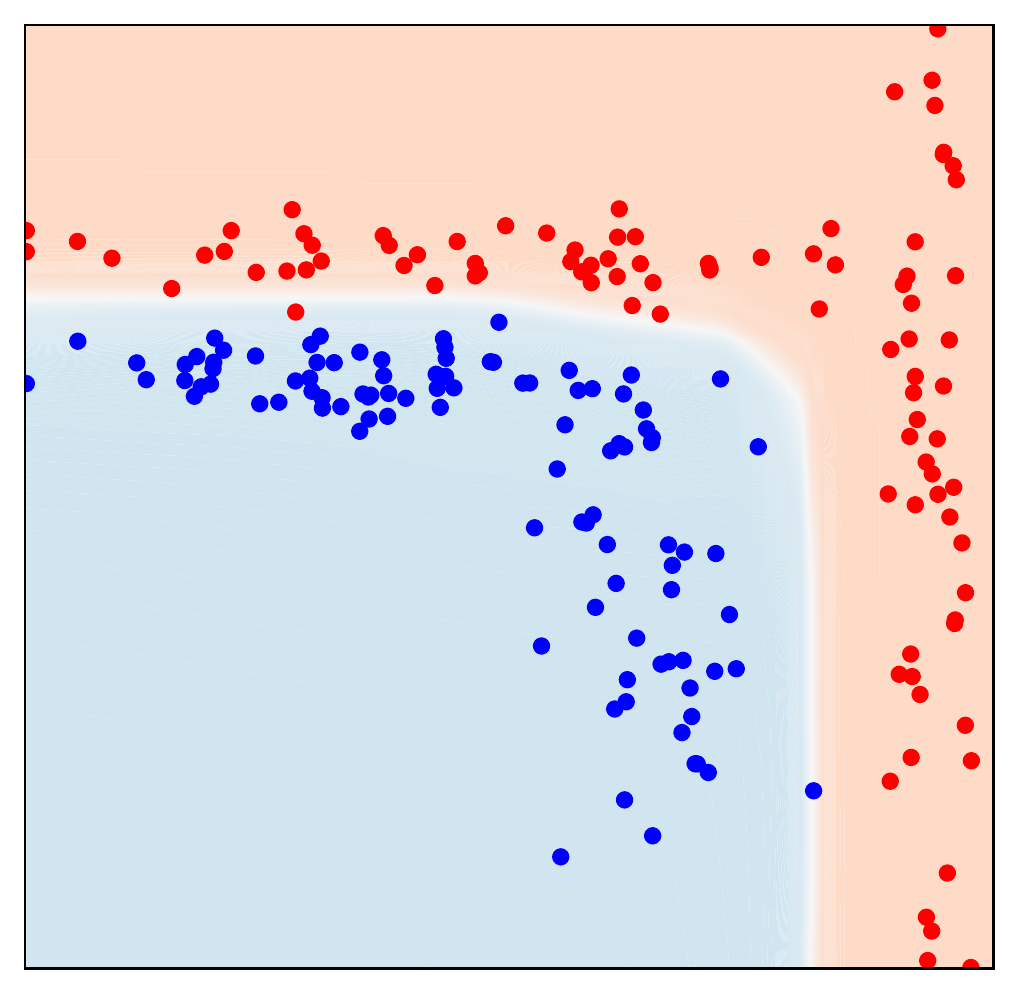}
  }
  \subfigure[Boundary oscillations]{\label{fig:toy_oscil}
  \includegraphics[width=.235\linewidth,clip]{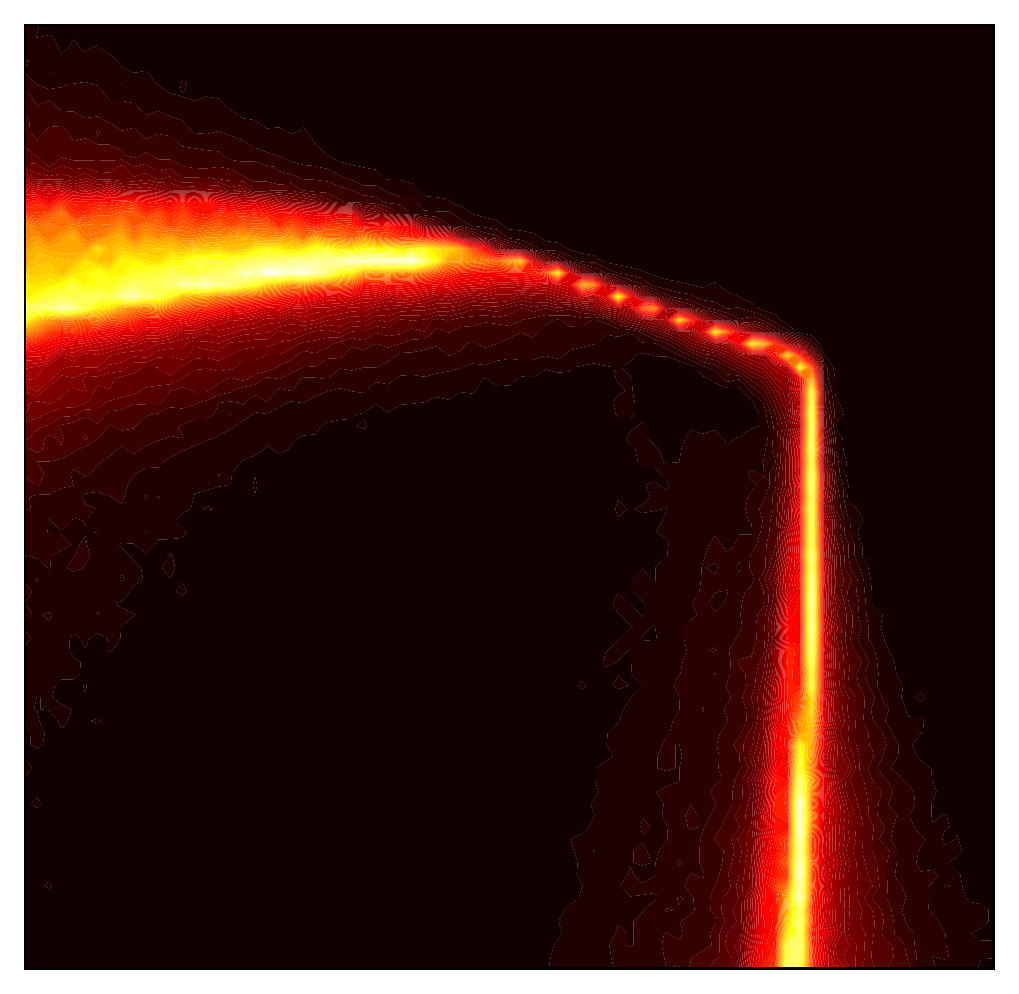}
  }
 \caption{2D experiment---performance of \textbf{LDP-PGD} (left) Vs. \textbf{UDP-PGD} (right) using $\epsilon$ which \textit{violates} the~\ref{eq:at_assumption}. 
 \textbf{Columns \subref{fig:toy_1_step}-\subref{fig:toy_20_steps}.} Attacks to \textit{a fixed trained classifier}, with $1$ and $20$ steps (columns), resp. $\circ$ and $+$ depict the clean and perturbed samples, resp., and their color blue/red depicts their class.
 \textit{Top row:} background depicts the assigned probability to each class of the attacked model, including its decision boundary.
 \textit{Bottom row:} background depicts the uncertainty estimates of a $10$ model ensemble, where darker green is higher. 
 \textbf{Column \subref{fig:toy_dec_boundary}.} 
 Decision boundaries \textit{after} adversarial training with \textit{large} $\epsilon$ (and $k=15$, $\alpha\!=\!0.05$). 
 \textbf{Column ~\subref{fig:toy_oscil}.} Summary of the boundary oscillations over the updates, for a fixed mini-batch size for both \ref{eq:pgd} and \ref{eq:udp}. For each point $\vx \in \R^2$ we count how many times the classifier changed its prediction, and depict with darker to lighter color zero to many changes, resp. 
 See App.~\ref{app:toy_setup_two_diff_distances} for discussion.
 }\label{fig:toy_pgd_uta}
\end{figure*}

\subsection{Additional results on the \textit{Narrow Corridor} dataset}\label{app:additional_results_nc}
Fig.~\ref{fig:nc_varying_ps} is complementary to that of Fig.~\ref{fig:toy_nc_lms5}, and shows the methods' performance for varying $\eps$  values.

\begin{figure}[ht]
    \centering
    \begin{tabular}{cccccc}
    \rotatebox{90}{\parbox{2mm}{\multirow{3}{*}{\qquad LDP-PGD}}} 
    & \includegraphics[width=.16\linewidth]{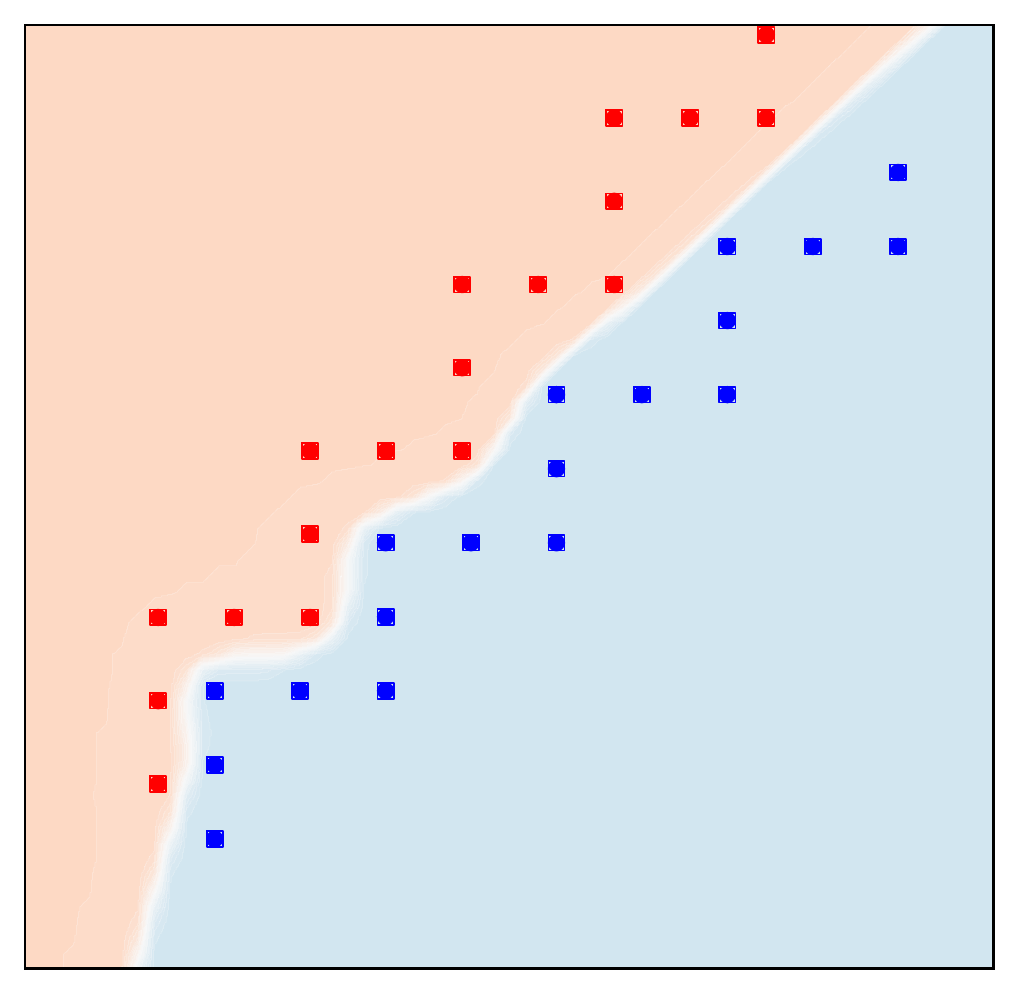}
    & \includegraphics[width=.16\linewidth]{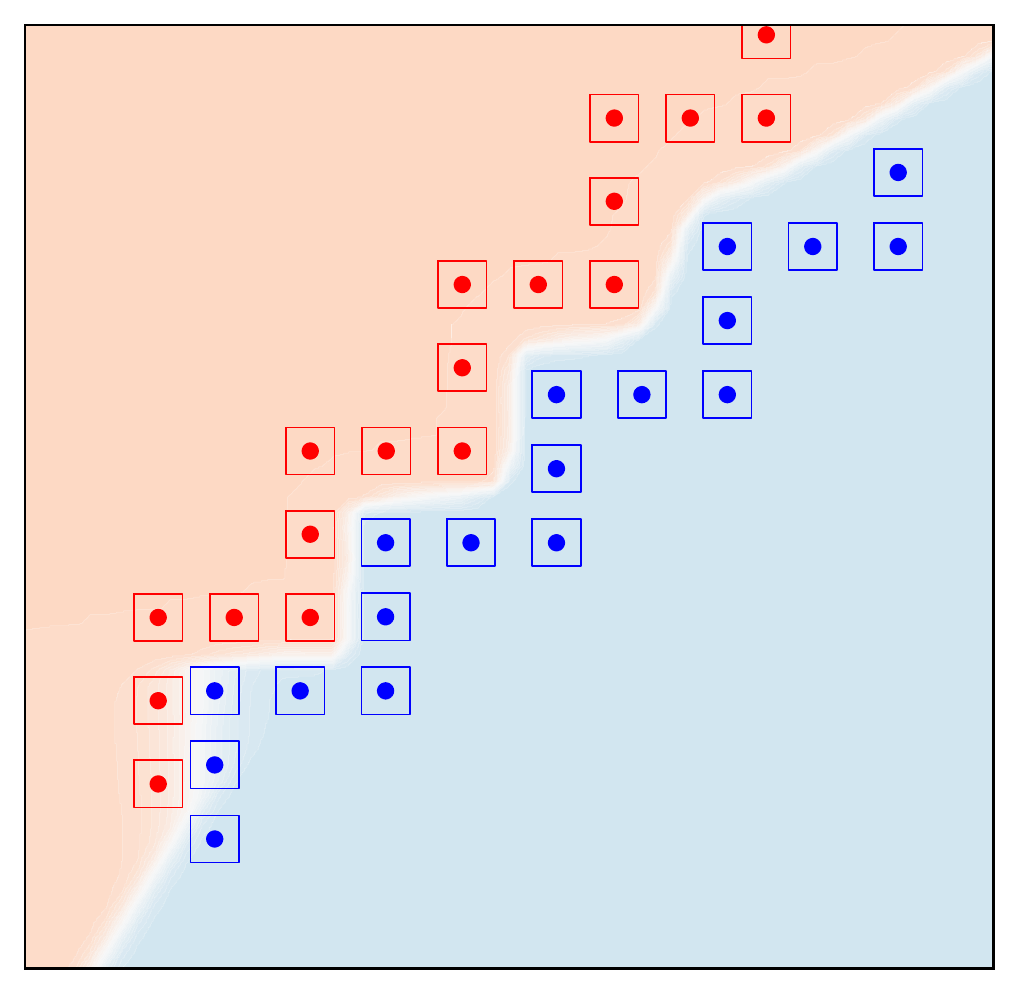}
    & \includegraphics[width=.16\linewidth]{converging_stairs_PGD_0.25_w_box.pdf}
    & \includegraphics[width=.16\linewidth]{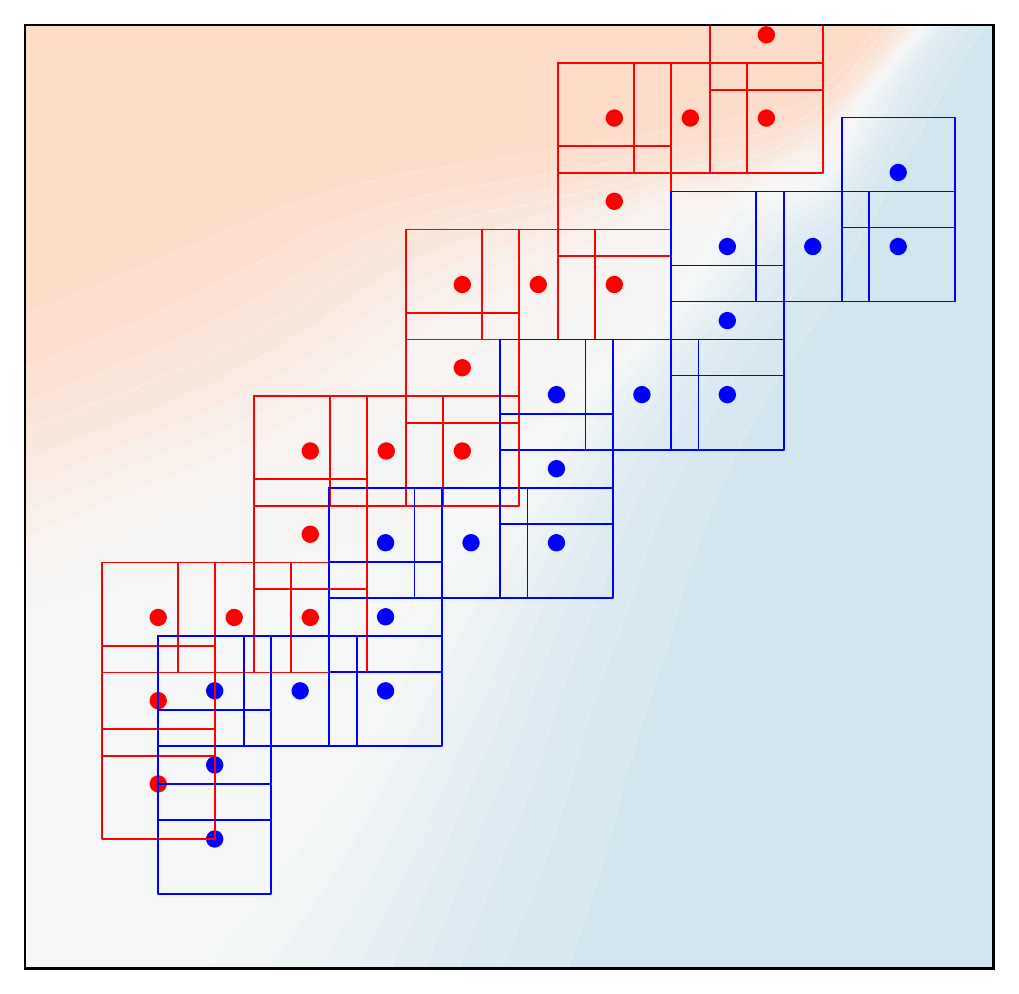}
    & \includegraphics[width=.16\linewidth]{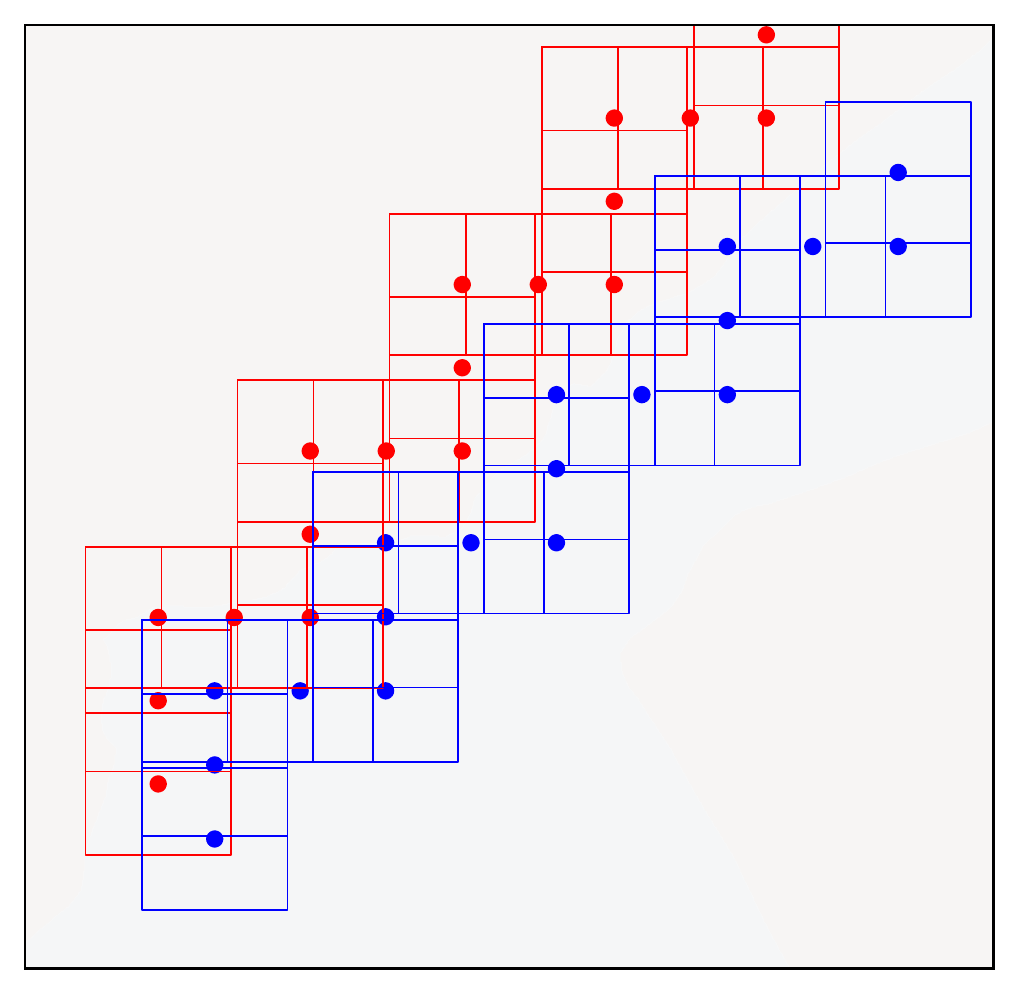}
    \\
    \rotatebox{90}{\parbox{2mm}{\multirow{3}{*}{\quad LDP-Trades}}}
    & \includegraphics[width=.16\linewidth]{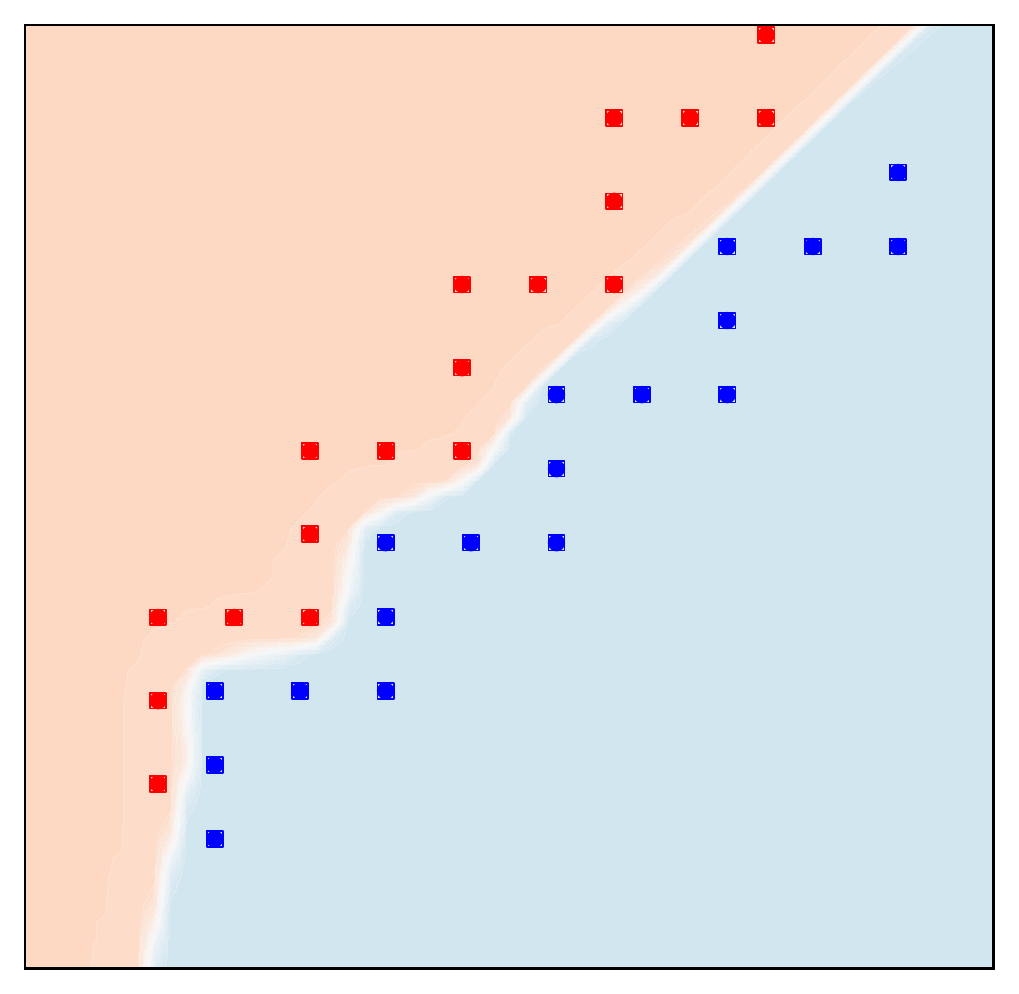}
    & \includegraphics[width=.16\linewidth]{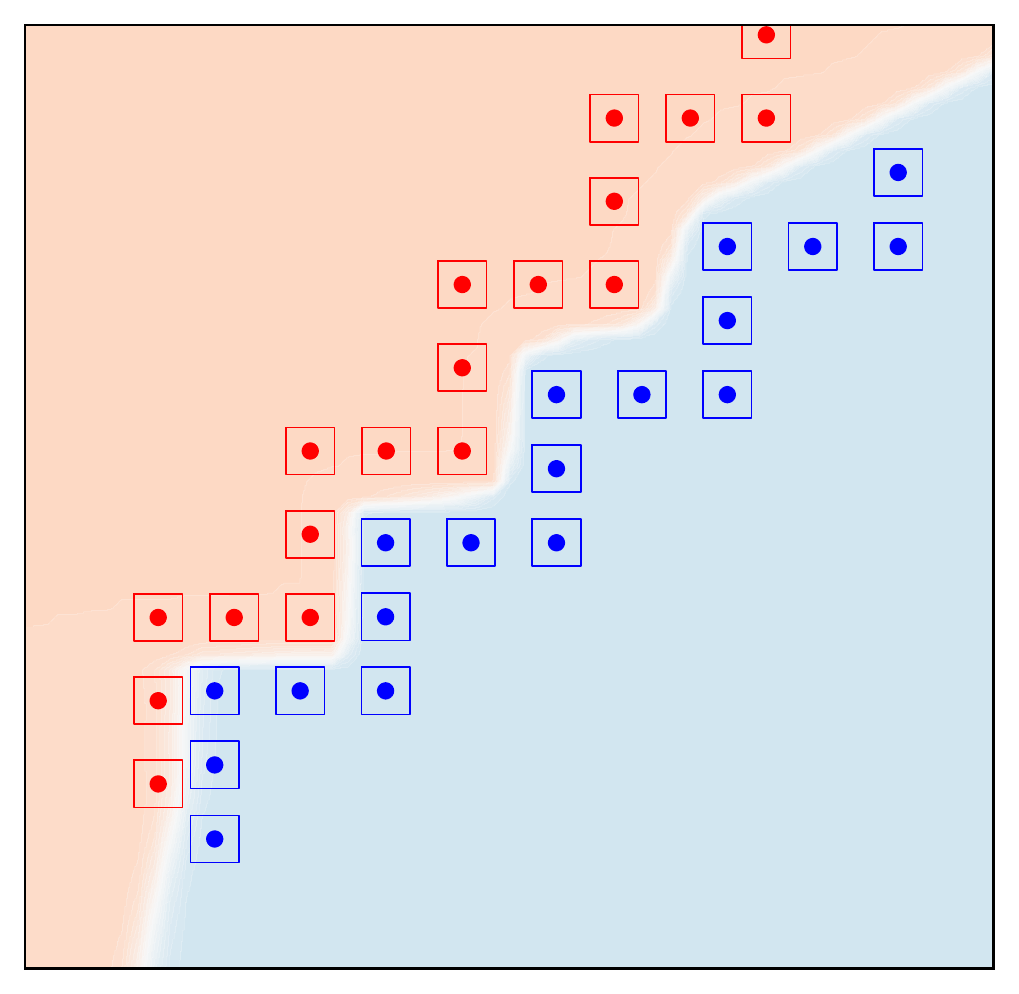}
    & \includegraphics[width=.16\linewidth]{converging_stairs_TRADES_0.25_w_box.pdf}
    & \includegraphics[width=.16\linewidth]{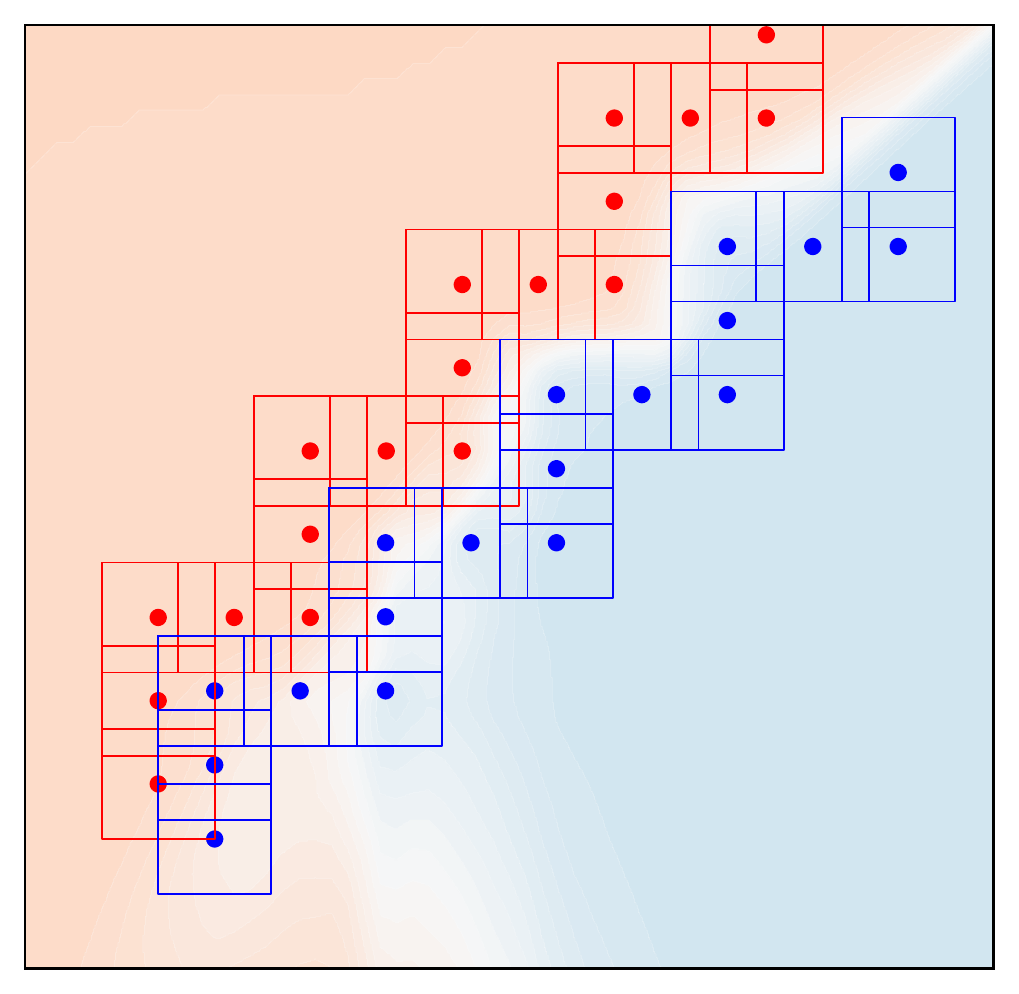}
    & \includegraphics[width=.16\linewidth]{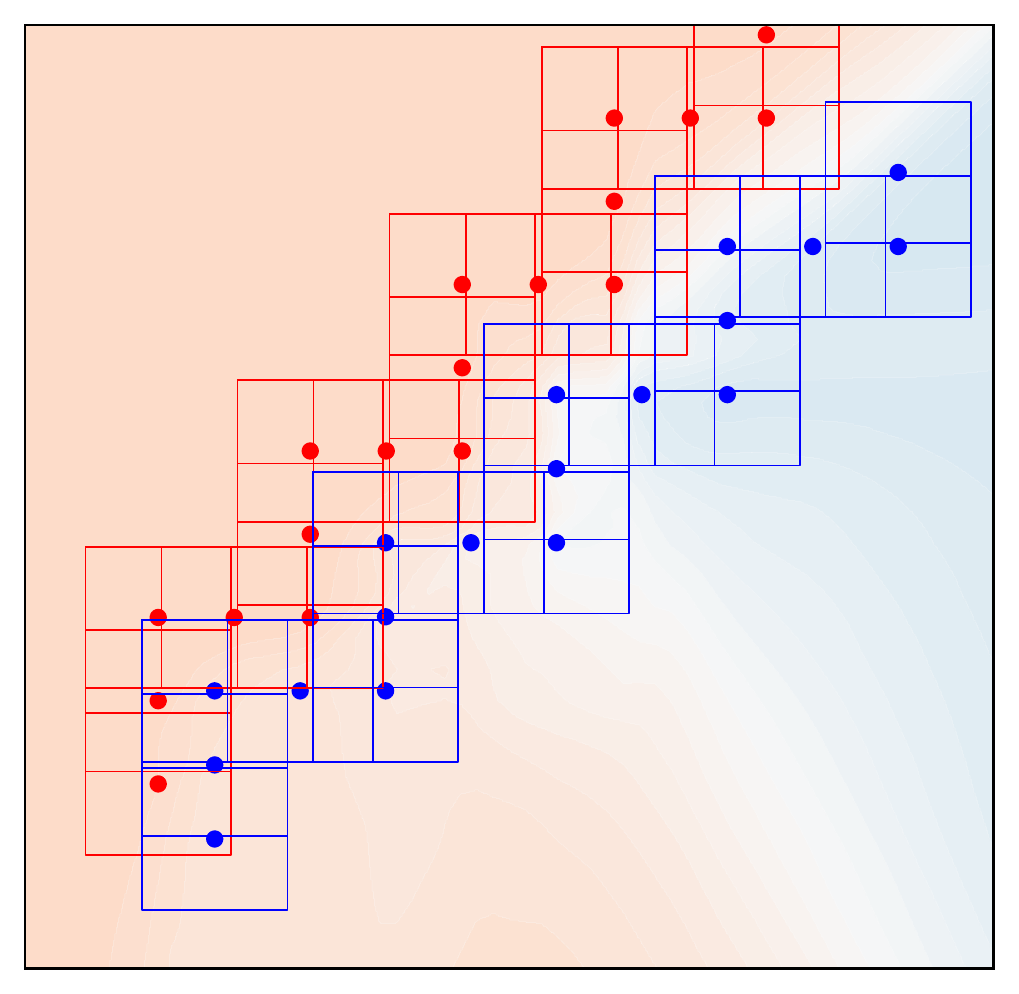}
    \\
    \rotatebox{90}{\parbox{2mm}{\multirow{3}{*}{ UDP-REG-PGD}}} 
    & \includegraphics[width=.16\linewidth]{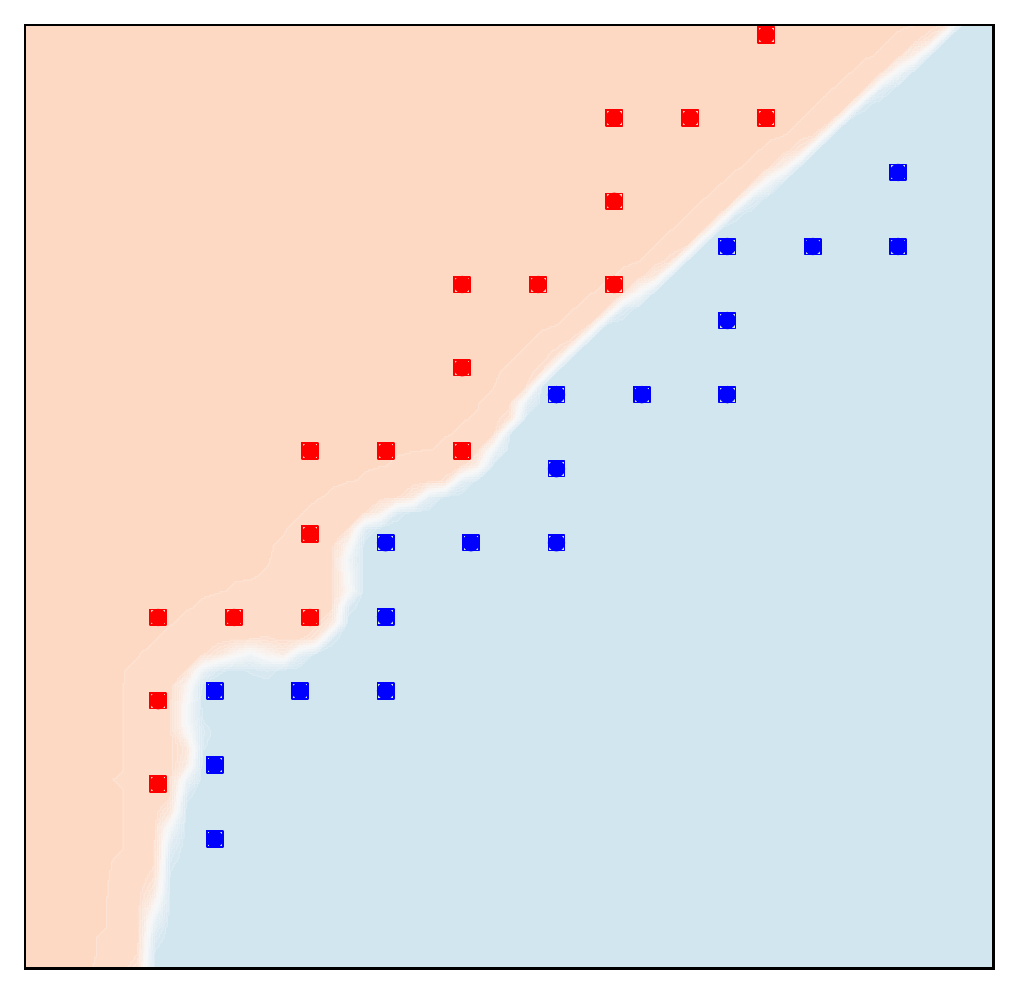}
    & \includegraphics[width=.16\linewidth]{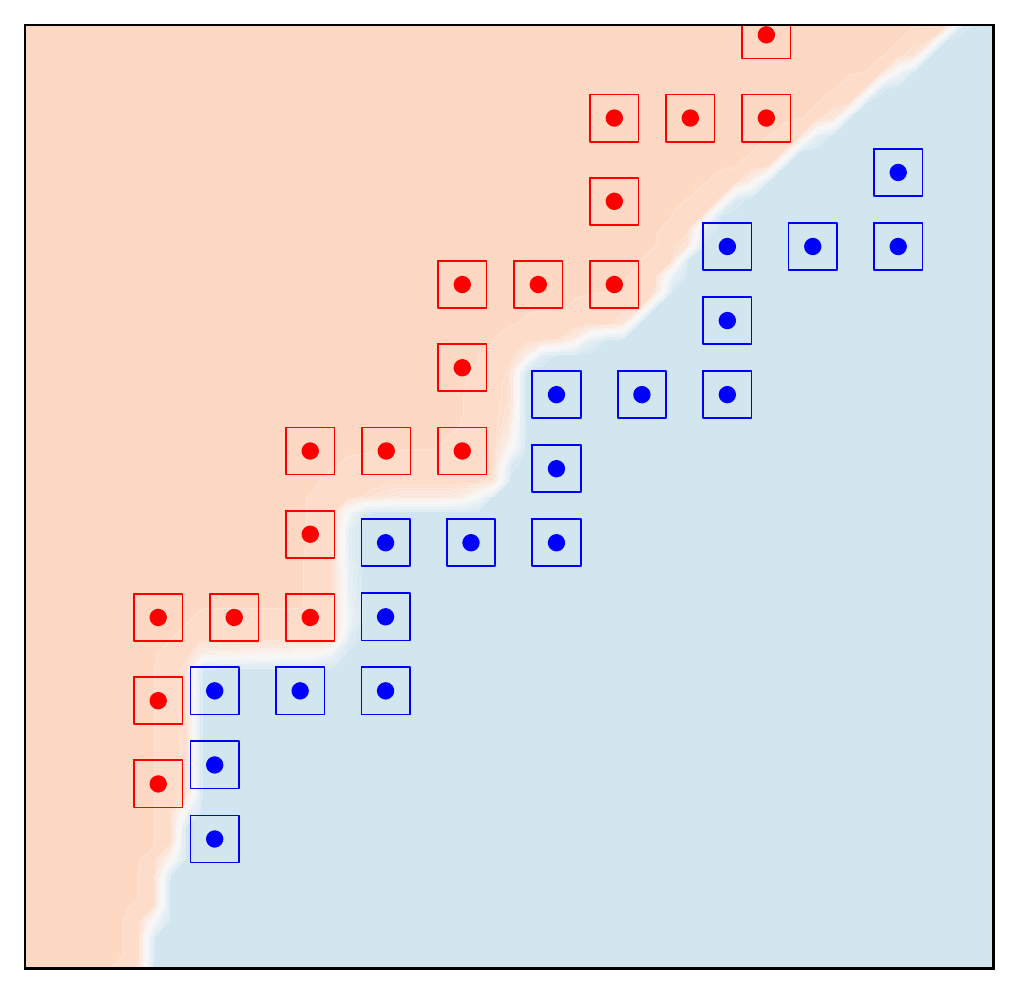}
    & \includegraphics[width=.16\linewidth]{converging_stairs_uta-reg_0.5_a01_ai50_w_box.pdf}
    & \includegraphics[width=.16\linewidth]{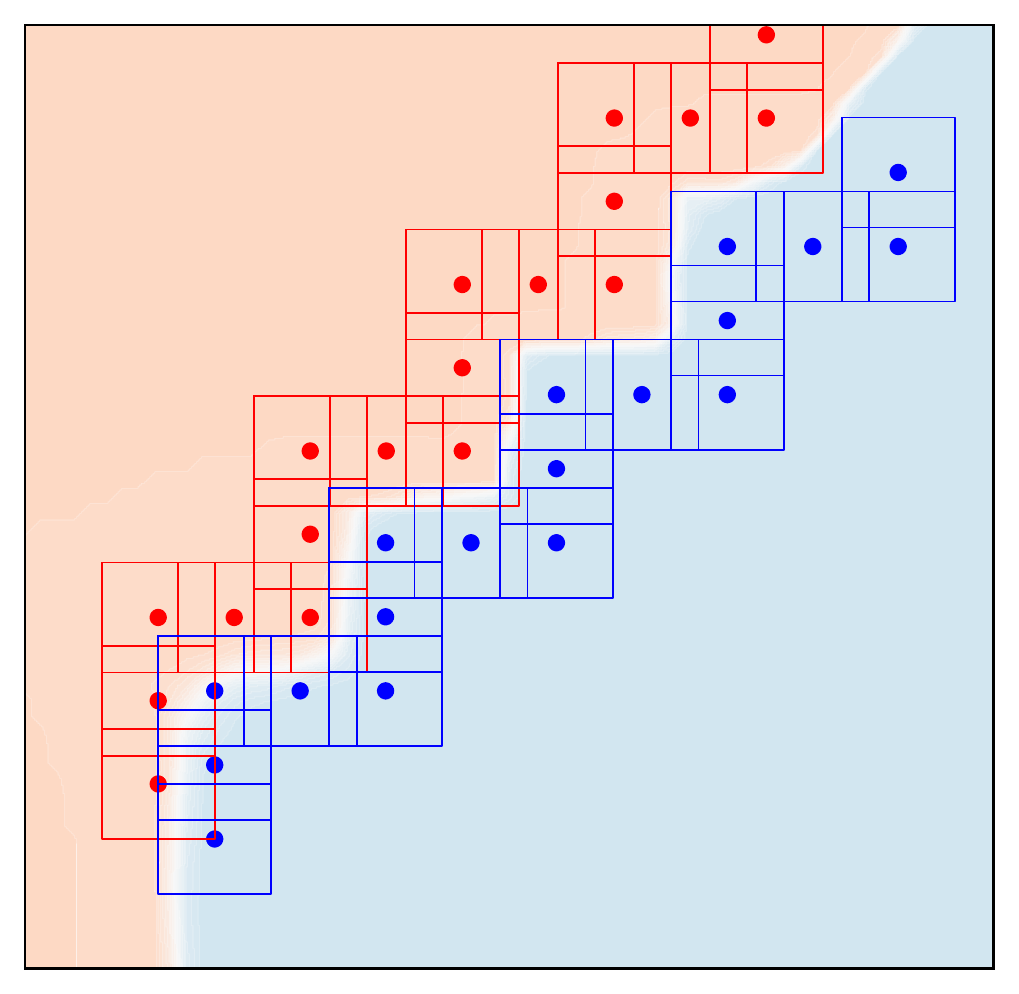}
    & \includegraphics[width=.16\linewidth]{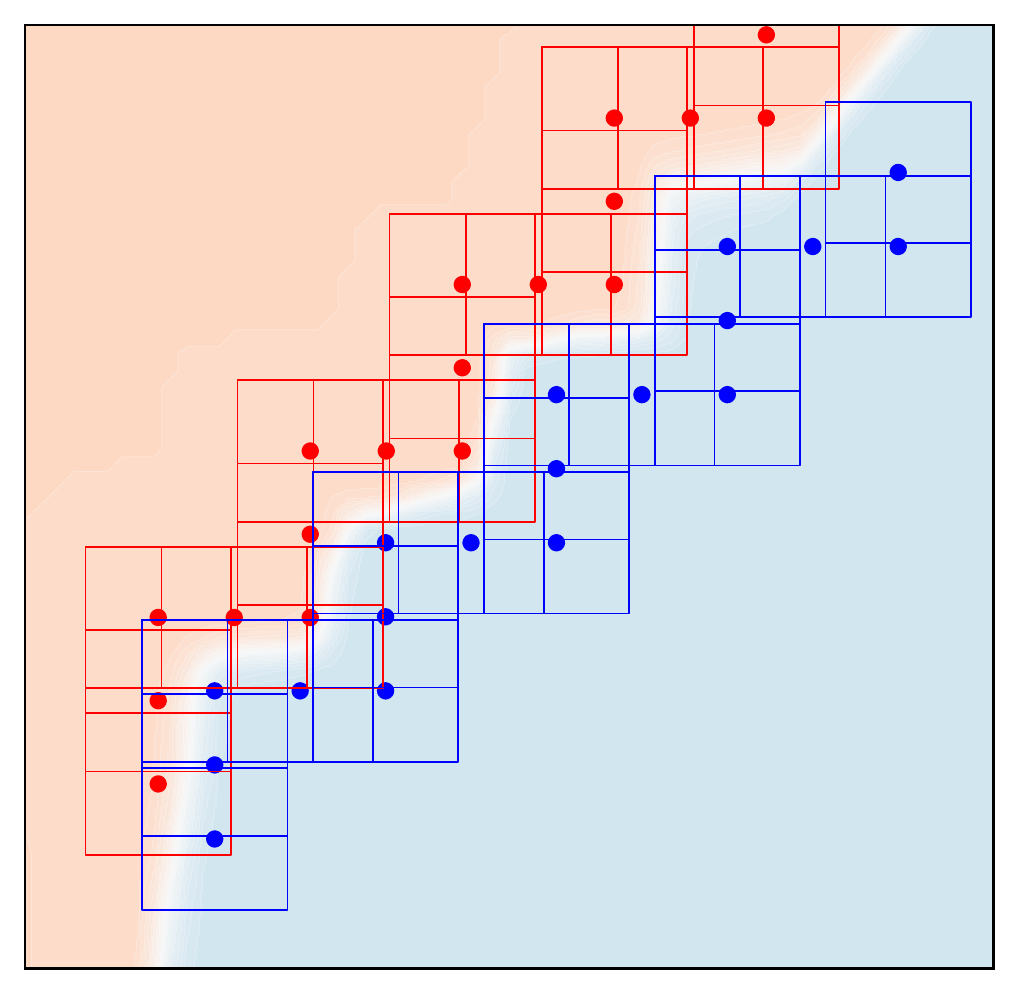}
    \\
    & $\eps=0.05$
    & $\eps=0.15$
    & $\eps=0.25$
    & $\eps=0.35$
    & $\eps=0.45$
    \\
    \end{tabular}
    \caption{Complementary results to that in Fig.~\ref{fig:toy_nc_lms5}, for varying $\eps \in {0.05, 0.15, 0.25, 0.35, 0.45}$.
    The hyperparameters step size and number of steps are fixed for all methods to $\alpha=0.05$, and  $k=10$
    }
    \label{fig:nc_varying_ps}
\end{figure}

\subsection{UDP with Gaussian Process-based model}\label{app:toy_example_gp}

In this section, we consider additional uncertainty estimation methods, and we show that UDP can be implemented with these too. 
As an uncertainty estimation model we use Gaussian Processes, which are considered a the gold standard in uncertainty estimation, see details in App.~\ref{app:related}. 
In particular, we use the \textit{Deterministic Uncertainty Estimation} (DUE) method \citep{van2021feature}. 
This method aims at estimating epistemic uncertainty using a single forward pass through the network by exploiting Gaussian Processes applied to deep architectures \citep{wilson2015deep}. We choose this particular method because it is known to give higher uncertainty estimates to outliers relative to other popular methods~\citep[see Fig. 1][]{van2021feature}. 

\subsection{Visualizations of the PGD and UDP Perturbations}

In  this section, we first train a model using DUE on the clean (unperturbed) 2D dataset, and then we depict the different attacks.

In Fig.~\ref{fig:2d_toy_due_attacks} we can see the decision boundary and the uncertainty, measured in terms of entropy of the model in classifying the clean dataset (without adversarial training). In this case, as in  \ref{sec:example}, we observe that while the PGD-$50$ samples cross notably the boundary, the UDP perturbations do not.
Moreover, UDP perturbations push the samples both towards the decision boundary and towards unexplored regions of the dataset. 
Both the PGD and UDP perturbations are applied using $\epsilon\!=\!1$ and $\alpha\!=\!0.1$.

\begin{figure}[ht] 
    \centering
    \subfigure[PGD-$50$ attack on DUE]{\label{fig:2d_toy_due_pgd}\includegraphics[width=.4\linewidth]{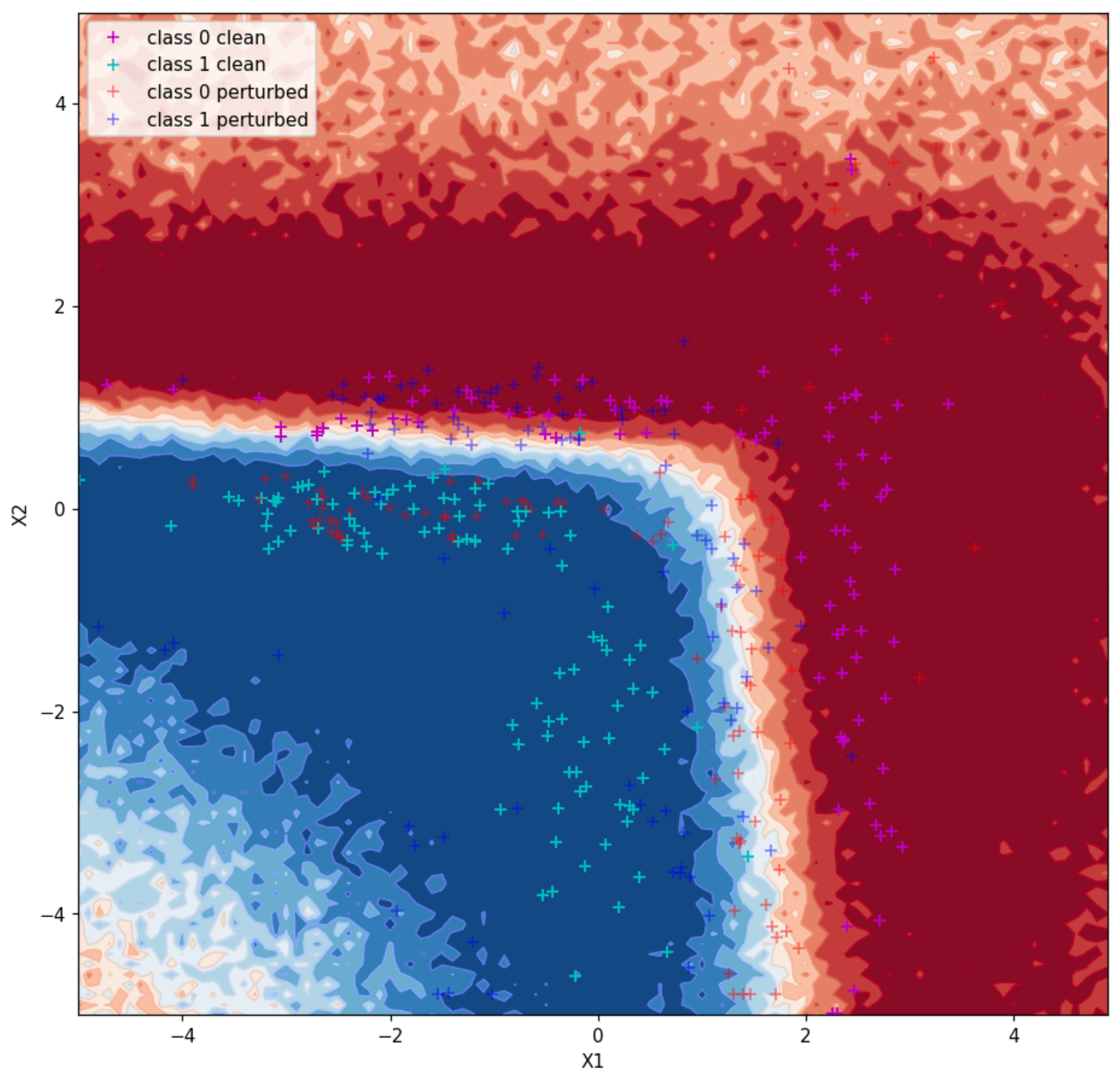}}
    \subfigure[UDP-$50$ perturbations on DUE]{\label{UTA-$50$ attack on DUE}\includegraphics[width=.4\linewidth]{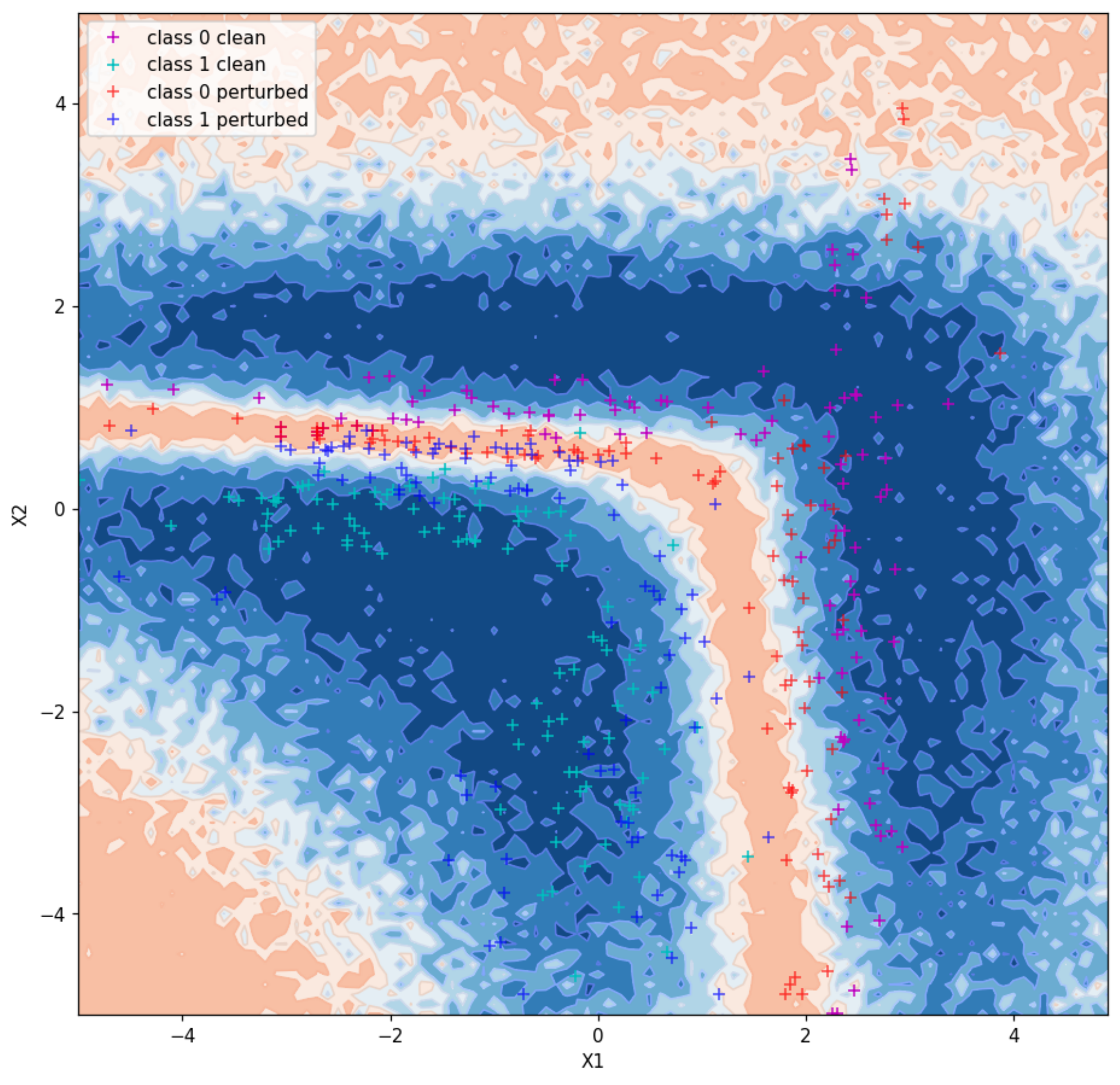}}
    \caption{PGD and UDP attacks on a \textit{fixed} DUE~\citep{van2021feature} model trained solely with clean data on the toy example from \S~\ref{sec:example}. 
    Light and dark blue crosses denote clean and perturbed samples of class 1, and similarly, pink and red crosses denote clean and perturbed samples of class 0.
    \textbf{Left}: the background depicts the loss landscape of the model  (and its decision boundary), clean dataset and perturbed samples using the PGD-$50$ attack. Note how the PGD samples cross the boundary and go to the other class' region. 
    \textbf{Right}: entropy of the model, clean dataset and UDP-$50$ perturbed samples. UDP samples do not cross the boundary and are able to go towards unexplored regions of the dataset.}
    \label{fig:2d_toy_due_attacks}
\end{figure}

\subsection{Adversarial Training using the DUE Classifier}

In this section, we use the same setup as in the previous section, however, we train adversarially the DUE model using PGD and UDP perturbations. For both PGD and UDP we use: $\epsilon\!=\!1$ and $\alpha\!=\!0.1$.

 Fig.~\ref{fig:2d_toy_due_pgd_uta_at} depicts the results. We observe that PGD training led to a very different decision boundary compared to the one obtained by training only with clean samples and leads to a model that is not able to classify the dataset with good performances. On the other hand, UDP-PGD preserves a good decision boundary, very similar to the one shown in Figure \ref{fig:2d_toy_due_attacks}.

\begin{figure}[ht]
    \centering
    \subfigure[Decision boundary of a PGD AT model]{\label{fig:2d_toy_due_pgd_at_db}\includegraphics[width=.4\linewidth]{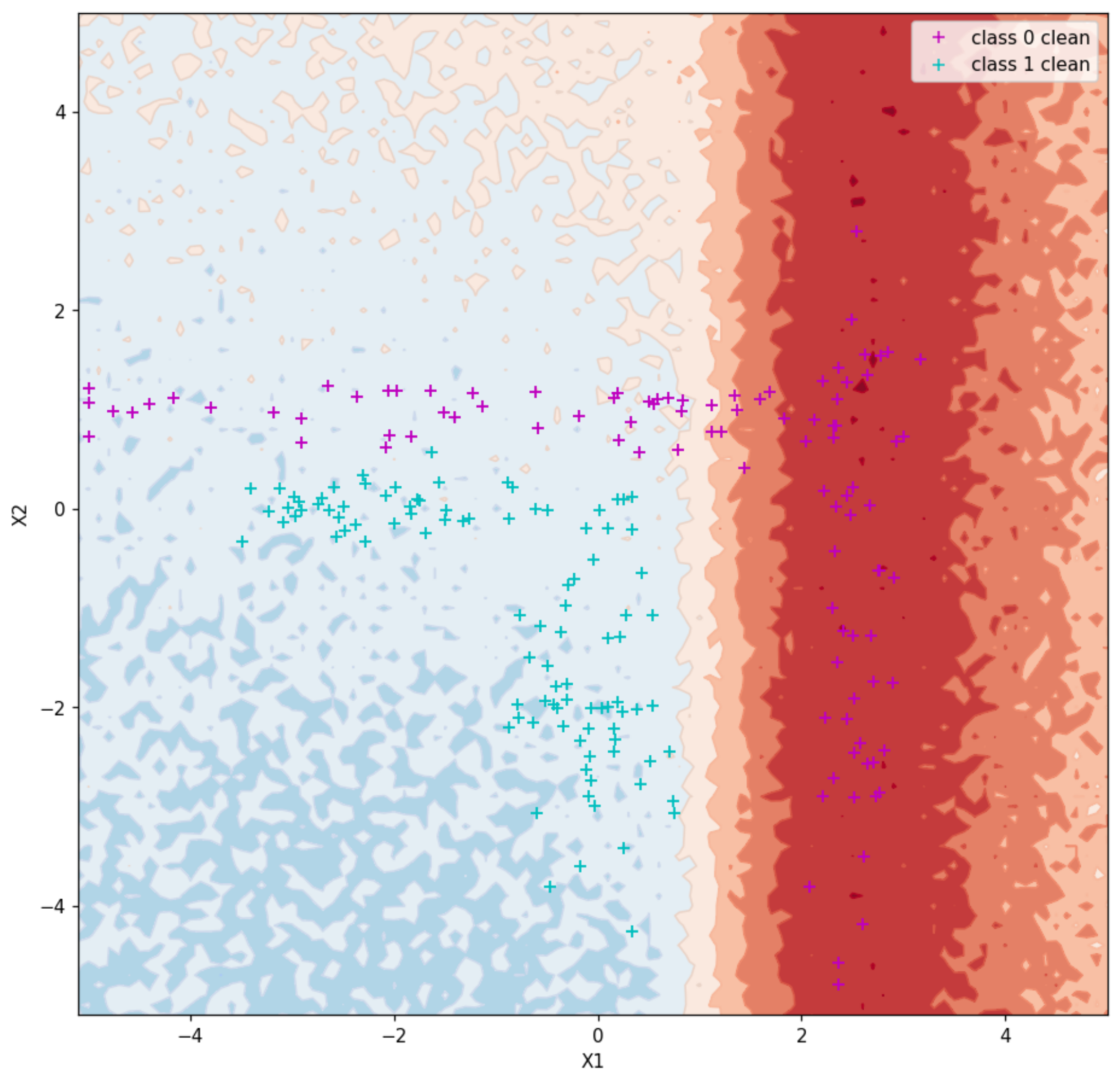}}
    \subfigure[Decision boundary of a UDP AT model]{\label{fig:2d_toy_due_pgd_at_unc}\includegraphics[width=.4\linewidth]{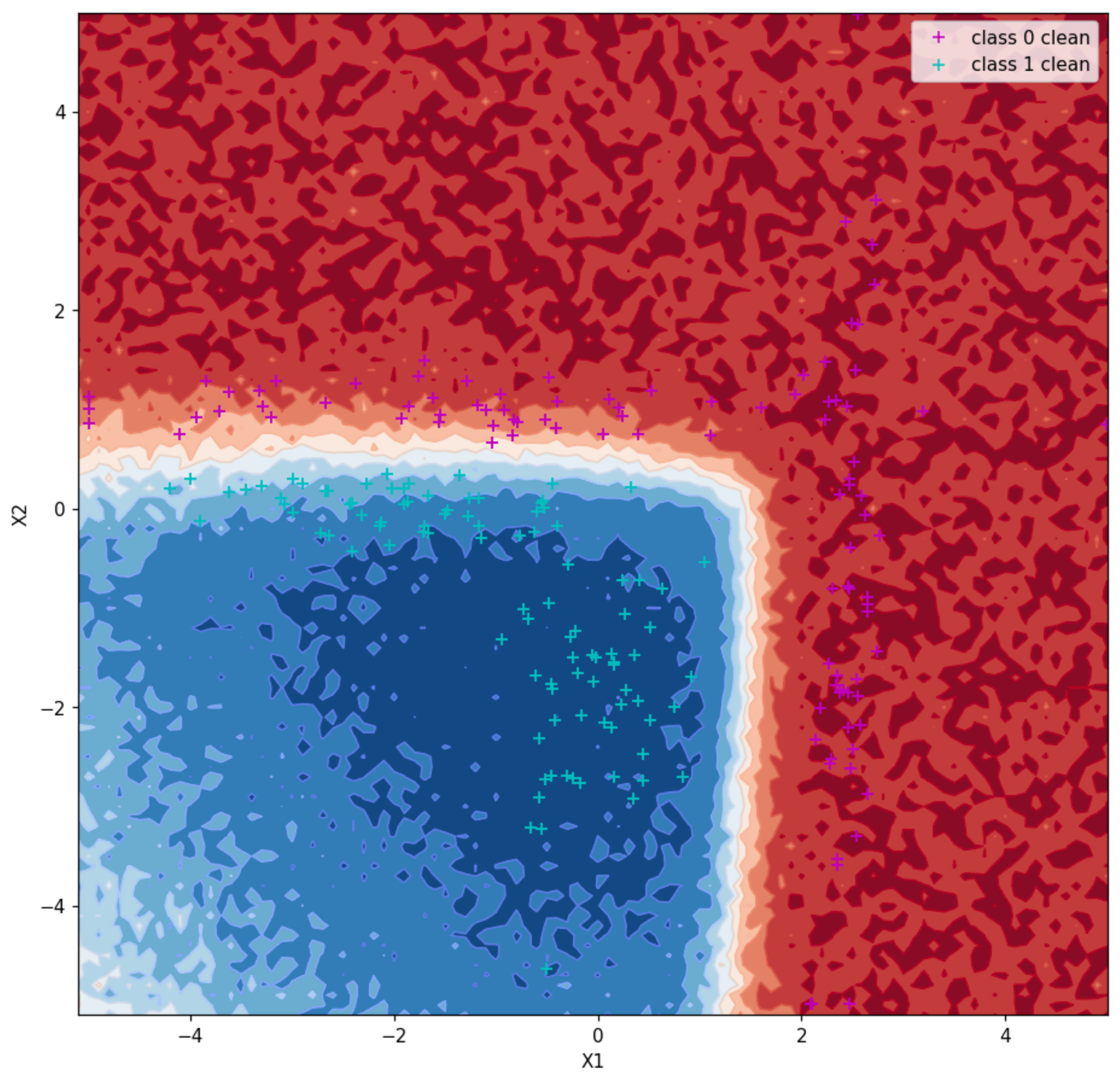}}
\caption{PGD and UDP-PGD on a DUE model trained on a non-isotropic toy example dataset. \textbf{Left}: decision boundary obtained by training adversarially with PGD-$50$ perturbed samples. The model is not able to correctly classify the dataset. \textbf{Right}: decision boundary obtained by training adversarially with UDP-$50$ perturbed samples. The model is not is now able to classify correctly the dataset and the decision boudnary is much more similar to the one obtained by training with clean samples only (shown in Figure \ref{fig:2d_toy_due_attacks}).}
\label{fig:2d_toy_due_pgd_uta_at}
\end{figure}

\clearpage
\section{Additional Results on Real-world datasets}\label{app:results}

\subsection{Omitted Results}\label{app:omitted_results}
In addition to Fig.~\ref{fig:clean-acc} of the main paper, Fig.~\ref{fig:clean-acc-app} depicts the clean test accuracy for \textbf{Fashion-MNIST}.

\begin{figure}[htb]
  \centering
  \subfigure[Fashion-MNIST]{\label{subfig:fmnist-single_clean-app}
  \includegraphics[width=.31\linewidth]{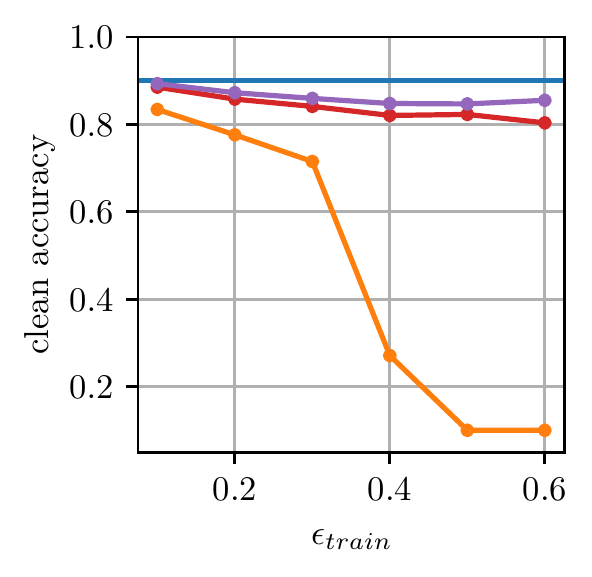}}
  \subfigure[SVHN]{\label{subfig:svhn-single_clean-app}
  \includegraphics[width=.31\linewidth]{clean-acc-SVHN_smaller.pdf}}
  \subfigure[CIFAR-10]{\label{subfig:cifar10-single_clean-app}
  \includegraphics[width=.31\linewidth,trim={0cm .0cm 0cm .0cm},clip]{clean-acc-CIFAR10_smaller.pdf}}
  \vspace*{-.1cm}
\caption{ Complementary figure to that of Fig.~\ref{fig:clean-acc} that includes \textbf{Fashion-MNIST}: \textit{Clean accuracy} ($y$-axis) comparison between PGD, TRADES, UDP-PGD, and UDPR, for varying $\eps_\text{train}$ used for training ($x$-axis). 
Results are averaged over multiple runs.}\label{fig:clean-acc-app}
\end{figure} 

\subsection{Catastrophic overfitting}\label{sec:cat_overfit}

\noindent\textbf{Fashion-MNIST.} 
We used a LeNet model trained on the Fashion-MNIST dataset. Both for FGSM and UDP-PGD \textit{with one step}  we used during training $\eps\!=\!0.2$ and $\alpha=\eps$; and for testing we used PGD with $20$ iterations, $\eps\!=\!0.2$ and $\alpha\!=\!0.01$. See App.~\ref{app:implementation} for further details.

\noindent\textbf{CIFAR-10.} 
For the experiments on CIFAR-10 we used UDP combined with an ensemble of five-models sampled with MC Dropout, using a ResNet-18~\citep{resnet} architecture, see App.~\ref{app:implementation} for details.

\begin{figure}[!htbp]
  \centering
  \subfigure[\ref{eq:fgsm} training]{\label{subfig:co_fgsm}\includegraphics[width=.35\linewidth]{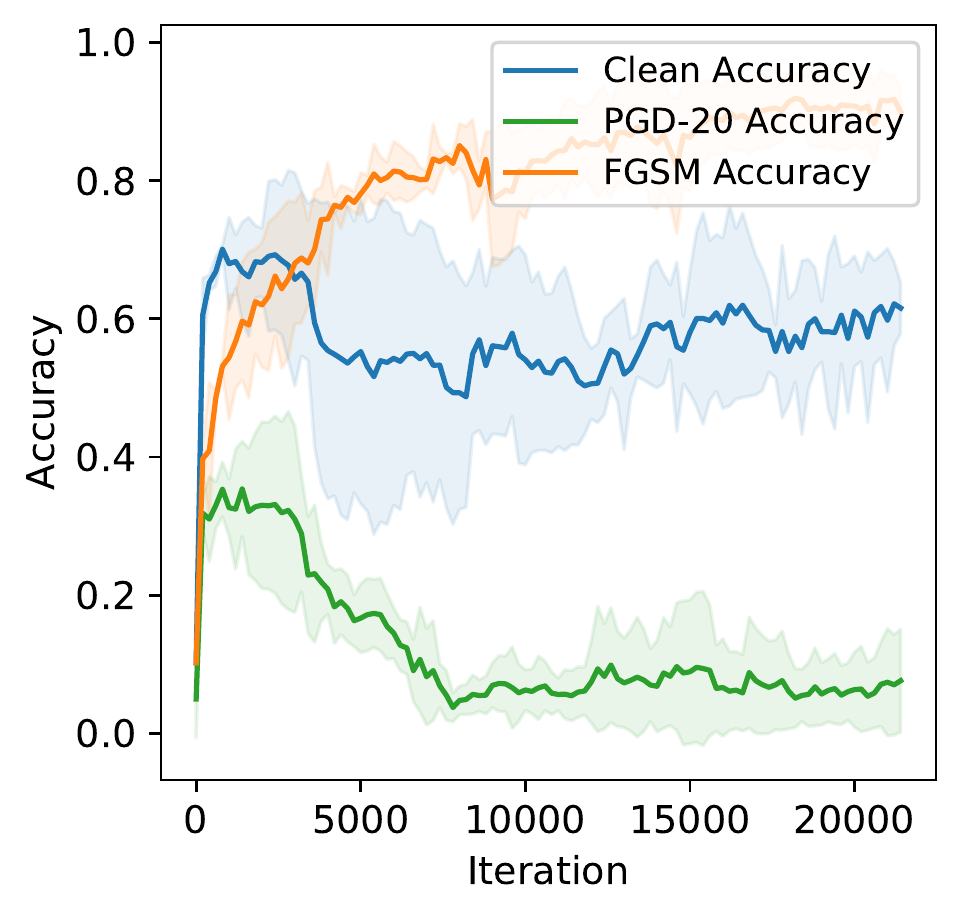}}
  \subfigure[single model UDP-PGD training]{\label{subfig:co_uta}\includegraphics[width=.35\linewidth]{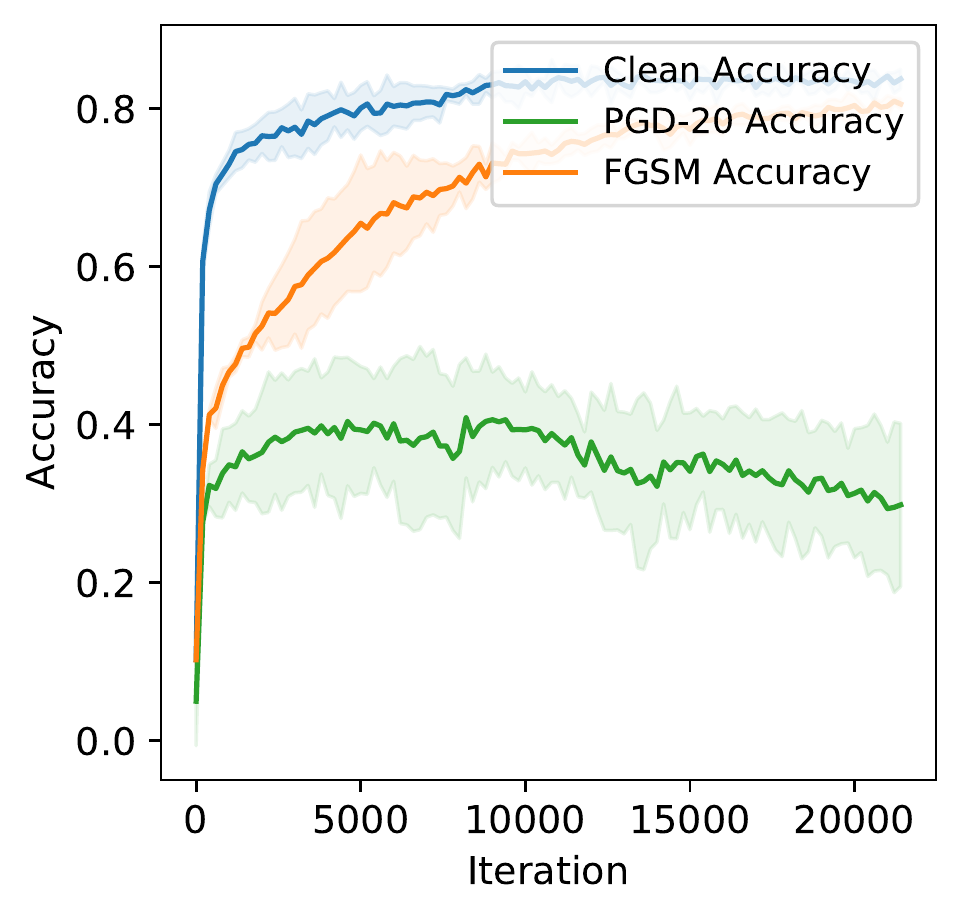}}
 \caption{Catastrophic Overfitting (CO) on the \textbf{FashionMNIST} dataset using a LeNet model; results are averaged over $3$ runs. \textbf{Left}: training with FGSM using a step size $\alpha\!=\!0.2$, $\epsilon=\alpha$; and testing against PGD-$20$ with $\epsilon\!=\!0.2$ and $\alpha\!=\!0.01$. Soon after the beginning of the training the FGSM accuracy suddenly jumps to very high values while the PGD-$20$ accuracy approaches $0$. Note also how the clean accuracy decreases and becomes lower than FGSM after CO. \textbf{Right}: training with UDP-PGD (with 1 single step) using a step size $\alpha\!=\!0.2$, $\epsilon=\alpha$; and testing against PGD-$20$ with $\epsilon\!=\!0.2$ and $\alpha\!=\!0.01$. 
 While the PGD robustness for UDP-PGD decreases to some extent, the drop is not as large as for FGSM AT.
 Note also how training with UDP-PGD with a single step leads to higher clean accuracy.
}\label{fig:co}
\end{figure}

\begin{figure}
\centering
     \includegraphics[width=.4\textwidth]{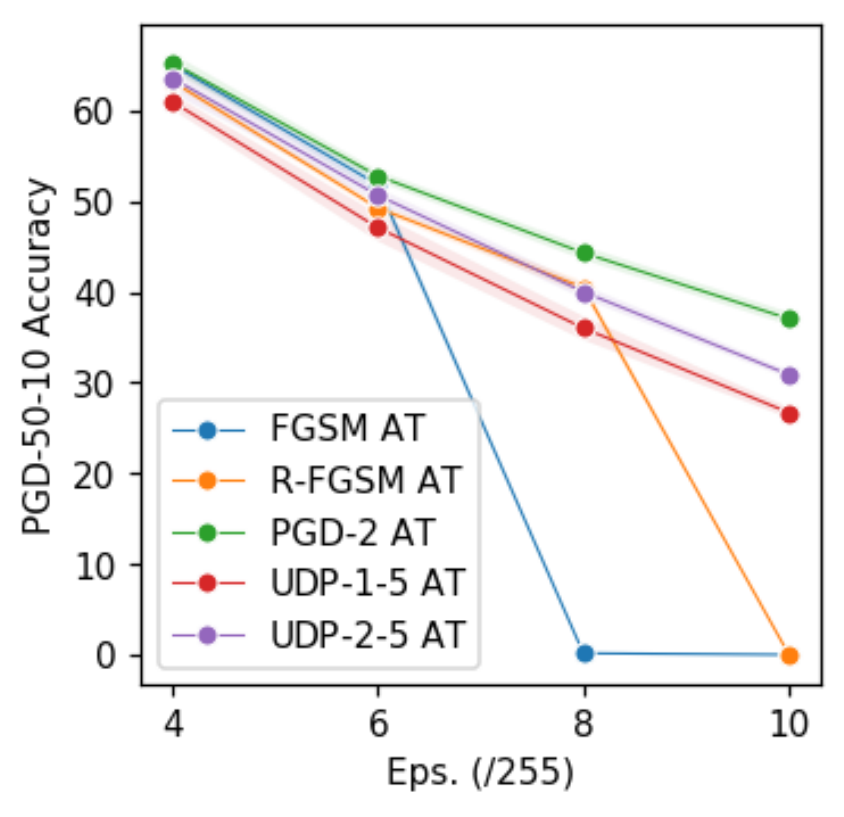}
     \caption{\textit{Catastrophic overfitting} (CO) on \textbf{CIFAR-10} using ResNet 18, averaged over $3$ runs. PGD robustness (50 iterations, 10 restarts) evaluated for multiple $\eps_{\text{test}}$ for FGSM, Random-FGSM, PGD with two steps, UDP-PGD with $1$ step and an ensemble of $5$ models, and UDP-PGD with 2 steps and and ensemble of 5 models.} \label{fig:cifar10_CO}
\end{figure}

\noindent\textbf{Results.} 
Fig. \ref{subfig:co_fgsm} and \ref{subfig:co_uta} depict the results on FashionMNIST which evaluate if the methods are prone to Catastrophic Overfitting (CO).
Relative to \ref{eq:fgsm}, single model \ref{eq:udp} notably improves CO, as although there is a downward trend in PGD-$20$ accuracy after a certain number of iterations, it does not reduce to~$0$.

\subsection{Visual Appearance of the Attacks}\label{app:visual}
In Fig.~\ref{fig:compare_imgs_1000steps_cifar} we use large $\epsilon$-ball to obtain clearly visible difference between the \ref{eq:pgd} and \ref{eq:udp-pgd} attacks, as well as to see if the insights from the experiment from Fig.~\ref{fig:toy_pgd_uta}, \S~\ref{sec:example} could extend to real world setups.
The topmost row depicts clean samples from CIFAR-10. Using an already trained classifier on CIFAR-10, we show how UDP and PGD perturbations look, in the middle and bottom rows, respectively.
We observe that for large $\epsilon$-ball PGD attacks \textit{can} make the content non-recognizable for human eye, whereas the class of UDP perturbed samples remains perceptible.

\begin{figure}[ht]
    \centering
    \includegraphics[width=.8\linewidth]{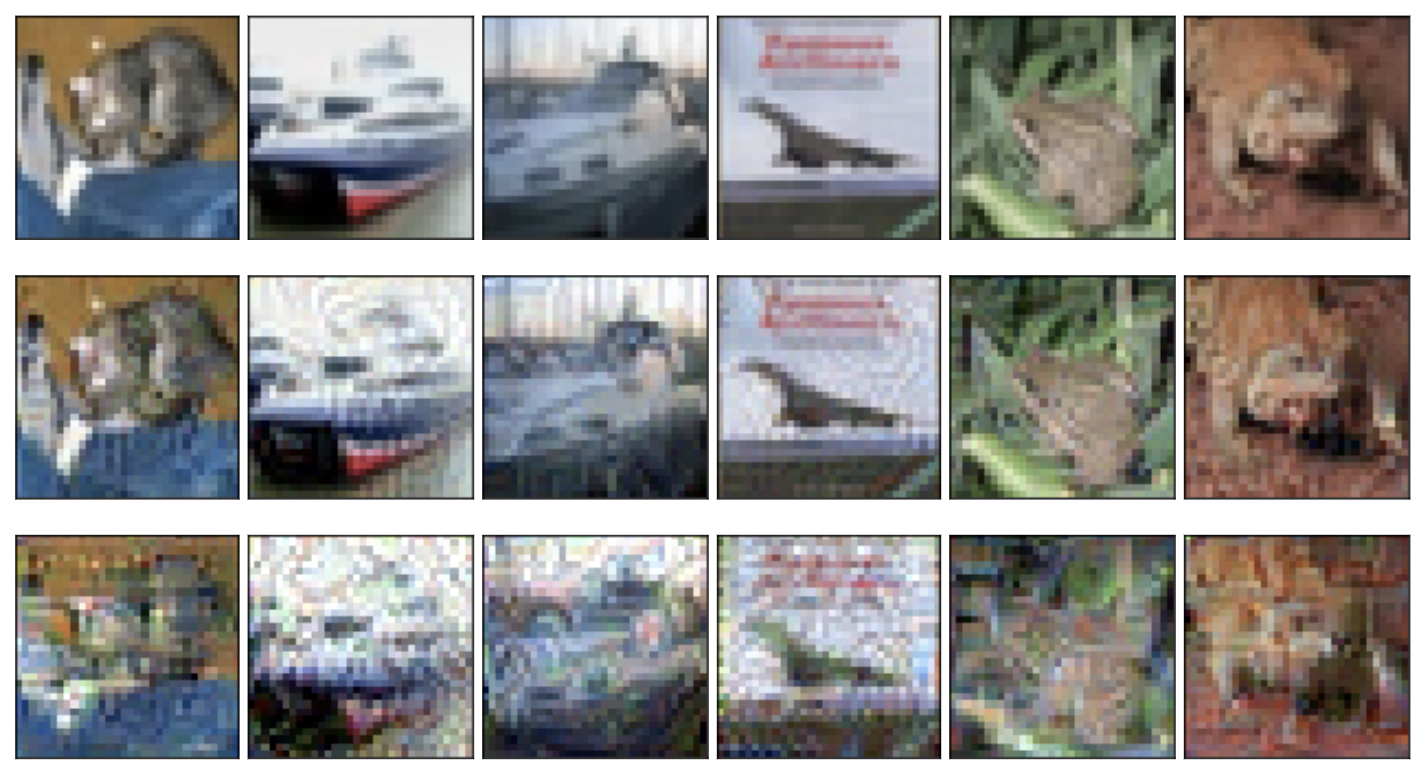}
    \caption{ Illustration of the difference between PGD and UDP perturbations for a classifier trained on CIFAR-10: 
    \textit{(i)} top row: clean samples,
    \textit{(ii}) middle row: UDP perturbations,
    \textit{(iii)} bottom row: PGD attacks, 
    where for UDP and PGD we use same setup (1000 steps, step size of $0.001$, $\epsilon = \infty$).
    We use large number of steps to verify empirically if the difference between UDP and PGD depicted in Fig.~\ref{fig:toy_pgd_uta} holds on real-world datasets as well. Contrary to the PGD-perturbed samples, the correct class of the UDP-perturbed ones remains perceptible. 
    See \S~\ref{app:visual}.
    }
    \label{fig:compare_imgs_1000steps_cifar}
\end{figure}

\clearpage
\section{Details on the implementation}\label{app:implementation}

\paragraph{Source code.} Our source code is provided in this anonymous repository:
~\href{https://github.com/mpagli/Uncertainty-Driven-Perturbations}{github.com/mpagli/Uncertainty-Driven-Perturbations}

For completeness, in this section we list the details of the implementation.

\subsection{Architectures \& Hyperparameters}\label{app:arch}
In this section, we describe in detail the architecture used for our experiments for the various datasets.

\subsubsection{Architecture for Experiments on Fashion-MNIST}\label{app:lenet_arch}

\paragraph{FashionMNIST.} 
We used a LeNet architecture as described in table \ref{tab:fashionmnist_arch}. 
This network has been trained for $100$ epochs using the Adam optimizer~\citep{kingma2014adam} with a learning rate of $0.001$.
The parameters of the models are initialized using PyTorch default initialization.

\begin{table}[h]\centering
\begin{tabular}{@{}c@{}}\toprule
\textbf{LeNet}\\\toprule
\textit{Input:} $x \in \mathds{R}^{28\times28}$ \\  \hdashline  
conv. (kernel: $5{\times}5$, $1 \rightarrow 6$;  padding: $2$; stride: $1$) \\
 ReLU  \\
 max pooling (kernel: $2\times2$; stride: 2)\\
 convolutional (kernel: $5{\times}5$, $6 \rightarrow 16$; stride: $1$)\\
 ReLU  \\
 max pooling (kernel: $2\times2$; stride: 2)\\
 Flattening \\
 fully connected ($16\times5\times5 \rightarrow 120$) \\
 ReLU  \\
 fully connected ($120 \rightarrow 84$) \\
 ReLU  \\
 fully connected ($84 \rightarrow 10$) \\ \hdashline 
 $Softmax(\cdot)$ \\ 
\bottomrule 
\end{tabular} 
\caption{LeNet architecture used for experiments on \textbf{FashionMNIST}.
With $h{\times}w$ we denote the kernel size.
With $ c_{in} \rightarrow y_{out}$ we denote the number of channels of the input and output, for the convolution layers, and the number of input and output units for fully connected layers.
}\label{tab:fashionmnist_arch}
\end{table}

\subsubsection{Architecture for Experiments on CIFAR-10 and SVHN}\label{app:arch_cifar10}
The ResNet-18 setup on CIFAR-10 and SVHN is as in~\citep{andriushchenko2020understanding}. We trained for $60$ epochs with a triangular learning rate scheduler and a peak learning rate of $0.15$. 
For the experiments in Sec.\ref{sec:model_capacity} we trained for $100$ epochs with a triangular learning rate scheduler and peak lerning rage of $0.05$. For both settings we use the SGD optimizer with a batch size of $256$.

Solely for some of the experiments in the appendix--see App.\ref{sec:cat_overfit}, we  modified the original ResNet architecture to accommodate the MC-dropout sampling procedure, see Tab.~\ref{tab:resblock}.  The modification consists of adding a dropout layer with dropout probability $p= 0.2$ after each convolutional layer.

\begin{table}\centering		 
\begin{minipage}[b]{0.5\hsize}\centering  
\resizebox{1.0\textwidth}{!}{
\begin{tabular}{@{}c@{}}\toprule
\textbf{ResBlock (part of the $\ell$--th layer)}\\\toprule
	\multicolumn{1}{l}{\textit{Bypass}:} \\
	conv. (ker: $1{\times}1$, $64 \rightarrow 64\times\ell $; str: $2$; pad: 1), if  $\ell \neq 1$\\
	Batch Normalization, if  $\ell \neq 1$\\
	\multicolumn{1}{l}{\textit{Feedforward}:} \\
	conv. (ker: $3{\times}3$, $64 \rightarrow 64\times\ell $; str: $1_{\ell=1} / 2_{\ell \neq 1}$; pad: 1) \\
	Batch Normalization \\
	$ReLU$ \\
	MCD ($p=0.2$)\\
	conv. (ker: $3{\times}3$, $64\times\ell \rightarrow 64\times\ell $; str: $1$; pad: 1) \\
	Batch Normalization \\
	\textit{Feedforward} + \textit{Bypass}\\
	$ReLU$ \\
	MCD ($p=0.2$)  \\
\bottomrule
\end{tabular}}
\end{minipage} \hspace{.1em}
\begin{minipage}[b]{0.49\hsize}\centering  
\resizebox{.8\textwidth}{!}{
\begin{tabular}{c}\toprule
 \textbf{ResNet Classifier}\\\toprule
\textit{Input:} $x \in \mathds{R}^{3{\times}32{\times}32} $  \\ \hdashline 
 conv. (ker: $3\times3$; $3 \rightarrow 64$; str: 1; pad: 1) \\
 Batch Normalization \\
 $ReLU$ \\
 MCD ($p=0.2$)\\
 $3\times$ResBlock ($\ell=1$)\\
 $6\times$ResBlock ($l \in [2, 3, 4]$) \\
 $ReLU$ \\
AvgPool (ker:$4{\times}4$ ) \\
Linear($512 \rightarrow 10$) \\
\bottomrule 
\end{tabular}
}
\end{minipage}
\caption{ResNet architectures for the experiments on catastrophic overfitting trained on \textbf{CIFAR-10}. Each ResNet block contains skip connection (bypass), and a sequence of convolutional layers, normalization, and the ReLU non--linearity. 
 For clarity we list the layers sequentially, however, note that the bypass layers operate in parallel with the layers denoted as ``feedforward''~\citep{resnet}. The ResNet block for the model (right) differs if it is the first block in the network (following the input to the model).}
\label{tab:resblock}
\end{table}

\subsection{Increasingly larger model capacity experiments}\label{app:model_capacity_implementation}

Although relatively less than LDP, we observe that when the robustness is improved, the clean test accuracy decreases for UDP too. 
Since the ``complexity'' of the dataset increases with training samples perturbations and while doing so we keep the same models that are used for the (clean) unperturbed dataset, a question arises if this decrease is due to insufficient model capacity. 
As in~\citep{madry2018towards} we run experiments with increasingly larger model capacity on Fashion-MNIST.
In particular, we keep the same architecture--in this case LeNet~\citep{Lecun98gradient}, but we increase the number of filters in the two convolutional layers and in the first fully connected layers. This is done via a multiplicative parameter. We tested with values of this parameter ranging from $2$ to $12$.

\end{document}